\newif\ifanonymoussubmission

\documentclass{article} 
\usepackage{conference,times}

\usepackage{amsmath,amsfonts,bm}

\def\eqref#1{equation~\ref{#1}}

\def\1{\bm{1}}

\DeclareMathAlphabet{\mathsfit}{\encodingdefault}{\sfdefault}{m}{sl}
\SetMathAlphabet{\mathsfit}{bold}{\encodingdefault}{\sfdefault}{bx}{n}

\DeclareMathOperator{\Tr}{Tr}

\usepackage{amsmath}
\usepackage{amssymb}
\usepackage{amsthm}
\usepackage{graphicx}
\usepackage{booktabs}
\usepackage{diagbox}
\usepackage{multirow}
\usepackage{makecell}
\usepackage{wrapfig}
\usepackage{mathtools}
\usepackage{nicefrac}
\usepackage{enumitem}
\usepackage{tabularx}
\usepackage{array}
\usepackage{tcolorbox}
\usepackage{placeins}
\usepackage{longtable}
\usepackage{xcolor}
\usepackage{todonotes}
\usepackage{algorithm,algorithmic}
\usepackage[hypertexnames=false]{hyperref}
\usepackage[capitalize,nameinlink,noabbrev]{cleveref}
\usepackage{autonum}

\definecolor{conferencegreen}{rgb}{0.0, 0.5, 0.5}

\definecolor{heldoutcolor}{rgb}{0.6, 0.15, 0.65}
\newcommand{\heldout}[1]{{\color{heldoutcolor}#1}}

\newtcolorbox{checkpointcheck}{
    colback=black!2,
    colframe=black!18,
    boxrule=0.35pt,
    arc=0pt,
    boxsep=0pt,
    left=6pt,
    right=6pt,
    top=3pt,
    bottom=3pt,
    before skip=4pt,
    after skip=4pt,
    before upper={\setlength{\parskip}{0pt}}
}

\hypersetup{
    colorlinks=true,
    citecolor=conferencegreen,    
    linkcolor=conferencegreen,     
    urlcolor=conferencegreen       
}

\newcolumntype{C}{>{\centering\arraybackslash}X}

\newtheorem{theorem}{Theorem}
\newtheorem{proposition}[theorem]{Proposition}
\newtheorem{remark}[theorem]{Remark}

\crefname{proposition}{Proposition}{Propositions}
\crefname{remark}{Remark}{Remarks}

\title{How Bregman Divergences Shape Shampoo}

\author{%
Bing Liu\textsuperscript{1},
Wenjie Zhou\textsuperscript{2},
Chengcheng Zhao\textsuperscript{\textdagger \ 1},
Hongtao Zhang\textsuperscript{2},
\textbf{Boao Kong\textsuperscript{3}},\\
\textbf{Felix Dangel\textsuperscript{*\ 4}, and Wu Lin\textsuperscript{*\ 5}}\\[5pt]
\textsuperscript{1}Zhejiang University;
\textsuperscript{2}University of the Chinese Academy of Sciences;
\textsuperscript{3}Peking University\\
\textsuperscript{4}Concordia University \& Mila;
\textsuperscript{5}University of Central Florida
}

\AddToHook{cmd/maketitle/after}{%
  \ifanonymoussubmission\else
    \begingroup
    \footnotetext[1]{Joint last authors.}%
    \footnotetext[2]{Corresponding author:
      \href{mailto:chengchengzhao@zju.edu.cn}{\texttt{chengchengzhao@zju.edu.cn}}.}%
    \endgroup
  \fi
}

\ifanonymoussubmission
  \hypersetup{pdfauthor={}}
\else
  \iclrfinalcopy
  \hypersetup{pdftitle={How Bregman Divergences Shape Shampoo},
    pdfauthor={Bing Liu, Wenjie Zhou, Chengcheng Zhao, Hongtao Zhang, Boao Kong, Felix Dangel, Wu Lin}}
  \AddToHook{cmd/maketitle/after}{\lhead{Preprint}}
\fi

\makeatletter
\newif\ifauthorcomment@inline
\AddToHook{cmd/@makefntext/before}{\authorcomment@inlinetrue}

\AddToHook{env/checkpointcheck/begin}{\authorcomment@inlinetrue}
\NewDocumentCommand{\authorcomment}{mmm}{%
  \ifinner\authorcomment@inlinetrue\fi
  \@ifundefined{@captype}{}{\authorcomment@inlinetrue}%
  \ifauthorcomment@inline
    {\color{#1} [\textbf{#2:} #3]}%
  \else
    \ifhmode\unskip\fi
    \todo[color=#1!15,bordercolor=#1!50!black,linecolor=#1!50!black,size=\scriptsize]{\textbf{#2:} #3}%
  \fi}
\makeatother

\begin{document}
\edef\restorefloatspacing{%
  \noexpand\setlength{\noexpand\textfloatsep}{\the\textfloatsep}%
  \noexpand\setlength{\noexpand\floatsep}{\the\floatsep}%
  \noexpand\setlength{\noexpand\intextsep}{\the\intextsep}%
}
\raggedbottom
\setlength{\textfloatsep}{8pt plus 2pt minus 2pt}
\setlength{\floatsep}{6pt plus 1pt minus 1pt}
\setlength{\intextsep}{8pt plus 2pt minus 2pt}
\AddToHook{env/figure/begin}[main-float-spacing]{\setlength{\parskip}{0pt}\setlength{\abovecaptionskip}{1pt}\setlength{\belowcaptionskip}{0pt}}
\AddToHook{env/wrapfigure/begin}[main-wrap-spacing]{\setlength{\parskip}{0pt}\setlength{\intextsep}{2pt}\setlength{\abovecaptionskip}{2pt}\setlength{\belowcaptionskip}{0pt}}
\AddToHook{env/table/begin}[main-float-spacing]{\setlength{\parskip}{0pt}\setlength{\abovecaptionskip}{1pt}\setlength{\belowcaptionskip}{1pt}}
\addtocontents{toc}{\protect\setcounter{tocdepth}{-1}}

\maketitle

\begin{abstract}

Understanding the principles behind Shampoo has recently guided the development of more effective neural network optimizers.
These methods learn a preconditioner by optimizing the Frobenius or Kullback-Leibler (KL) divergence against the gradient second moment.
In this work, we investigate how the choice of divergence shapes preconditioning, which remains unclear and blocks further improvements.
To do so, we develop a unified Bregman divergence framework that connects all popular divergences, allowing us to study them jointly.
Through empirical spectral analysis of gradient second moments, we examine how divergence choice shapes Kronecker approximation and interacts with finite-sample error in preconditioning. We find that some divergences can better compensate for finite-sample underestimation of the empirical second moment, helping explain the differing behavior of their corresponding Shampoo variants.
We further validate this explanation through GPT-2 pretraining experiments.
By connecting divergence choice to practical training behavior, we believe our framework provides principled guidance for understanding the foundations of, and further improving, Shampoo.

\end{abstract}

\section{Introduction}
\label{sec:introduction}

Shampoo's empirical success~\citep{dahl2023benchmarking,kasimbeg2025accelerating} {raises a fundamental question: what makes its preconditioner effective? This question} has motivated a growing body of work~\citep{morwaninew2025,vyas2025soap,an2025asgo,eschenhagen2026purifying,eschenhagen2026clarifying,lin2026understanding} with the goal of developing more effective neural network optimizers.
A promising approach is to {view} Shampoo's update scheme as a Kronecker-factored approximation of the full-matrix gradient second moment via Bregman divergence minimization~\citep{lin2026understanding}.
{This approach recovers} existing variants and produces new ones, including F-Shampoo with Frobenius divergence, VN-Shampoo with von Neumann divergence, and KL-Shampoo with Kullback-Leibler divergence.
Recent studies have reported strong performance of KL-Shampoo in LLM pre-training~\citep{eschenhagen2026clarifying,zhang2026onlineklshampoo}.
However, the variants {derived from} different divergences perform very differently in practice (\cref{fig:intro_training_motivation}, right).
{The framework of} \citet{lin2026understanding} does not explain this performance gap.
Given that improvements {to} Shampoo have been driven by understanding its principles, we argue that explaining this gap is essential for guiding the design of better variants, since many other Bregman divergences remain to be explored and evaluated.
For example, preconditioned stochastic gradient descent~\citep[PSGD,][]{li2017preconditioned,li2019preconditioner} is a related method that has shown promising results~\citep{frans2026stable} but whose success remains poorly understood.
Surprisingly, we show that PSGD's strong performance can also be understood through its connection to a new Bregman divergence.
Thus, our central question is:
\begin{center}
\emph{\textcolor{blue}{How and why} do different divergence objectives lead to \textcolor{blue}{differences in} training performance?}
\end{center}


To answer this question, we consider effects that are often conflated in practice.
(1) The Kronecker factorization does not exactly capture the second moment. 
(2) Finite samples introduce errors both in the second moment and in the estimated Kronecker factors.
(3) They also interact with each other when used for preconditioning in practice.
Divergence choice determines how these effects shape the resulting preconditioner.
Thus, we decompose the question into the following sub-questions:
\begin{itemize}[label=(Q1),leftmargin=*,labelsep=0.5em]
    \item[(Q1)] \emph{How does divergence choice shape approximation errors?}

    \item[(Q2)] \emph{How does finite sampling affect Kronecker approximation?}

    \item[(Q3)] \emph{How do divergence choices and finite sampling jointly affect preconditioning?}
\end{itemize}

By answering these questions, we make the following contributions:
\begin{itemize}[label=(C1),labelsep=0.5em,leftmargin=1em,itemindent=*]
    \item[(C1)] \textbf{A unified analysis of Kronecker approximation.}
    We unify the three divergences from \citet{lin2026understanding} into a \emph{trace-power Bregman family} and introduce a fourth, the square-root (SQ) divergence, which recovers the gradient-whitening objective of PSGD and yields SQ-Shampoo.
    We show how the family exponent determines which parts of the spectrum are prioritized, and whether over- or underestimation is penalized more (\cref{tab:spectral_preferences}).

    \item[(C2)] \textbf{Finite-sample effects on Kronecker approximation.}
    Finite sampling enters in two places.
    On GPT-2, we show that empirical gradient second moments severely underestimate tail eigenvalues even when gradients outnumber dimensions (\cref{fig:intro_training_motivation}, left). 
    As for the Kronecker factor estimation, the same exponent $q_D$ also controls the finite-sample error, with KL achieving the lowest error (\cref{tab:spectral_preferences}, last row).

    \item[(C3)] \textbf{Explaining differences in inverse-root preconditioning.}
    Last, we study how divergence-dependent approximation (C1) interacts with the two finite-sample effects in (C2) during preconditioning.
    The divergence properties help explain why SQ and KL better compensate for underestimated empirical eigenvalues (\cref{fig:intro_training_motivation}, middle) and more reliably recover their large-sample gains from short gradient histories (\cref{tab:real_state_stability}).
    Consistently, SQ and KL also achieve lower GPT-2 validation losses than F and VN, with KL's advantage growing at larger widths.
\end{itemize}

Together, these results help explain how divergences affect preconditioning and training performance, and provide practical criteria for designing better Kronecker-factored optimizers.

\begin{table}[t]
\centering
\setlength{\parskip}{0pt}
\setlength{\abovecaptionskip}{3pt}
\setlength{\belowcaptionskip}{3pt}
\caption{Overview of the divergence properties studied in this work.
The exponent $q_{\mathrm D}$ indexes the trace-power family; spectral weighting and error asymmetry determine how approximation error is allocated, while few-sample error measures the stability of Kronecker factor estimation.}
\label{tab:spectral_preferences}
\small
\setlength{\tabcolsep}{3pt}
\begin{tabularx}{\linewidth}{@{}lCCCC@{}}
\toprule
\smash{\diagbox[trim=l,width=\widthof{Weight on large eigenvalues}+\tabcolsep,height=\dimexpr\normalbaselineskip+\aboverulesep+\belowrulesep\relax]{\textbf{Property}}{\textbf{Divergence}}} & \textbf{Frobenius} & \textbf{von Neumann} & \textbf{Square root} & \textbf{KL} \\
\midrule
Trace-power exponent $q_{\mathrm D}$ & $2$ & $1$ (limit) & $\nicefrac{1}{2}$ & $0$ (limit) \\
Weight on large eigenvalues & Strong & Moderate & Mild & None \\
Penalized more & Overestimation & Overestimation & Symmetric & Underestimation \\
Error with few samples & Highest & Higher & Lower & Lowest \\
\bottomrule
\end{tabularx}
\end{table}

\begin{figure}[t]
\centering
\includegraphics[width=\linewidth]{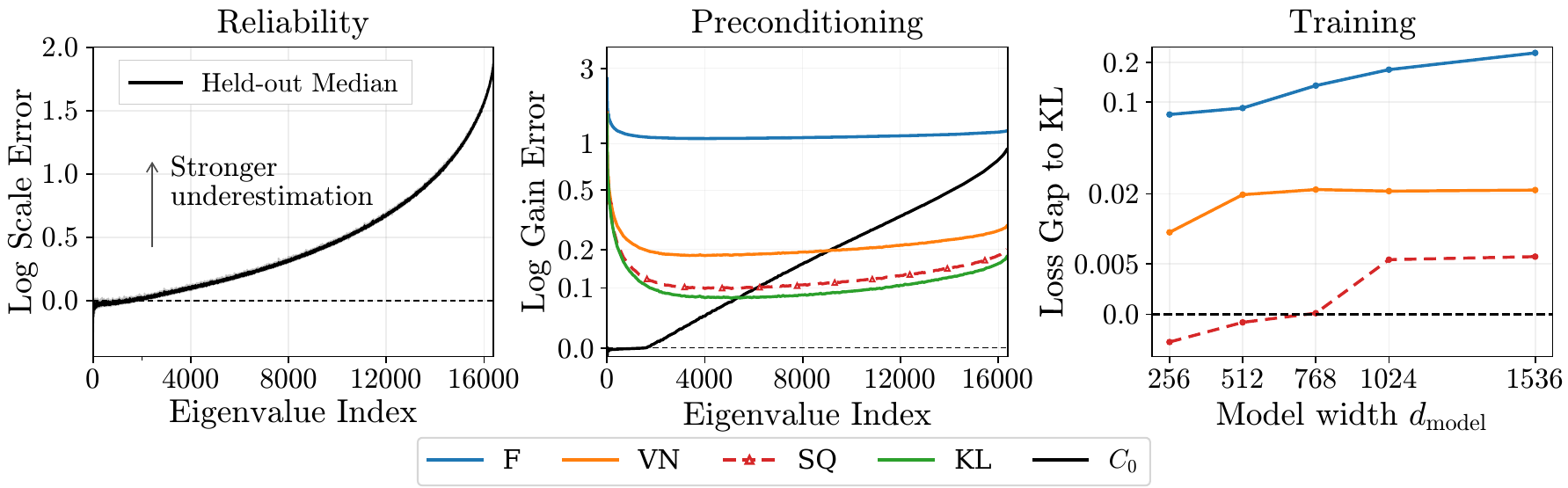}\par
\caption{\textbf{Divergences shape preconditioning and training.}
Left and middle: A GPT-2 checkpoint (\texttt{output\_proj}) at step $2{,}500$, ordered by decreasing eigenvalue.
Left: Empirical second-moment ($C_0$) eigenvalues become less reliable toward the tail, increasingly underestimating held-out scales (\cref{sec:real_second_moment_reliability}).
Middle (smoothed curves, asinh y-axis for better resolution near zero): KL stays closest to the held-out target (dashed zero line), followed by SQ, compensating $C_0$'s scale errors (\cref{sec:from_approx_to_precondition}).
Right: GPT-2 validation loss relative to KL with asinh y-axis; lower is better (\cref{sec:gpt2_training}).}
\label{fig:intro_training_motivation}
\end{figure}

\section{Background}
\label{sec:preliminaries}
\textbf{Notation.}
\label{sec:notation}
Let $G\in\mathbb R^{m\times n}$ be the stochastic gradient of a matrix parameter $\Theta$.
We use row-wise vectorization $g=\operatorname{vec}(G)$ and write $\smash{C:=\mathbb E[gg^\top]\in\mathbb S_{++}^{mn}}$ for its gradient second moment, where $\smash{\mathbb S_{++}^d}$ denotes symmetric positive-definite (SPD) matrices. For simplicity, we omit momentum, as in prior work. 

\paragraph{Kronecker Approximation and Preconditioning}
For factors $L\in\mathbb S_{++}^m, R\in\mathbb S_{++}^n$, we consider methods that approximate $C$ by a Kronecker product $S:=L\otimes R$, with inverse-root preconditioning:
\begin{align}
S^{-\nicefrac{1}{2}}g=(L^{-\nicefrac{1}{2}}\otimes R^{-\nicefrac{1}{2}})g=\operatorname{vec}(L^{-\nicefrac{1}{2}}GR^{-\nicefrac{1}{2}}).
\end{align}

\textbf{Shampoo.}
Shampoo~\citep{gupta2018shampoo,anil2020scalable} maintains averages $L_t,R_t$ of $G_tG_t^\top$ and $G_t^\top G_t$.
Its preconditioning step is $\smash{\Theta_{t+1}=\Theta_t-\eta_tL_t^{-p}G_tR_t^{-p}}$.
Motivated by gradient second-moment approximation, modern Shampoo variants ~\citep{shi2023distributed,morwaninew2025,eschenhagen2026purifying} use  $p=\nicefrac{1}{2}$, while original Shampoo \citep{gupta2018shampoo} uses $p=\nicefrac{1}{4}$.

\textbf{PSGD.}
Preconditioned stochastic gradient descent~\citep[PSGD,][]{li2017preconditioned, li2019preconditioner} learns a preconditioner $P$ that acts on the stochastic gradient as $Pg$.
Its Fisher-type gradient-whitening objective is
\begin{align}
\min_{P\in\mathbb S_{++}^{mn}}\ \Tr(PC)+\Tr(P^{-1}),
\label{eq:psgd_whitening_objective}
\end{align}
with unique dense solution $P=C^{-\nicefrac{1}{2}}$.
PSGD-Kron restricts the preconditioner to a Kronecker product $P=P_L\otimes P_R$ and learns the factors through relative-gradient updates~\citep{li2019preconditioner}.

Although PSGD and Shampoo were developed independently,  we will show in the next section that PSGD-Kron is closely related to a Shampoo variant under a new Bregman divergence. This allows us to understand them and explain their strong performance in a unified way.

\textbf{Bregman matrix divergences.}
For a strictly convex, differentiable $f:(0,\infty)\to\mathbb R$, we apply $f$ spectrally to SPD matrices. 
Given the eigendecomposition $M=\sum_{i=1}^d \mu_i u_i u_i^\top$,
$f(M)=\sum_{i=1}^d f(\mu_i)u_i u_i^\top$ and $\Phi(M)=\Tr(f(M))$.
Since $\Phi(M)=\sum_i f(\mu_i)$, $\Phi$ is also strictly convex on SPD matrices, whenever $f$ is strictly convex.
For $X,Y\in\mathbb S_{++}^d$, the Bregman matrix divergence~\citep{dhillon2008matrix} is
\begin{align}
\mathcal B_\Phi(X,Y) = \Phi(X)-\Phi(Y)-\langle\nabla\Phi(Y),X-Y\rangle_{\mathrm{F}}.
\label{eq:bregman_div}
\end{align}
where $\langle A,B\rangle_{\mathrm{F}}=\Tr(A^\top B)$ is the Frobenius inner product.
It is nonnegative and vanishes exactly when $X=Y$.
Choosing $f(x)=x^2/2$, $x\log x-x$, or $-\log x$ gives Frobenius (F), von Neumann (VN), or Kullback-Leibler (KL) divergence, respectively.
The KL divergence is also known as the LogDet divergence in the literature. \cref{app:matrix_divergence_details} gives matrix formulas for these divergences.

\paragraph{Deriving Shampoo methods via Bregman divergence minimization.}
\citet{lin2026understanding} derive F-, VN-, and KL-Shampoo variants from the stationarity condition of the following Bregman objective:
\begin{align}
\min_{L\in\mathbb S_{++}^{m},\;R\in\mathbb S_{++}^{n}}\mathcal B_\Phi(C,L\otimes R),
\label{eq:bregman_kron_approx}
\end{align}
with different generators $\Phi$. 
However, these variants perform differently in practice (\cref{fig:intro_training_motivation}, right) when training neural networks.
The relationship between the choice of $\Phi$ and empirical training performance remains unknown, as the authors do not explain it.
Understanding this connection helps guide the design of better Kronecker-factored optimizers. \cref{app:related_work} discusses related work in detail.

\section{Kronecker Approximation Error Allocation}
\label{sec:approximation_error}
We first address Q1: \emph{How does divergence choice shape approximation errors?}
We use a trace-power Bregman family to unify Shampoo variants and connect the Kronecker-constrained PSGD objective in \cref{eq:psgd_whitening_objective} to the Bregman formulation in \cref{eq:bregman_kron_approx}.
Then, within this family, we compare how the divergences allocate spectral approximation error.

\newpage
\subsection{A Trace-Power Family of Kronecker Preconditioners}
\label{sec:trace_power}
\leavevmode\par\vspace{-\baselineskip}
\begin{wrapfigure}{r}{.35\linewidth}
\centering
\vspace{-2.5ex}
\setlength{\abovecaptionskip}{0pt}
\includegraphics[width=\linewidth]{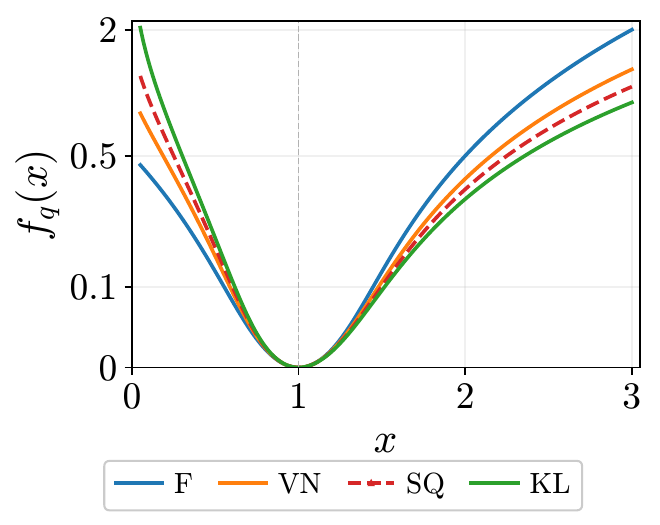}\par
\caption{The family of scalar generators $f_q(x)$ inducing trace-power Bregman divergences we study through one unifying lens. Shown with an asinh vertical scale.}
\label{fig:trace_power_generators}
\end{wrapfigure}

\noindent\textbf{Unified analysis of Bregman objectives.}
We formulate the objectives above through a \emph{trace-power Bregman family}, obtained by extending scalar divergence generators~\citep{cichocki2010families} to SPD matrices via spectral functions~\citep{amari2014information}.
Our unification further clarifies the choice of the generators and lets us connect PSGD-Kron and study its related variant within the same framework.

\par
\begin{proposition}[Unification of Shampoo and PSGD-Kron via the Trace-Power Bregman Objective]
\label[proposition]{prop:trace_power_unification}
The objectives of F-, VN-, and KL-Shampoo and the gradient-whitening objective of PSGD-Kron belong to the same trace-power Bregman family, generated by $\Phi_q(M)=\Tr(f_q(M))$.
Here, $f_q$ is applied spectrally to the eigenvalues of $M$ as defined in \cref{sec:preliminaries}.
The four scalar generators are shown in \cref{fig:trace_power_generators}:
\begin{align}
f_q(x):=\frac{x^q-qx+q-1}{q(q-1)}\ (q\neq0,1),\ f_1(x):=x\log x-x+1,\ f_0(x):=-\log x+x-1,
\label{eq:trace_power_generator}
\end{align}
Let $q_{\mathrm D}$ denote the exponent for divergence $\mathrm D$.
Frobenius has $q_{\mathrm F}=2$, while von Neumann and KL divergences arise as $q\to1$ and $q\to0$, respectively, for which we denote the exponents by $q_{\mathrm{VN}}=1$ and $q_{\mathrm{KL}}=0$.
The square-root (SQ) member, $q_{\mathrm{SQ}}=\nicefrac{1}{2}$, recovers PSGD-Kron's gradient-whitening objective~\eqref{eq:psgd_whitening_objective} under $P=S^{-\nicefrac{1}{2}}$, up to a positive scale and an additive constant.

Setting $S=L\otimes R$ gives the coupled stationarity equations:
\begin{align}
L_q=\frac{\mathbb E[GR^{q-1}G^\top]}{\Tr(R^q)},\ R_q=\frac{\mathbb E[G^\top L^{q-1}G]}{\Tr(L^q)}.
\label{eq:trace_power_bilateral_solutions}
\end{align}
Every joint stationary point satisfies these coupled equations.
See proof in \cref{app:trace_power_proof}.
\end{proposition}

\textbf{Square-root (SQ) Shampoo.}
To better understand how the choice of Bregman divergence shapes optimization, we consider a new Shampoo variant---\emph{square-root (SQ)-Shampoo}---that is induced by the Bregman family and recovers PSGD's objective function by setting $q=\nicefrac{1}{2}$ in \cref{eq:trace_power_bilateral_solutions}.
We consider this variant because the existing PSGD method does not employ a Shampoo-like scheme.

We consider four divergences in this work.
For a fair comparison, we adopt the estimation technique suggested by \citet{lin2026understanding} in our study of divergence-induced update schemes, which reduces cost and resembles Shampoo's update scheme.
Implementation details are given in \cref{app:shampoo_variant_updates}. We then examine how these divergences allocate approximation error under limited Kronecker capacity (\cref{sec:divergence_geometry}).

\subsection{Spectral Error Allocation under Limited Kronecker Capacity}
\label{sec:divergence_geometry}
As our theory (\cref{app:unavoidable_kron_error}) and empirical results
(\cref{fig:real_kronecker_approximation}, left) show, a single Kronecker product generally cannot match $C$ exactly.
The divergence therefore determines how the remaining approximation errors are penalized.
These penalty properties are not specific to gradient second moments, so we first derive them for arbitrary SPD matrices to isolate the effect of the divergence itself.

\begin{proposition}[Spectral Error Decomposition and Scaling]
\label[proposition]{prop:spectral_scale_preference}
Let $A=U\Lambda U^\top,B=V\Omega V^\top\in\smash{\mathbb S_{++}^d}$ have eigenpairs $(\lambda_i,u_i)$ and $(\omega_i,v_i)$ ordered by decreasing eigenvalue, and set $P_{ij}=\langle u_i,v_j\rangle^2$.
Writing $d_f$ for the scalar Bregman divergence generated by $f$, we have
\begin{align}
\mathcal B_\Phi(A,B) = \underbrace{ \textstyle \sum_{i=1}^d d_f(\lambda_i,\omega_i)}_{\text{eigenvalue mismatch}} + \underbrace{\Delta_f(\Lambda,\Omega;P)}_{\text{eigenspace misalignment}}, \quad \Delta_f(\Lambda,\Omega;P)\geq 0,
\label{eq:spectral_error_decomposition}
\end{align}
For the trace-power scalar generator $f_q$, write $d_q:=d_{f_q}$ and $\Delta_q:=\Delta_{f_q}$.
Both mismatch costs behave similarly when eigenvalues are re-scaled: for any $r>0$, they satisfy
\begin{align}
d_q(\lambda,r\lambda)=\lambda^q d_q(1,r),\quad \Delta_q(r\Lambda,r\Omega;P)=r^q\Delta_q(\Lambda,\Omega;P).
\label{eq:relative_spectral_penalties}
\end{align}
Thus $q$ controls how strongly both mismatch costs depend on eigenvalue scale.
\end{proposition}

\textbf{Eigenspace mismatch.}
Keeping the eigenvectors fixed (and hence $P$ unchanged), multiplying all eigenvalues of both matrices by $r>1$ multiplies the eigenspace-mismatch cost by $r^q$. Thus larger $q$ places greater relative weight on the same eigenspace misalignment at larger eigenvalue scales.

\noindent\textbf{Eigenvalue (spectral) mismatch.}\hspace{1em}%
Writing $d_{\mathrm F}=d_2$, $d_{\mathrm{VN}}=d_1$, $d_{\mathrm{SQ}}=d_{\nicefrac{1}{2}}$, and $d_{\mathrm{KL}}=d_0$, the scalar eigenvalue penalties are
\begin{equation}
\label{eq:proof_relative_spectral_penalties}
\begin{aligned}
d_{\mathrm F}(\lambda,r\lambda)
&= \tfrac{1}{2}\lambda^2(r-1)^2,
&\quad d_{\mathrm{VN}}(\lambda,r\lambda)
&= \lambda(r-1-\log r), \\
d_{\mathrm{SQ}}(\lambda,r\lambda)
&= 2\lambda^{\nicefrac{1}{2}}\left(\sqrt r+\nicefrac{1}{\sqrt r}-2\right),
& d_{\mathrm{KL}}(\lambda,r\lambda)
&= \log r+\nicefrac{1}{r}-1.
\end{aligned}
\end{equation}

\par
\begin{wrapfigure}{r}{.53\linewidth}
\vspace{-2.5pt}
\centering
\includegraphics[width=\linewidth]{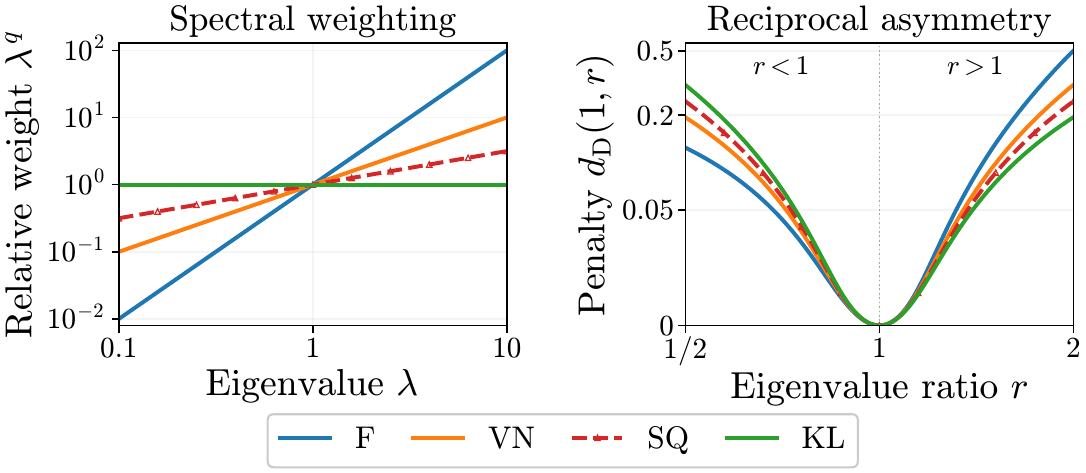}\par
\caption{Left: Normalized penalties $\smash{\frac{d_{\mathrm D}(\lambda,r\lambda)}{d_{\mathrm D}(1,r)}}=\lambda^{q_{\mathrm D}}$ for $r\ne1$. Smaller $q$ places less emphasis on large eigenvalues; KL is scale-independent. Right: $d_{\mathrm D}(1,r)$ versus $r$, with a log horizontal scale and an asinh vertical scale. F and VN favor underestimation over reciprocal overestimation, KL shows the opposite preference, and SQ is symmetric.}
\label{fig:spectral_penalties}
\end{wrapfigure}

For fixed relative mismatch $r$, the weights $\lambda^2,\lambda,\lambda^{\nicefrac{1}{2}},1$ show that larger $q$ puts more emphasis on large eigenvalues (\cref{fig:spectral_penalties}, left).
We call their associated eigenvectors \emph{leading directions}.

\noindent\textbf{Over- versus under-estimation and asymmetry.}\hspace{1em}%
For $r>1$, the fitted values $r\lambda$ and $\lambda/r$ over- and under-estimate $\lambda$ by the same multiplicative factor.
Their penalties satisfy (\cref{fig:spectral_penalties}, right)
\begin{align}
\label{eq:over_under_ordering}
d_{\mathrm F}(\lambda,r\lambda)>d_{\mathrm F}(\lambda,\lambda/r), \\
d_{\mathrm{VN}}(\lambda,r\lambda)>d_{\mathrm{VN}}(\lambda,\lambda/r),\\
d_{\mathrm{SQ}}(\lambda,r\lambda)=d_{\mathrm{SQ}}(\lambda,\lambda/r), \\
d_{\mathrm{KL}}(\lambda,r\lambda)<d_{\mathrm{KL}}(\lambda,\lambda/r).
\end{align}
\WFclear
\noindent\cref{tab:spectral_preferences} summarizes these properties together with the finite-sample results from \cref{sec:finite_sample_stability}.
\cref{sec:kronecker_spectral_penalties} gives the Kronecker decomposition and shows that the same penalty rules still hold.

\begin{checkpointcheck}
\textbf{Takeaway.}
Smaller $q_{\mathrm D}$ gives more weight to the tail spectrum and penalizes underestimation more strongly, encouraging larger fitted eigenvalues.
These larger eigenvalues can help avoid excessive inverse-root amplification, motivating smaller $q_{\mathrm D}$ in preconditioner design.
\end{checkpointcheck}

\section{Finite-Sample Effects on Kronecker Approximation}
\label{sec:finite_sample_effects}
We now turn to Q2: \emph{How does finite sampling affect Kronecker approximation?}
Finite sampling enters in two places: the empirical second-moment target and the estimated Kronecker factors in~\cref{eq:trace_power_bilateral_solutions}.
We examine the reliability of the target before Kronecker fitting, and then study divergence-dependent factor estimation from limited samples.

\subsection{Reliability of the Empirical Second-Moment Target}
\label{sec:real_second_moment_reliability}
We test the reliability of the empirical gradient second moment by computing its eigenvectors on one data split and comparing second-moment values along the same directions on other independent data splits (\cref{fig:intro_training_motivation}, left).

\noindent\textbf{Shared experimental design (full details in \cref{app:real_model_protocol}).}
\cref{sec:real_second_moment_reliability,sec:from_approx_to_precondition} use the same setup.
First, we train a model (GPT-2~\citep{modded_nanogpt_2024} on FineWeb, $12$ layers, hidden size $128$) with AdamW and collect a large number of mini-batch gradients (batch size $64$, sequence length $512$), at selected checkpoints ($100$, $2{,}500$, and $4{,}500$) during training.
We then split those gradients into four disjoint splits (each of size $N=20{,}480$).
Each split gives rise to an empirical gradient second moment $C_s = \nicefrac{1}{N} \smash{\sum_{n=1}^N g_{s,n} g_{s,n}^\top}$ with $g_{s,n}$ the $n$th mini-batch gradient in split $s$.
In what follows, we focus on the query, value, and output projection weights of the model's 5th block, for which $N > d = 128^2$.
Thus, rank deficiency is not forced by the sample count.
We repeated the analyses on checkpoints generated by training with Shampoo and find similar results (\cref{app:distributed_shampoo_checkpoints}).
This large-sample setting provides \emph{an offline reference} for subsequent preconditioning comparisons.


\textbf{Are empirical second-moments reliable?}
Following high-dimensional covariance estimation work~\citep{ledoit2012nonlinear}, we estimate eigenvectors on a source split $s$ and test their directional scales on a held-out split $s'\ne s$.
Let $\{(\lambda_{s,i},u_{s,i})\}_{i=1}^d$
denote the eigenpairs of the empirical moment $C_s$ on split $s$.
To assess the representativeness of the source eigenvalues, we keep the source directions $u_{s,i}$ fixed and compare their \emph{directional scales} across splits.
Along each direction, the source scale is $\lambda_{s,i}$ and the held-out scale is $u_{s,i}^\top C_{s'}u_{s,i}$.
These scales can agree even if $u_{s,i}$ is not an eigenvector of $C_{s'}$.
\pagebreak[1]
We measure \emph{held-out scale error} by comparing held-out and source scales:
\begin{align}
\xi_{s\to \heldout{s'},i}:=\log_{10}\frac{u_{s,i}^\top \heldout{C_{s'}} u_{s,i}}{u_{s,i}^\top C_su_{s,i}}=\log_{10}\frac{u_{s,i}^\top \heldout{C_{s'}} u_{s,i}}{\lambda_{s,i}}.
\label{eq:heldout_second_moment_ratio}
\end{align}
\textbf{Interpretation.} Zero indicates cross-split agreement; positive values indicate underestimation and negative values indicate overestimation. 
Larger absolute values indicate greater mismatch. 
\textbf{Thus, values close to zero across the spectrum indicate reliable empirical second moments.}
\begin{remark}
We focus on eigenvalue (scale) reliability, without eigenvector stability.
These directions need not remain eigenvectors across splits.
What matters is whether the scales along these directions reproduce on independent data, supporting the scaling applied by the preconditioner is stable.
\end{remark}

\noindent\textbf{Empirical eigenvalues become less reliable toward the tail.}
In \cref{fig:intro_training_motivation} (left), using the AdamW step-$2{,}500$ output projection, held-out scale error is near zero for large eigenvalues and rises toward smaller eigenvalues, showing increasingly severe underestimation.
Held-out scale error remains near zero in leading directions and strongly positive in the tail across checkpoints and projections from both AdamW and Distributed Shampoo (\cref{tab:real_reliability_regions_full,fig:real_reliability_all}). 
We also compare scales along directions selected from a third split $s''$, independently of both splits above.
The resulting cross-fit errors stay near zero across the spectrum, supporting sampling noise as the source of the scale differences (see \cref{app:finite_sample_tail_selection}).

\noindent\textbf{Held-out scale as the correction.}
Covariance estimation and mini-batch curvature debiasing consider a correction based on separate data: keep the estimated eigenvectors, but replace their sample eigenvalues with scales measured along those directions on an independent batch \citep{ledoit2012nonlinear,tatzel2025debiasing}.
Here, the replacement for $\lambda_{s,i}$ is precisely the held-out scale $u_{s,i}^{\top}\heldout{C_{s'}}u_{s,i}$ in~\cref{eq:heldout_second_moment_ratio}.
This motivates using it as the reference for preconditioner gains in \cref{sec:real_inverse_root_action} 

\begin{checkpointcheck}
\textbf{Takeaway.}
Even with many gradients, empirical second moments can underestimate the tail of the spectrum.
Using one data split to estimate the eigenvectors and another to measure their scales reveals this finite-sample bias.
Under inverse-root preconditioning, such underestimation can lead to overly large gains, and is even more pronounced in practice with short gradient histories.
\end{checkpointcheck}

\subsection{Finite-Sample Stability of Kronecker Factors}
\label{sec:finite_sample_stability}
The previous subsection shows finite-sample bias in the second moment.
In practice, Kronecker factors are also estimated from shorter gradient histories, introducing another source of sampling error.
We use a controlled model to isolate how divergence choice affects this factor estimation.
\paragraph{Model and isolated empirical factor estimation.} To separate sampling error from approximation error in Kronecker approximation, we consider an exact Kronecker Gaussian model $C=L_\star\otimes R_\star$:
\begin{align}
G=L_\star^{\nicefrac{1}{2}} Z R_\star^{\nicefrac{1}{2}}, \quad Z_{ij}\overset{\mathrm{i.i.d.}}{\sim}\mathcal{N}(0,1).
\label{eq:kronecker_gaussian_model}
\end{align}
Fixing $R=R_\star$ and replacing the expectation in \cref{eq:trace_power_bilateral_solutions} with the sample expectation over $N$ independent gradients $\smash{\{G_b=L_\star^{\nicefrac{1}{2}}Z_bR_\star^{\nicefrac{1}{2}}\}_{b=1}^N}$ gives the empirical estimate
\begin{align}
\widehat L_{\mathrm D} = \frac{1}{N\Tr(R_\star^{q_{\mathrm D}})} \sum_{b=1}^{N} G_b R_\star^{q_{\mathrm D}-1}G_b^\top = L_\star^{\nicefrac{1}{2}}\left[\frac{1}{N\Tr(R_\star^{q_{\mathrm D}})}\sum_{b=1}^{N}Z_bR_\star^{q_{\mathrm D}}Z_b^\top\right]L_\star^{\nicefrac{1}{2}}.
\label{eq:finite_sample_factor_estimator}
\end{align}
Analogously, we define the right-side counterpart $\widehat{R}_\text{D}$ of $\widehat{L}_\text{D}$ to be the empirically estimated right factor on $N$ samples when fixing the left factor to be $L_\star$.

\begin{proposition}[Finite-Sample Factor Stability]
\label[proposition]{prop:finite_sample_factor_stability}
Define the effective rank $r_{\mathrm{eff}}(A):= \nicefrac{\Tr(A)^2}{\Tr(A^2)}$ and the relative factors $\smash{\widetilde L_{\mathrm D}:=L_\star^{-\nicefrac{1}{2}}\widehat L_{\mathrm D}L_\star^{-\nicefrac{1}{2}}}$ and $\smash{\widetilde R_{\mathrm D}:=R_\star^{-\nicefrac{1}{2}}\widehat R_{\mathrm D}R_\star^{-\nicefrac{1}{2}}}$, which in the case of perfect estimation simplify to the identity matrix.
Under~\cref{eq:kronecker_gaussian_model}, the estimator in~\cref{eq:finite_sample_factor_estimator} and its right-factor counterpart, with the opposite factor fixed at its population value, satisfy
\begin{align}
\operatorname{Err}_{\mathrm D}^{L}:=\frac{1}{m}\mathbb E\|\widetilde L_{\mathrm D}-I\|_F^2=\frac{m+1}{N\,r_{\mathrm{eff}}(R_\star^{q_{\mathrm D}})},\quad \operatorname{Err}_{\mathrm D}^{R}:=\frac{1}{n}\mathbb E\|\widetilde R_{\mathrm D}-I\|_F^2=\frac{n+1}{N\,r_{\mathrm{eff}}(L_\star^{q_{\mathrm D}})}.
\label{eq:finite_sample_factor_error}
\end{align}
Since $r_{\mathrm{eff}}(A^q)$ is non-increasing in $q\geq0$, as shown in~\cref{fig:effective_rank}, left,
these errors satisfy
\begin{align}
\operatorname{Err}_{\mathrm{KL}}^{s}
\leq \operatorname{Err}_{\mathrm{SQ}}^{s}
\leq \operatorname{Err}_{\mathrm{VN}}^{s}
\leq \operatorname{Err}_{\mathrm F}^{s},
\quad s\in\{L,R\}.
\label{eq:finite_sample_error_ordering}
\end{align}
with equality when the opposite factor has a flat spectrum. See proof in \cref{app:proof_finite_sample}.
\end{proposition}

\noindent\textbf{Why Frobenius fitting has larger estimation error.}\hspace{1em}%
For a long-tailed $R_\star$, F ($q=2$) puts most of its weight on a few large eigenvalues, reducing the effective number of directions that average out sampling noise.
Equivalently, $R_\star^2$ has a smaller effective rank, leading to larger estimation error.
KL ($q=0$) weights all directions equally, giving the largest effective rank and the smallest error under the assumptions of \cref{prop:finite_sample_factor_stability}.

\noindent\textbf{Scope, limitations, and empirical evidence.}\hspace{1em}%
Here, our aim is to compare \emph{finite-sample factor stability} across F, VN, SQ, and KL using the common exponent $q$ introduced in~\cref{sec:trace_power}.
Joint Kronecker Gaussian estimation has been extensively studied in statistics.
In particular, \citet{mccormack2025information} characterize the asymptotic efficiency loss of VN relative to KL, while \citet{franks2026near} give high-probability error bounds for KL.
However, these results do not directly provide a finite-sample comparison across all four divergences in our setting.

\vspace{1.5ex plus 0.5ex minus .2ex}
\begin{wrapfigure}{r}{.55\linewidth}
\centering
\vspace{-15pt}
\includegraphics[width=\linewidth]{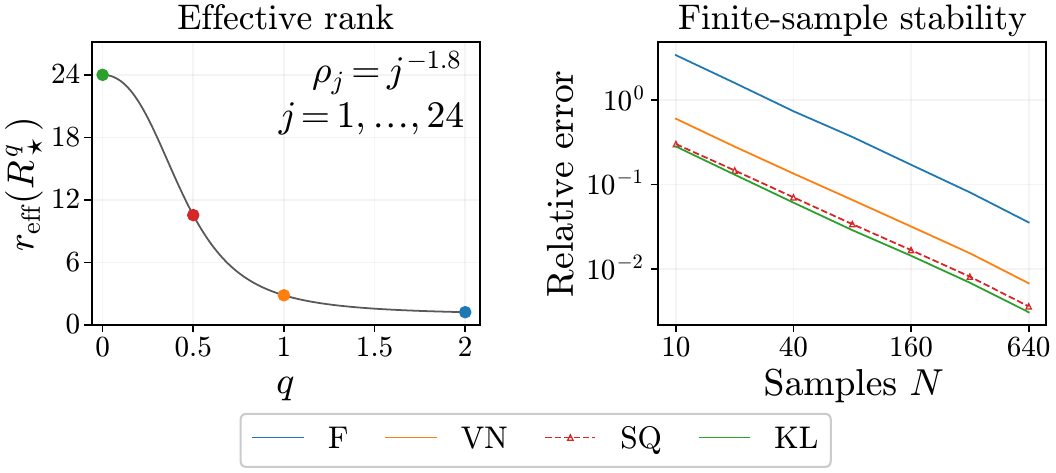}\par
\caption{
Left: Effective rank $r_{\mathrm{eff}}(R_\star^q)$ for eigenvalues $\rho_j=j^{-1.8}$.
Right: Mean squared relative Frobenius error of the jointly fitted $S=L\otimes R$ (\cref{eq:synthetic_stability_metric}) on non-Kronecker gradients over $20$ repeats, relative to each method's $4096$-sample fit.
Both results favor smaller $q_{\mathrm D}$, consistent with the theoretical prediction in~\cref{prop:finite_sample_factor_stability}.
}
\label{fig:effective_rank}
\end{wrapfigure}
Our analysis estimates one factor with the other fixed under an exact Kronecker Gaussian model.
This removes the coupling between factors and isolates the effect of spectral weighting, yielding a unified exact error formula for all four divergences.
To test whether the ordering persists under joint fitting beyond these assumptions, we fit both factors on synthetic long-tailed non-Kronecker Gaussian gradients.
The mean errors follow the predicted ordering at every tested sample size (\cref{fig:effective_rank}, right).
\cref{sec:real_finite_sample_stability} further examines the recovery of large-sample inverse-root gains on real gradients, without Gaussian or Kronecker assumptions.
\par\WFclear

\begin{checkpointcheck}
\textbf{Takeaway.}
With limited gradient samples, smaller $q_{\mathrm D}$ gives more accurate Kronecker factor estimates under our model, with KL achieving the lowest expected relative error.
The same trend also holds empirically beyond these assumptions.
This favors smaller $q_{\mathrm D}$ for preconditioners, where more stable factor estimates lead to reliable preconditioning from limited gradient histories.
\end{checkpointcheck}

\section{From Kronecker Approximations to Preconditioning}
\label{sec:from_approx_to_precondition}
Q1 characterizes divergence-dependent approximation error allocation, while Q2 shows how finite sampling affects both the empirical target and Kronecker factor estimation.
We now address Q3: \emph{how do divergence choice and finite sampling jointly affect preconditioning?}
We first examine whether large-sample Kronecker fits on real gradients follow the spectral predictions from Q1.
Then, we ask how these fits interact with the finite-sample bias identified in Q2 when converted into inverse-root gains, and whether these gains remain reliable with short gradient histories.

\subsection{Large-Sample Kronecker Fits}
\label{sec:real_kronecker_approximation}
We fit each Kronecker state $S_D^{\mathrm{ref}}$  to the empirical second moment $C_0$, using all gradients in split $0$.
These large-sample fits show how different divergences allocate approximation error and whether the observed results, shown in~\cref{fig:real_kronecker_approximation}, match the predictions in~\cref{sec:divergence_geometry}.

\begin{wrapfigure}{R}{.63\linewidth}
\vspace{-8pt}
\centering
\includegraphics[width=\linewidth]{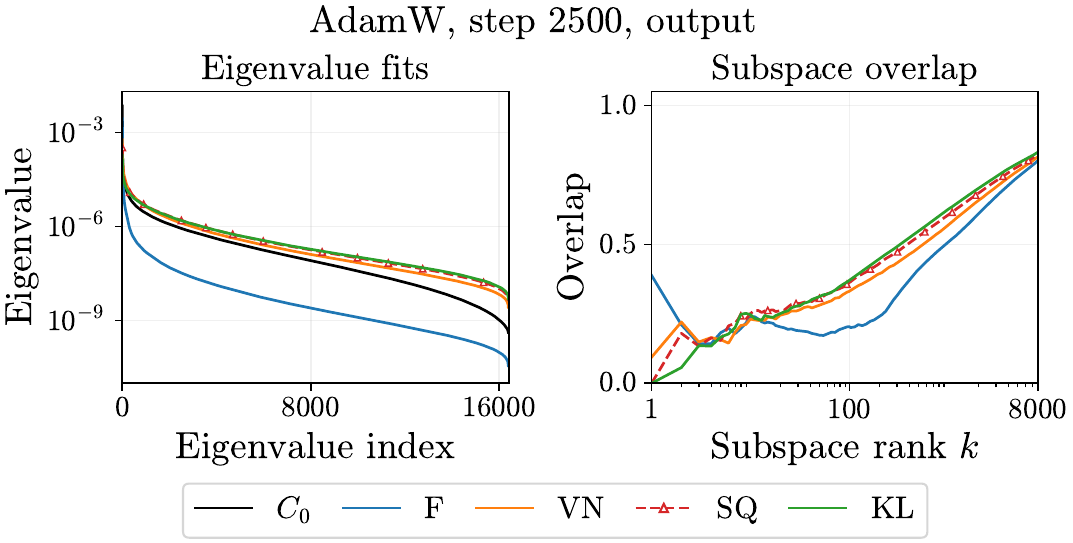}\par
\caption{Large-sample Kronecker fits. Left: empirical and fitted eigenvalues. Right: overlap of the leading subspaces. The fitted spectra and eigenspaces follow the divergence-dependent preferences predicted in~\cref{sec:divergence_geometry}.}
\label{fig:real_kronecker_approximation}
\end{wrapfigure}

\textbf{Eigenvalues.}
We compare empirical and fitted eigenvalues in decreasing order.
F underestimates much of the empirical spectrum, while VN follows it; SQ and KL have larger fitted tail eigenvalues. The same pattern holds across the other checkpoints and projections; see \cref{fig:real_divergence_spectrum}.

\textbf{Eigenspaces.}
Let $U_k$ and $V_{D,k}$ contain the leading $k$ eigenvectors of $C_0$ and $S_D^{\mathrm{ref}}$.
Following~\citet{song2025does}, we measure their subspace overlap by
$O_D(k):=\|U_k^\top V_{D,k}\|_{\mathrm{F}}^2/k \in [0,1]$, where $1$ means identical subspaces and $0$ means orthogonal subspaces.
F best matches the leading eigenvector, while KL achieves higher overlap for most larger subspaces. Since inverse-root gains depend on both eigenspace overlap and fitted eigenvalues, we examine their joint effect in~\cref{sec:real_inverse_root_action}.

\par\WFclear
\vspace{0.6\baselineskip}
\begin{checkpointcheck}
\textbf{Takeaway.}
The Kronecker fits reflect the spectral preferences predicted in \cref{sec:divergence_geometry}.
\end{checkpointcheck}

\subsection{Preconditioning Gains of Large-Sample Fits}
\label{sec:real_inverse_root_action}
The previous subsection shows that divergence choice leads to different eigenvalues and eigenspaces in the Kronecker fits.
To answer Q3, we now ask how these differences translate into preconditioning through $S_D^{-\nicefrac{1}{2}}$.
We compare the resulting preconditioning gains with held-out inverse-root scales to test whether each fit compensates for the finite-sample scale errors identified in~\cref{sec:real_second_moment_reliability}.

Following the held-out correction in \cref{sec:real_second_moment_reliability}, we use scales measured on independent data along the fixed source eigenvectors as the reference for preconditioning.
Let $\{(\lambda_{0,i},u_{0,i})\}$ be the eigenpairs of the source second moment $C_0$.
We average the three held-out directional scales,
$h_i:=\nicefrac{1}{3}\smash{\sum_{\heldout{s'}=1}^{3}u_{0,i}^{\top}\heldout{C_{s'}}u_{0,i}}$,
and use $\smash{h_i^{-\nicefrac{1}{2}}}$ as the reference gain along $u_{0,i}$.
For a fitted state $S_D$, we define the \textbf{held-out gain error} as
\begin{align}
e_{D,i}:=\log_{10}\frac{\|S_D^{-\nicefrac{1}{2}}u_{0,i}\|_2}{h_i^{-\nicefrac{1}{2}}}=\frac12\log_{10}\left(h_i u_{0,i}^\top S_D^{-1}u_{0,i}\right).
\label{eq:heldout_log_gain}
\end{align}
\textbf{Interpretation.}
Zero means matching the scaling factor $\nicefrac{1}{\sqrt{h_i}}$ estimated from held-out data; positive and negative values mean larger and smaller scales, respectively. \textbf{Our goal is to keep this error close to zero across the whole spectrum.}

\vspace{-5pt}
\begin{figure}[htbp]
\centering
\setlength{\abovecaptionskip}{2pt}
\includegraphics[width=\linewidth]{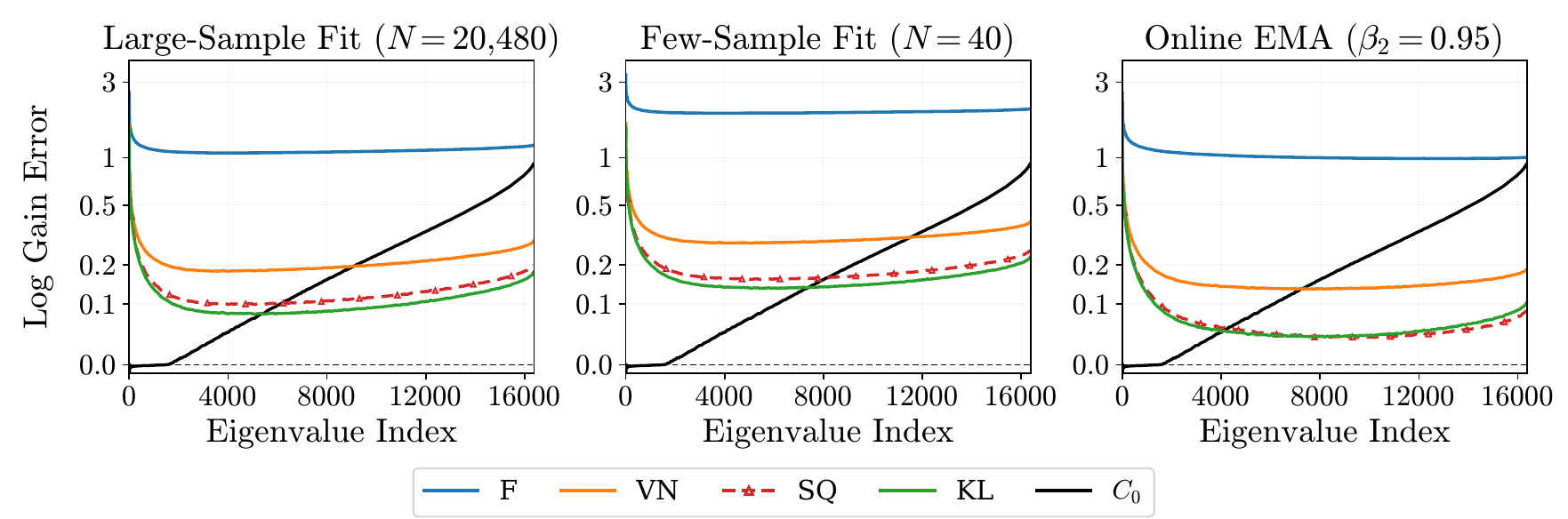}\par
\caption{Smoothed log held-out gain error for the output projection at AdamW step $2{,}500$. Closer to zero is better; with asinh vertical scale.
Left: large-sample fit; middle: $40$-sample fits ($4$ repeats); right: EMA ($16$ repeats).
Black solid: large-sample $\smash{C_0^{-\nicefrac{1}{2}}}$ in all panels; dashed: the held-out target.}
\label{fig:real_inverse_root_gain}
\end{figure}
\vspace{-5pt}

\textbf{Divergence choice shapes inverse-root gains.}
KL and SQ better compensate for empirical second-moment scaling error and stay closer to the held-out target across the whole spectrum, with KL achieving the lowest error (\cref{fig:real_inverse_root_gain}, left, and \cref{tab:real_heldout_gain}). Results across the remaining checkpoints and projections show the same pattern (\cref{app:real_gain_all}).

\begin{table}[htbp]
\centering
\caption{Absolute held-out gain error (\cref{eq:heldout_log_gain}, $\downarrow$). KL has the lowest mean error, followed by SQ.
Each experiment contributes the median $|e_{D,i}|$ over all directions and repeats; entries show mean $\pm$ sample SD across nine checkpoint-projection pairs per training optimizer.}
\label{tab:real_heldout_gain}
\small
\setlength{\tabcolsep}{6pt}
\begin{tabular}{llccc}
\toprule
\textbf{Optimizer} & \textbf{Method} & \textbf{Large-sample} & \textbf{40-sample} & \textbf{EMA} \\
\midrule
\multirow{4}{*}{AdamW}
& F & $0.9759\,\pm\,0.4623$ & $2.4945\,\pm\,0.8265$ & $0.7645\,\pm\,0.3453$ \\
& VN & $0.3328\,\pm\,0.1769$ & $0.5549\,\pm\,0.2986$ & $0.2897\,\pm\,0.1473$ \\
& SQ & $0.1467\,\pm\,0.0573$ & $0.2175\,\pm\,0.0893$ & $0.1294\,\pm\,0.0412$ \\
& KL & $\mathbf{0.1024}\,\pm\,0.0182$ & $\mathbf{0.1561}\,\pm\,0.0297$ & $\mathbf{0.1071}\,\pm\,0.0203$ \\
\midrule
\multirow{4}{*}{\makecell[l]{Distributed\\Shampoo}}
& F & $0.8801\,\pm\,0.3234$ & $2.6922\,\pm\,0.5674$ & $0.7008\,\pm\,0.2267$ \\
& VN & $0.3353\,\pm\,0.1109$ & $0.6207\,\pm\,0.1966$ & $0.2806\,\pm\,0.0935$ \\
& SQ & $0.1273\,\pm\,0.0285$ & $0.1876\,\pm\,0.0407$ & $0.1257\,\pm\,0.0300$ \\
& KL & $\mathbf{0.1003}\,\pm\,0.0177$ & $\mathbf{0.1482}\,\pm\,0.0249$ & $\mathbf{0.1072}\,\pm\,0.0204$ \\
\bottomrule
\end{tabular}
\end{table}
\vspace{-5pt}

\textbf{Link to divergence properties.}
The behavior above follows the spectral preferences in Section~\ref{sec:divergence_geometry}.
F and VN emphasize the leading spectrum and tend to underestimate the remaining eigenvalues, whereas KL weights errors uniformly across the spectrum and tends to overestimate; SQ lies in between.
This yields larger eigenvalues for SQ and KL, which can reduce excessive inverse-root amplification during preconditioning and brings the gains closer to the held-out targets (more details in \cref{sec:real_divergence_calibration}).

\begin{checkpointcheck}
\textbf{Takeaway.}
KL, closely followed by SQ, best compensates for the finite-sample underestimation in the empirical second moment, keeping inverse-root gains closer to the desired scaling across the spectrum.
\end{checkpointcheck}

\FloatBarrier
\subsection{Few-Sample Fits and Online EMA}
\label{sec:real_finite_sample_stability}
\label{sec:real_online_ema}
\setlength{\parskip}{0pt}
The large-sample behavior is useful in practice only if it can be recovered from short gradient histories.
We therefore evaluate both few-sample fitting and online EMA using two metrics: recovery error, which measures how closely each method reproduces its own large-sample gains, and the held-out gain error from \cref{sec:real_inverse_root_action}, which measures agreement with the held-out target.

\paragraph{Configuration.}
All four methods follow the EMA scheme of \citet{lin2026understanding}, detailed in \cref{app:online_bregman_shampoo}, with $\beta_2=0.95$ and eigendecomposition every five steps.
Equal-weight fits use $N=40$ gradients, close to the nominal EMA effective sample size $(1+\beta_2)/(1-\beta_2)=39$.\footnote{For EMA weights $w_k=(1-\beta_2)\beta_2^k$, the effective sample size is $N_{\mathrm{eff}}=1/\sum_k w_k^2=\nicefrac{(1+\beta_2)}{(1-\beta_2)}$.}
Each EMA repeat uses $512$ gradients to reduce initialization effects.
\cref{tab:real_gradient_settings} summarizes these settings.

\noindent\textbf{Recovery error.}
For this metric, the reference is each method's large-sample fit $\smash{S_D^{\mathrm{ref}}}$.
Let $\widehat S_D$ be the few-sample or EMA fit, and let $\{\omega_{D,j},v_{D,j}\}_{j=1}^d$ be the eigenpairs of $S_D^{\mathrm{ref}}$. We compare \textbf{inverse-root amplification} along the reference directions $v_{D,j}$:
\begin{align}
\gamma_{D,j}^{\mathrm{rec}}:=\log_{10} \frac{\|\widehat S_D^{-\nicefrac{1}{2}}v_{D,j}\|_2}{\|(S_D^{\mathrm{ref}})^{-\nicefrac{1}{2}}v_{D,j}\|_2}=\frac{1}{2}\log_{10}(\omega_{D,j}v_{D,j}^\top\widehat S_D^{-1}v_{D,j}).
\label{eq:same_divergence_stability_gain}
\end{align}
\textbf{Interpretation.} The gain mismatch $\gamma_{D,j}^{\mathrm{rec}}$ is zero when the large-sample scaling is recovered; positive and negative values indicate stronger and weaker amplification, respectively. \textbf{We also want to keep this error close to zero across the whole spectrum.}

\begin{table}[!htbp]
\centering
\caption{Recovery error (\cref{eq:same_divergence_stability_gain}, $\downarrow$). KL has the lowest mean error, followed by SQ.
Each experiment contributes the median $|\gamma_D^{\mathrm{rec}}|$ over all reference directions and repeats; entries show mean $\pm$ sample SD across nine checkpoint-projection pairs per training optimizer.}
\label{tab:real_state_stability}
\small
\setlength{\tabcolsep}{5pt}
\begin{tabular}{lcccc}
\toprule
& \multicolumn{2}{c}{\textbf{AdamW}} & \multicolumn{2}{c}{\textbf{Distributed Shampoo}} \\
\cmidrule(lr){2-3}\cmidrule(lr){4-5}
\textbf{Divergence} & \textbf{Few-sample} & \textbf{EMA} & \textbf{Few-sample} & \textbf{EMA} \\
\midrule
F & $1.668\,\pm\,0.601$ & $0.173\,\pm\,0.070$ & $1.967\,\pm\,0.430$ & $0.154\,\pm\,0.017$ \\
VN & $0.259\,\pm\,0.141$ & $0.082\,\pm\,0.032$ & $0.339\,\pm\,0.121$ & $0.087\,\pm\,0.017$ \\
SQ & $0.089\,\pm\,0.042$ & $0.063\,\pm\,0.020$ & $0.077\,\pm\,0.021$ & $0.048\,\pm\,0.008$ \\
KL & $\mathbf{0.066}\,\pm\,0.020$ & $\mathbf{0.044}\,\pm\,0.011$ & $\mathbf{0.060}\,\pm\,0.015$ & $\mathbf{0.040}\,\pm\,0.009$ \\

\bottomrule
\end{tabular}
\end{table}

\noindent\textbf{Finite-sample gain results.}
KL has the lowest recovery error, followed by SQ, VN, and F, across both $40$-sample fits and EMA and both checkpoint sources (\cref{tab:real_state_stability}).
The same ordering holds for aggregate held-out gain error (\cref{tab:real_heldout_gain,fig:real_inverse_root_gain}, middle and right).
Additional results in \cref{app:real_state_details} show the same trend.

\noindent\textbf{Connecting Q1 and Q2.}
Q1 shows that SQ and KL put more emphasis on the entire spectrum and produce larger fitted tail eigenvalues.
Q2 reveals two finite-sample effects: the empirical second moment underestimates the tail, while lower-$q_{\mathrm D}$ factor estimates are more stable with limited gradients.
Together, these effects explain why SQ and KL better compensate for the empirical second moment bias and recover their large-sample gains more reliably.

\begin{checkpointcheck}
\textbf{Takeaway.}
Divergence choice and finite sampling jointly shape preconditioning.
KL, closely followed by SQ, best compensates for tail underestimation and yields more reliable gains from short gradient histories.
\end{checkpointcheck}

\section{GPT-2 Training}
\label{sec:gpt2_training}
We next test whether the gain differences observed above are reflected in GPT-2 training.

\textbf{Prediction.}
More generally, we expect preconditioners whose inverse-root gains stay close to the held-out targets across the spectrum and remain reliable with short gradient histories to perform better in training.
This suggests the held-out gain and recovery errors in~\cref{sec:from_approx_to_precondition} as practical metrics for evaluating optimizer designs.
For the methods studied here, we therefore expect KL and SQ to perform similarly and better than F and VN, with a slight advantage for KL.

We compare F-, VN-, SQ-, and KL-Shampoo by their final validation losses after training $12$-layer GPT-2 models on FineWeb at several widths.
Across widths $256$--$1{,}536$, all methods use $5{,}100$ updates and spectral EMA ($\beta_2=0.95$). \cref{app:training_protocol} gives the full settings and training curves.

\begin{table}[htbp]
\centering
\begin{minipage}[c]{0.56\linewidth}
\centering
\small
\setlength{\tabcolsep}{4pt}
\begin{tabular}{@{}lccccc@{}}
\toprule
\textbf{Method / width} & \textbf{256} & \textbf{512} & \textbf{768} & \textbf{1024} & \textbf{1536} \\
\midrule
F & 3.8281 & 3.5082 & 3.3993 & 3.3528 & 3.3111 \\
VN & 3.7573 & 3.4377 & 3.2876 & 3.1979 & 3.0972 \\
SQ & \textbf{3.7450} & \textbf{3.4173} & 3.2661 & 3.1824 & 3.0817 \\
KL & 3.7475 & 3.4180 & \textbf{3.2660} & \textbf{3.1769} & \textbf{3.0758} \\
\bottomrule
\end{tabular}
\end{minipage}\hfill
\begin{minipage}[c]{0.41\linewidth}
\caption{Final GPT-2 validation loss on FineWeb after $5,100$ updates ($\downarrow$); one run per configuration.
SQ and KL outperform F and VN at every width. Bold marks the lowest loss in each column.}
\label{tab:gpt2_width_sweep}
\end{minipage}
\end{table}

SQ and KL achieve lower validation losses than F and VN at every width, consistent with this prediction (\cref{fig:intro_training_motivation}, right, and \cref{tab:gpt2_width_sweep}).
Their losses are similar at smaller widths, while KL achieves lower losses at the two largest widths.

\section{Conclusion}
\label{sec:conclusion}

We used a trace-power Bregman family to unify divergence-based Shampoo variants and connect SQ-Shampoo to PSGD-Kron, showing how divergence choice shapes Kronecker approximation error.
We then identified two finite-sample effects: empirical second moments can underestimate tail eigenvalues, while divergence choice also affects the stability of Kronecker factor estimation.
Together, these effects explain why SQ and KL better compensate for the underestimation of empirical second moments and recover their large-sample gains more reliably from short gradient histories.
Consistently, SQ and KL achieve lower GPT-2 validation losses than F and VN.
Our held-out gain and recovery errors provide practical metrics for evaluating new preconditioners.

\bibliography{conference}
\bibliographystyle{conference}

\appendix
\crefalias{section}{appendix}\crefalias{subsection}{subappendix}\crefalias{subsubsection}{subsubappendix}
\makeatletter
\setlength{\@fptop}{0pt}
\makeatother
\clearpage
\restorefloatspacing
\flushbottom
\RemoveFromHook{env/figure/begin}[main-float-spacing]
\RemoveFromHook{env/wrapfigure/begin}[main-wrap-spacing]
\RemoveFromHook{env/table/begin}[main-float-spacing]

\addtocontents{toc}{\protect\setcounter{tocdepth}{2}}
\begingroup
\renewcommand{\contentsname}{Appendix Contents}
\pdfbookmark[0]{Appendix Contents}{appendix.contents}
\tableofcontents
\endgroup
\clearpage

\section{Related Work}
\label{app:related_work}

\paragraph{Natural gradients and structured curvature.}
Natural gradient descent uses the Fisher metric to account for parameter-space geometry~\citep{amari1998natural}.
Layered Fisher approximations~\citep{heskes2000natural} and K-FAC~\citep{martens2015optimizing} make this approach tractable by exploiting network structure.
Subsequent work extends Kronecker factorizations to convolutional layers~\citep{grosse2016kronecker}, recurrent networks~\citep{martens2018kronecker}, and general linear layers with weight sharing~\citep{eschenhagen2023kronecker}.
Tensor Normal Training uses a tensor-normal model of sampled gradients to approximate the Fisher~\citep{ren2021tensor}, TKFAC imposes trace constraints~\citep{gao2021trace}, and SINGD incorporates inverse-free updates to avoid explicit matrix inversion~\citep{lin2024structured}.

\paragraph{Shampoo and Kronecker approximation.}
AdaGrad preconditions gradients using their accumulated second moments~\citep{duchi2011adaptive}, while Adam uses an exponential moving average of squared gradients for coordinate-wise scaling~\citep{kingma2014adam}.
Shampoo exploits tensor structure to construct a Kronecker-factored bound on full-matrix AdaGrad~\citep{gupta2018shampoo}.
Distributed computation, infrequent inverse-root updates, and learning-rate grafting make Shampoo practical at scale~\citep{anil2020scalable,agarwal2020disentangling,shi2023distributed}.
SOAP applies Adam in Shampoo's eigenbasis to improve training performance and stability~\citep{vyas2025soap}.
\citet{morwaninew2025} show that the Kronecker product of Shampoo's statistics corresponds to one power-iteration step from identity initialization toward the Frobenius-optimal Kronecker approximation.
In their experiments, its cosine similarity to the target closely tracks that of the optimal approximation.
They also examine how batch gradients and label sampling affect the estimation.

\paragraph{Matrix divergences and Kronecker estimation.}
Bregman matrix divergences provide alternatives to Frobenius approximation~\citep{dhillon2008matrix}, with applications to positive-semidefinite kernel learning~\citep{kulis2009low}.
For Kronecker approximation, the LogDet (KL divergence between two zero-mean multivariate Gaussian distributions) objective corresponds to matrix-normal maximum likelihood, which estimates row and column covariance factors alternately~\citep{dutilleul1999mle}.
\citet{franks2026near} establish nearly optimal finite-sample guarantees for joint maximum likelihood estimation without condition-number assumptions.
\citet{mccormack2025information} study partial-trace estimation, corresponding to VN fitting, and relate its statistical error and asymptotic efficiency loss to the factor spectra.
\citet{lin2026understanding} formulate gradient second-moment estimation through matrix divergences, yielding F-, VN-, and KL-Shampoo.
We study spectral error allocation and finite-sample stability across F, VN, SQ, and KL in a trace-power family.
Our fixed-factor analysis is complemented by joint fits on synthetic non-Kronecker and real gradients.

\paragraph{Reliability of empirical second moments.}
Empirical gradient second moments can differ substantially from curvature matrices like the Fisher and Hessian~\citep{kunstner2019limitations,thomas2020interplay}. Furthermore, with finite samples, empirical eigenvalues systematically misestimate population scales along their eigenvectors. Prior work addresses this bias through nonlinear covariance shrinkage or cross-validation~\citep{ledoit2012nonlinear,bartz2016cross} in high-dimensional covariance estimation.
A similar selection bias occurs in deep-learning Hessian and GGN estimates: directions with extreme curvature on one mini-batch have less extreme scales on others~\citep{tatzel2025debiasing}.
We examine these estimation effects in mini-batch gradient second moments and their consequences for divergence-dependent Kronecker preconditioning.

\clearpage

\section{Proofs}
\label{app:matrix_divergence_details}
Throughout the theoretical analysis, following \citet{eschenhagen2026purifying,lin2026understanding}, we assume the gradient second moment $C$ to be positive definite.
\paragraph{Matrix notation and divergence formulas.}
For $M=U\operatorname{diag}(\mu_1,\ldots,\mu_d)U^\top\in\mathbb S_{++}^d$, matrix functions act on its eigenvalues:
\begin{align}
f(M)=U\operatorname{diag}(f(\mu_1),\ldots,f(\mu_d))U^\top,\quad \Phi(M)=\sum_{i=1}^d f(\mu_i),\quad \nabla \Phi(M)=f'(M).
\label{eq:trace_spectral_function}
\end{align}
We use $\|A\|_F^2=\langle A,A\rangle_{\mathrm{F}}$ and write $\nabla^2\Phi(S)[H]$ for the Hessian on a perturbation $H$.
Substituting the generators $f_{\mathrm F}(x)=x^2/2$, $f_{\mathrm{VN}}(x)=x\log x-x$, and $f_{\mathrm{KL}}(x)=-\log x$ into \cref{eq:bregman_div} gives
\begin{align}
\mathcal B_F(X,Y)&=\frac{1}{2}\|X-Y\|_F^2,\quad \mathcal B_{\mathrm{VN}}(X,Y)=\Tr(X\log X-X\log Y-X+Y),\\
\mathcal B_{\mathrm{KL}}(X,Y)&=\Tr(Y^{-1}X)-\log\det(Y^{-1}X)-d.
\end{align}
In particular, $\mathcal B_{\mathrm{KL}}(X,Y)=2D_{\mathrm{KL}}(\mathcal N(0,X)\|\mathcal N(0,Y))$.
For row-wise vectorization, $\operatorname{vec}(AGB)=(A\otimes B^\top)\operatorname{vec}(G)$.

\subsection{Proof of \texorpdfstring{\cref{prop:trace_power_unification}}{Proposition \getrefnumber{prop:trace_power_unification}}}
\label{app:trace_power_proof}
For $q=2$, $f_2(x)=(x-1)^2/2$ and $f_2'(x)=x-1$.
Substitution into \cref{eq:bregman_div} gives
\begin{align}
\mathcal B_{\Phi_2}(X,Y)=\frac12\Tr(X^2-Y^2)-\Tr(Y(X-Y))=\frac12\|X-Y\|_F^2=\mathcal B_F(X,Y).
\end{align}
L'H\^opital's rule gives $f_q(x)\to x\log x-x+1$ as $q\to1$ and $f_q(x)\to-\log x+x-1$ as $q\to0$.
The derivative $f_q'(x)=(x^{q-1}-1)/(q-1)$ converges to $\log x$ and $1-x^{-1}$, respectively.
For fixed $X,Y\in\mathbb S_{++}^d$, we can obtain
\begin{align}
\lim_{q\to1}\mathcal B_{\Phi_q}(X,Y)=\mathcal B_{\Phi_1}(X,Y)&=\Tr(X\log X-X\log Y-X+Y)=\mathcal B_{\mathrm{VN}}(X,Y),\\
\lim_{q\to0}\mathcal B_{\Phi_q}(X,Y)=\mathcal B_{\Phi_0}(X,Y)&=\Tr(Y^{-1}X)-\log\det(Y^{-1}X)-d=\mathcal B_{\mathrm{KL}}(X,Y).
\end{align}

For $q=\frac{1}{2}$, which we call square root (SQ), $f_{\nicefrac{1}{2}}(x)=-4\sqrt{x}+2x+2$ and $f_{\nicefrac{1}{2}}'(x)=2-2x^{-\nicefrac{1}{2}}$.
Substituting into \cref{eq:bregman_div} yields
\begin{align}
\frac12\mathcal B_{\Phi_{\nicefrac{1}{2}}}(C,S)=\Tr(S^{\nicefrac{1}{2}})+\Tr(CS^{-\nicefrac{1}{2}})-2\Tr(C^{\nicefrac{1}{2}}).
\label{eq:psgd_sqrt_divergence}
\end{align}
Under $P=S^{-\nicefrac{1}{2}}$, this becomes
\begin{align}
\frac12\mathcal B_{\Phi_{\nicefrac{1}{2}}}(C,S)=\Tr(PC)+\Tr(P^{-1})-2\Tr(C^{\nicefrac{1}{2}}).
\end{align}
The first two terms are exactly the PSGD objective in \cref{eq:psgd_whitening_objective}; the last term is independent of $P$.
Thus SQ and PSGD have the same objective up to a positive scale and an additive constant.
For $S=L\otimes R$, the change of variables gives $P=L^{-\nicefrac{1}{2}}\otimes R^{-\nicefrac{1}{2}}$, so the correspondence also holds under the Kronecker constraint.

We now derive the factor equation in~\cref{eq:trace_power_bilateral_solutions}.
For $q\neq0,1$,
\begin{align}
\nabla \Phi_q(S)=\frac{S^{q-1}-I}{q-1}.
\end{align}
By substituting this into the Bregman divergence and keeping only the terms that depend on $S$, we can derive
\begin{align}
\mathcal B_{\Phi_q}(C,S)=\frac1q\Tr(S^q)-\frac1{q-1}\Tr(CS^{q-1})+\mathrm{const}.
\label{eq:trace_power_bregman_expansion}
\end{align}
For $S=L\otimes R$,
\begin{align}
\Tr(S^q)=\Tr(L^q)\Tr(R^q),
\end{align}
and the definition of $C=\mathbb{E}[gg^\top]$ gives
\begin{align}
\Tr(CS^{q-1})=\mathbb E\left[\Tr\left(L^{q-1}GR^{q-1}G^\top\right)\right].
\label{eq:trace_power_kron_cross_term}
\end{align}
Fix $R$, and write
\begin{align}
A_R:=\mathbb E[GR^{q-1}G^\top],\quad a_R:=\Tr(R^q).
\end{align}
Combining all the equations above, we can write the objective $\mathcal B_{\Phi_q}$ that depends on $L$ as
\begin{align}
J_q(L;R)=\frac{a_R}{q}\Tr(L^q)-\frac1{q-1}\Tr(A_RL^{q-1}).
\label{eq:trace_power_left_conditional}
\end{align}

For $q\neq0,1$, let $X=L^{q-1}$.
Then $L^q=X^{\nicefrac{q}{q-1}}$, and
\begin{align}
\nabla_X J_q=\frac{1}{q-1}\left(a_RX^{\nicefrac{1}{q-1}}-A_R\right)=\frac{1}{q-1}(a_RL-A_R).
\end{align}
Thus the stationary equation gives $L=A_R/a_R=\mathbb E[GR^{q-1}G^\top]/\Tr(R^q)$. The objective is strictly convex in $X$, so this stationary point is the unique minimizer.

For the VN defined by $q\to1$, $\Phi_1(S)=\Tr(S\log S-S+I)$.
The limit established above gives
\begin{align}
\mathcal B_{\Phi_1}(C,S)=\Tr(S)-\Tr(C\log S)+\mathrm{const}.
\end{align}
Using
\begin{align}
\Tr(L\otimes R)=\Tr(L)\Tr(R),\qquad
\log(L\otimes R)=\log L\otimes I_n+I_m\otimes\log R,
\end{align}
and the row-wise vectorization identity
\begin{align}
\operatorname{vec}(G)^\top(\log L\otimes I_n)\operatorname{vec}(G)
=\Tr(GG^\top\log L),
\end{align}
the terms depending on $L$ reduce to
\begin{align}
J_1(L;R)=\Tr(R)\Tr(L)-\Tr(A_R\log L),
\qquad A_R:=\mathbb E[GG^\top].
\end{align}
Here $R$ is fixed, so the term involving $I_m\otimes\log R$ is constant in $L$.
Set $X=\log L$. Then
\begin{align}
J_1(e^X;R)=\Tr(R)\Tr(e^X)-\Tr(A_RX),
\end{align}
whose gradient with respect to $X$ is
\begin{align}
\nabla_X J_1=\Tr(R)e^X-A_R.
\end{align}
Setting this to zero gives
$L=e^X=A_R/\Tr(R)$.
Since $\Tr(e^X)$ is strictly convex on symmetric matrices, this is the unique minimizer.

For the KL defined by $q\to0$, $\Phi_0(S)=-\log\det S+\Tr(S)-mn$.
Its affine terms cancel in the divergence, giving
\begin{align}
\mathcal B_{\Phi_0}(C,S)=\Tr(CS^{-1})+\log\det S+\mathrm{const}.
\end{align}
Using $S^{-1}=L^{-1}\otimes R^{-1}$ and $\log\det(L\otimes R)=n\log\det L+m\log\det R$, the terms that depend on $L$ are
\begin{align}
J_0(L;R)=\Tr(A_RL^{-1})+n\log\det L,\quad A_R:=\mathbb E[GR^{-1}G^\top].
\end{align}
Setting $X=L^{-1}$, we can derive:
\begin{align}
J_0(X^{-1};R)=\Tr(A_RX)-n\log\det X,
\qquad
\nabla_X J_0=A_R-nX^{-1}.
\end{align}
Setting the gradient to zero gives $X^{-1}=A_R/n$, hence
\begin{align}
L=\frac{A_R}{n}
=\frac{\mathbb E[GR^{-1}G^\top]}{\Tr(R^0)}.
\end{align}
Since $-\log\det X$ is strictly convex, this is the unique minimizer.

The cases $q=1$ and $q=0$ give the same formula, so all cases can be written as
\begin{align}
L_q(R)=\frac{\mathbb E[GR^{q-1}G^\top]}{\Tr(R^q)}.
\end{align}
Exchanging $L$ and $R$ gives
\begin{align}
R_q(L)=\frac{\mathbb E[G^\top L^{q-1}G]}{\Tr(L^q)}.
\end{align}

\subsection{Limited Kronecker Capacity}
\label{app:unavoidable_kron_error}

A single Kronecker product generally cannot represent a gradient second moment exactly.
The factors $L$ and $R$ have $m(m+1)/2+n(n+1)/2-1$ free parameters after accounting for reciprocal rescaling, fewer than the $mn(mn+1)/2$ parameters of a general symmetric matrix when $m,n\geq2$.
Exact matching is possible for second moments with Kronecker structure; otherwise, the choice of divergence determines how approximation errors are allocated.

\addtocontents{toc}{\protect\setcounter{tocdepth}{3}}
\subsection{Spectral Error Decomposition and Kronecker Approximations}
\label{sec:proof_spectral_scale_preference}

\subsubsection{Proof of~\cref{prop:spectral_scale_preference}}

\paragraph{Notation and spectral costs.}
Write $U=[u_1,\ldots,u_d]$, $V=[v_1,\ldots,v_d]$, $\Lambda=\operatorname{diag}(\lambda_1,\ldots,\lambda_d)$, and $\Omega=\operatorname{diag}(\omega_1,\ldots,\omega_d)$, with both spectra ordered decreasingly.
The scalar divergence and cost on eigenspace misalignment are
\begin{align}
d_f(\lambda,\omega)&:=f(\lambda)-f(\omega)-f'(\omega)(\lambda-\omega),\\
\Delta_f(\Lambda,\Omega;P)&:=-\sum_{i,j=1}^d\lambda_i f'(\omega_j)(P_{ij}-\delta_{ij}),\quad P_{ij}=\langle u_i,v_j\rangle^2,
\end{align}
where $\delta_{ij}$ is the Kronecker delta.

\paragraph{Spectral decomposition.}
By the definition of the Bregman matrix divergence,
\begin{align}
\mathcal{B}_\Phi(A,B) = \Phi(A)-\Phi(B) - \Tr \left( \nabla \Phi(B)(A-B) \right).
\label{eq:proof_bregman_definition}
\end{align}
Since $\Phi(M)=\Tr(f(M))$ is a spectral function, its gradient is given by $\nabla \Phi(M)=f'(M)$.
Therefore, for
$B=V\Omega V^\top$,
\begin{align}
\nabla \Phi(B) = Vf'(\Omega)V^\top.
\label{eq:proof_spectral_gradient}
\end{align}

Let $Q=U^\top V$.
Then $Q$ is orthogonal and
$Q_{ij}=\langle u_i,v_j\rangle$, so that
$P_{ij}=Q_{ij}^2$.
Using the cyclic property of the trace,
\begin{align}
\Tr\left(\nabla \Phi(B)A\right) = \Tr \left( Vf'(\Omega)V^\top U\Lambda U^\top \right)= \Tr \left( f'(\Omega)Q^\top\Lambda Q \right)= \sum_{i=1}^{d}\sum_{j=1}^{d} \lambda_i f'(\omega_j)P_{ij}.
\label{eq:proof_cross_trace}
\end{align}
Similarly, using $V^\top V=I$,
\begin{align}
\Tr\left(\nabla \Phi(B)B\right) = \Tr \left( Vf'(\Omega)V^\top V\Omega V^\top \right)=\Tr \left( f'(\Omega)V^\top V\Omega V^\top V\right)=\sum_{j=1}^{d}\omega_j f'(\omega_j).
\label{eq:proof_self_trace}
\end{align}
Substituting \cref{eq:proof_cross_trace,eq:proof_self_trace} into \cref{eq:proof_bregman_definition} gives
\begin{align}
\mathcal{B}_\Phi(A,B) = \sum_{i=1}^{d}f(\lambda_i) - \sum_{j=1}^{d}f(\omega_j) + \sum_{j=1}^{d}\omega_jf'(\omega_j) - \sum_{i=1}^{d}\sum_{j=1}^{d} \lambda_i f'(\omega_j)P_{ij}.
\label{eq:proof_expanded_divergence}
\end{align}
Adding and subtracting
$\sum_{i=1}^{d}\lambda_i f'(\omega_i)$ yields
\begin{align}
\mathcal{B}_\Phi(A,B) &= \sum_{i=1}^{d} \left[ f(\lambda_i)-f(\omega_i) -f'(\omega_i)(\lambda_i-\omega_i) \right]  + \sum_{i=1}^{d}\lambda_i f'(\omega_i) - \sum_{i=1}^{d}\sum_{j=1}^{d} \lambda_i f'(\omega_j)P_{ij}\\
&= \sum_{i=1}^{d}d_f(\lambda_i,\omega_i) + \sum_{i=1}^{d}\sum_{j=1}^{d} -\lambda_i f'(\omega_j) \left(P_{ij}-\delta_{ij}\right)  \\
&= \sum_{i=1}^{d}d_f(\lambda_i,\omega_i) + \Delta_f(\Lambda,\Omega;P).
\label{eq:proof_spectral_decomposition}
\end{align}
The first term measures eigenvalue mismatch for aligned eigenspaces; $\Delta_f$ is the cost of eigenspace misalignment.

\paragraph{Nonnegativity of $\Delta_f(\Lambda,\Omega;P)$.}
Since $Q$ is orthogonal, $P_{ij}=Q_{ij}^2$ is doubly stochastic.
The Birkhoff--von Neumann theorem therefore gives
$P=\sum_k\alpha_k\Pi_k$, where $\Pi_k$ are permutation matrices,
$\alpha_k\geq0$, and $\sum_k\alpha_k=1$.
Both $\lambda_i$ and $f'(\omega_i)$ are decreasing, since the eigenvalues are ordered and $f$ is convex.
The rearrangement inequality says that pairing them in this order maximizes their product sum.
Thus, for each permutation $\sigma_k$ associated with $\Pi_k$,
\begin{align}
\sum_i\lambda_i f'(\omega_{\sigma_k(i)})
\leq \sum_i\lambda_i f'(\omega_i).
\end{align}
Taking the weighted average gives
\begin{align}
\sum_{i,j}\lambda_i f'(\omega_j)P_{ij}
\leq \sum_i\lambda_i f'(\omega_i),
\label{eq:proof_rearrangement_detailed}
\end{align}
which proves $\Delta_f(\Lambda,\Omega;P)\geq0$.

\subsubsection{Divergence-Specific Spectral Penalties}
We overlook the affine terms.
For Frobenius, $f_{\mathrm{F}}(x)=\frac{1}{2}x^2$ and $f_{\mathrm{F}}'(x)=x$, which gives
\begin{align}
d_{\mathrm{F}}(\lambda,\omega) = \frac{1}{2}(\lambda-\omega)^2.
\label{eq:proof_frobenius_terms}
\end{align}
For von Neumann, $f_{\mathrm{VN}}(x)=x\log x-x$ and $f_{\mathrm{VN}}'(x)=\log x$, and hence
\begin{align}
d_{\mathrm{VN}}(\lambda,\omega) = \lambda\log\frac{\lambda}{\omega} -\lambda+\omega.
\label{eq:proof_vn_terms}
\end{align}
For the square-root member, $f_{\mathrm{SQ}}(x)=-4\sqrt{x}$ and $f_{\mathrm{SQ}}'(x)=-2x^{-\nicefrac{1}{2}}$, giving
\begin{align}
d_{\mathrm{SQ}}(\lambda,\omega)=2\left(\sqrt{\omega}+\frac{\lambda}{\sqrt{\omega}}-2\sqrt{\lambda}\right).
\label{eq:proof_sqrt_terms}
\end{align}
For LogDet, $f_{\mathrm{KL}}(x) = -\log x$ and $f_{\mathrm{KL}}'(x)= -\nicefrac{1}{x}$, which yields
\begin{align}
d_{\mathrm{KL}}(\lambda,\omega) = \log\frac{\omega}{\lambda} + \frac{\lambda}{\omega} -1.
\label{eq:proof_logdet_terms}
\end{align}
Setting $\omega=r\lambda$ gives \cref{eq:proof_relative_spectral_penalties}, which establishes the scalar scale dependence in \cref{eq:relative_spectral_penalties}.

By using the exponents $q_{\mathrm D}$ in \cref{prop:trace_power_unification}, Frobenius and KL give
\begin{align}
\Delta_{\mathrm{F}}(r\Lambda,r\Omega;P) &= r^2\Delta_{\mathrm{F}}(\Lambda,\Omega;P), & \Delta_{\mathrm{KL}}(r\Lambda,r\Omega;P) &= \Delta_{\mathrm{KL}}(\Lambda,\Omega;P).
\label{eq:proof_f_ld_scale}
\end{align}
For von Neumann, using
$\sum_{j=1}^{d}(P_{ij}-\delta_{ij})=0$, we obtain
\begin{align}
\Delta_{\mathrm{VN}}(r\Lambda,r\Omega;P) &= -r \sum_{i=1}^{d}\sum_{j=1}^{d} \lambda_i \log(r\omega_j) \left(P_{ij}-\delta_{ij}\right)  \\ &= -r\log r \sum_{i=1}^{d} \lambda_i \sum_{j=1}^{d} \left(P_{ij}-\delta_{ij}\right) + r\Delta_{\mathrm{VN}}(\Lambda,\Omega;P)  \\ &= r\Delta_{\mathrm{VN}}(\Lambda,\Omega;P).
\label{eq:proof_vn_scale}
\end{align}
For SQ, $f'_{\mathrm{SQ}}(\omega)=-2\omega^{-\nicefrac{1}{2}}$ gives $\Delta_{\mathrm{SQ}}(r\Lambda,r\Omega;P)=r^{\nicefrac{1}{2}}\Delta_{\mathrm{SQ}}(\Lambda,\Omega;P)$ directly.
Therefore, for
$\mathrm{D}\in\{\mathrm{F},\mathrm{VN},\mathrm{SQ},\mathrm{KL}\}$,
\begin{align}
\Delta_{\mathrm{D}}(r\Lambda,r\Omega;P) = r^{q_{\mathrm D}} \Delta_{\mathrm{D}}(\Lambda,\Omega;P), \quad r>0.
\label{eq:proof_eigenspace_scale_homogeneity}
\end{align}

\subsubsection{Reciprocal Asymmetry}
Substituting $r$ and $r^{-1}$ into \cref{eq:proof_relative_spectral_penalties} proves \cref{eq:over_under_ordering}.
For F, $(r-1)^2>(r^{-1}-1)^2$ when $r>1$; for SQ, the two penalties are equal.
For VN and KL, the differences are $\lambda h(r)$ and $-h(r)$, respectively, where $h(r):=r-r^{-1}-2\log r$.
Since $h(1)=0$ and $h'(r)=(r-1)^2/r^2>0$ for $r>1$, VN penalizes overestimation more, while KL penalizes underestimation more.

\subsubsection{Kronecker Approximations}
\label{sec:kronecker_spectral_penalties}
Let $C\in\mathbb S_{++}^{mn}$ have eigenpairs $(\lambda_i,u_i)$ ordered by decreasing eigenvalue, and let $L$ and $R$ have eigenpairs $(\ell_a,v_{L,a})$, $a=1,\ldots,m$, and $(\rho_b,v_{R,b})$, $b=1,\ldots,n$, respectively.
Sort the $mn$ eigenvalue products $\ell_a\rho_b$ in decreasing order.
If the $j$-th product comes from the $a$-th eigenvalue of $L$ and the $b$-th eigenvalue of $R$, write $a_j=a$ and $b_j=b$.
The corresponding eigenpair of $S=L\otimes R$ is
\begin{align}
\omega_j=\ell_{a_j}\rho_{b_j},\qquad v_j=v_{L,a_j}\otimes v_{R,b_j},\qquad j=1,\ldots,mn.
\end{align}
Set $\Lambda=\operatorname{diag}(\lambda_i)$, $\Omega=\operatorname{diag}(\omega_j)$, and $P_{ij}=|\langle u_i,v_j\rangle|^2$ as in \cref{prop:spectral_scale_preference}.
We can obtain:
\begin{align}
\mathcal B_{\Phi_q}(C,L\otimes R)=\underbrace{\sum_{i=1}^{mn}d_q(\lambda_i,\ell_{a_i}\rho_{b_i})}_{\text{eigenvalue mismatch}}+\underbrace{\Delta_q(\Lambda,\Omega;P)}_{\text{eigenspace misalignment}},\quad \Delta_q(\Lambda,\Omega;P)\geq0.
\end{align}
For Kronecker approximations, both mismatch terms retain the scaling relations in \cref{eq:relative_spectral_penalties}, and the scalar penalties still satisfy the reciprocal-error comparisons in \cref{eq:over_under_ordering}.

\addtocontents{toc}{\protect\setcounter{tocdepth}{2}}
\subsection{Proof of \texorpdfstring{\cref{prop:finite_sample_factor_stability}}{Proposition \getrefnumber{prop:finite_sample_factor_stability}}}
\label{app:proof_finite_sample}
We only prove the result for the left factor; the right-factor result follows symmetrically.
Let $A:=R_\star^{q_{\mathrm D}}$, $\widetilde L_{\mathrm D} := L_\star^{-\nicefrac{1}{2}}\widehat L_{\mathrm D}L_\star^{-\nicefrac{1}{2}}$.
By~\cref{eq:finite_sample_factor_estimator}, we have
\begin{align}
\widetilde L_{\mathrm D} = \frac{1}{N}\sum_{b=1}^{N} M_b, \quad M_b:= \frac{Z_b A Z_b^\top}{\Tr(A)}.
\end{align}
Let $z_i^\top$ denote the $i$-th row of $Z_b$, so $z_i\sim\mathcal N(0,I_n)$.
Then
\begin{align}
\mathbb{E}\left[(Z_bAZ_b^\top)_{ij}\right] = \mathbb{E}[z_i^\top A z_j] =
\begin{cases}
\Tr(A), & i=j,\\ 0, & i\neq j,
\end{cases}
\end{align}
where we use $\mathbb{E}[z_i z_i^\top]=I$ and the independence of different rows.
Hence, $\mathbb{E}[Z_bAZ_b^\top]=\Tr(A)I$ and $\mathbb{E}[M_b]=I$.
Using $\widetilde L_{\mathrm D}-I=\nicefrac{1}{N}\sum_{b=1}^N(M_b-I)$, we obtain
\begin{align}
\mathbb{E}\left[\left\|\widetilde L_{\mathrm D}-I\right\|_F^2\right] = \frac{1}{N^2} \sum_{b,c=1}^{N} \mathbb{E}\left[ \left\langle M_b-I,M_c-I\right\rangle_{\mathrm{F}} \right].
\end{align}
For $b\neq c$, independence together with $\mathbb{E}[M_b-I]=0$ gives
\begin{align}
\mathbb{E}\left[ \left\langle M_b-I,M_c-I\right\rangle_{\mathrm{F}} \right]=0.
\end{align}
Hence, only the $N$ diagonal terms remain, yielding
\begin{align}
\mathbb{E}\left[\left\|\widetilde L_{\mathrm D}-I\right\|_F^2\right] = \frac{1}{N} \mathbb{E}\left[\left\|M_1-I\right\|_F^2\right].
\label{eq:proof_sample_average_variance}
\end{align}
Let $z_i^\top$ denote the $i$-th row of $Z_1$.
For the diagonal entries, $(M_1)_{ii}=\frac{z_i^\top A z_i}{\Tr(A)}$.
Since $z_i\sim\mathcal N(0,I_n)$ and $A$ is symmetric, the Gaussian quadratic-form identity $\operatorname{Var}(z_i^\top A z_i)=2\Tr(A^2)$ gives
\begin{align}
\mathbb{E}\left[\big((M_1)_{ii}-1\big)^2\right] = \frac{2\Tr(A^2)}{\Tr(A)^2}.
\end{align}
For $i\neq j$, using the independence of $z_i$ and $z_j$ and
$\mathbb{E}[z_jz_j^\top]=I$,
\begin{align}
\mathbb{E}\left[(z_i^\top A z_j)^2\right] = \mathbb{E}\left[z_i^\top A\,\mathbb{E}[z_jz_j^\top]\,A z_i\right] = \Tr(A^2),
\end{align}
and therefore
\begin{align}
\mathbb{E}\left[(M_1)_{ij}^2\right] = \frac{\Tr(A^2)}{\Tr(A)^2}.
\end{align}
Summing over the $m$ diagonal and $m(m-1)$ off-diagonal entries yields
\begin{align}
\mathbb{E}\left[\left\|M_1-I\right\|_F^2\right] = m(m+1) \frac{\Tr(A^2)} {\Tr(A)^2} = \frac{m(m+1)}{r_{\mathrm{eff}}(A)}.
\end{align}
Combining this with~\cref{eq:proof_sample_average_variance} and dividing by $m$ gives
\begin{align}
\mathbb{E}\left[ \frac{1}{m} \left\| L_\star^{-\nicefrac{1}{2}}\widehat L_{\mathrm D}L_\star^{-\nicefrac{1}{2}}-I \right\|_F^2 \right] = \frac{m+1} {N\,r_{\mathrm{eff}}(R_\star^{q_{\mathrm D}})}.
\end{align}
The right-factor identity follows by exchanging $L_\star,m$ with
$R_\star,n$.

It remains to compare the four exponents $q_{\mathrm D}$.
For positive eigenvalues $\{\rho_i\}$, define $h(q):=\log\sum_i\rho_i^q$ and $p_i(q):=\rho_i^q/\sum_j\rho_j^q$.
Then
\begin{align}
h''(q)=\sum_i p_i(q)(\log\rho_i)^2-\left(\sum_i p_i(q)\log\rho_i\right)^2\geq0,
\end{align}
so $h'$ is non-decreasing.
Noting that the log effective rank is $\log r_{\mathrm{eff}}(R_\star^q) = 2h(q) - h(2q)$, we have
\begin{align}
\frac{d}{dq}\log r_{\mathrm{eff}}(R_\star^q)=2h'(q)-2h'(2q)\leq0,\quad q\geq0.
\end{align}
Thus effective rank is decreasing in $q_D$.
Evaluating at $q=q_{\mathrm D}$ gives
\begin{align}
r_{\mathrm{eff}}(R_\star^2)\leq r_{\mathrm{eff}}(R_\star)\leq r_{\mathrm{eff}}(R_\star^{\nicefrac{1}{2}})\leq r_{\mathrm{eff}}(I)=n.
\end{align}
Since the expected estimation error is proportional to the reciprocal of effective rank, we obtain
\begin{align}
\operatorname{Err}_{\mathrm{KL}} \leq \operatorname{Err}_{\mathrm{SQ}} \leq \operatorname{Err}_{\mathrm{VN}} \leq \operatorname{Err}_{\mathrm F}.
\end{align}
The inequalities become equalities for an isotropic opposite-factor spectrum.

\subsection{Local Geometry of SQ and KL}
\label{app:psgd_kl_comparison}
\label{sec:local_divergence_geometry}
We compare the local penalties of SQ and KL through their Hessians.
KL matches the Gaussian Fisher geometry up to a constant factor; SQ does not.

For $S\in\mathbb S_{++}^d$ and symmetric $H$, Taylor expansion gives
\begin{align}
\mathcal B_{\Phi_{q_{\mathrm D}}}(S+\varepsilon H,S)=\frac{\varepsilon^2}{2}\left\langle H,\nabla^2\Phi_{q_{\mathrm D}}(S)[H]\right\rangle_{\mathrm{F}}+o(\varepsilon^2).
\label{eq:trace_power_local_geometry}
\end{align}
Thus the Hessian determines the penalty for small perturbations.
For eigenvalue perturbations, $f_{q_{\mathrm D}}''(s)=s^{q_{\mathrm D}-2}$, giving $s^{-\nicefrac{3}{2}}$ for SQ and $s^{-2}$ for KL.

To compare this geometry with Gaussian Fisher information, consider the auxiliary model $x\sim\mathcal N(0,S)$ with density function $p_S$.
The directional score for a symmetric perturbation $H$ is
\begin{align}
s_H(x):=D_S\log p_S(x)[H]=\frac12\left(x^\top S^{-1}HS^{-1}x-\Tr(S^{-1}H)\right).
\end{align}
Its second moment gives the Fisher quadratic form:
\begin{align}
\mathbb E_{x\sim p_S}[s_H(x)^2]=\frac12\Tr(S^{-1}HS^{-1}H)=\langle H,\mathcal I_S[H]\rangle_{\mathrm{F}},\quad \mathcal I_S[H]:=\frac12S^{-1}HS^{-1}.
\label{eq:gaussian_covariance_fisher}
\end{align}
The KL generator satisfies
\begin{align}
\nabla^2\Phi_0(S)[H]=S^{-1}HS^{-1}=2\mathcal I_S[H].
\label{eq:logdet_fisher_hessian}
\end{align}
Hence its Hessian equals twice the Gaussian Fisher information operator~\citep{lin2026understanding}.
SQ does not share the Gaussian Fisher geometry.
At $S=sI$, its Hessian is $s^{-3/2}H$, whereas the Gaussian Fisher operator is $H/(2s^2)$.
Their ratio, $2\sqrt{s}$, varies with $s$, so they cannot agree.
Among the four divergences, only KL has a Hessian proportional to the Gaussian Fisher operator for every $S\succ0$.

\clearpage

\section{Preconditioner Update Details}
\label{app:shampoo_variant_updates}

We describe the factor updates of original Shampoo, PSGD-Kron, SQ-Shampoo, and KL-Shampoo, then specify the factor statistics, offline fitting procedure, and online EMA scheme used in our experiments.

\subsection{Original Shampoo}
\label{app:original_shampoo_updates}
In the EMA form of Shampoo~\citep{gupta2018shampoo,anil2020scalable}, the marginal statistics satisfy
\begin{align}
L_t=\beta_2L_{t-1}+(1-\beta_2)G_tG_t^\top,\quad R_t=\beta_2R_{t-1}+(1-\beta_2)G_t^\top G_t.
\label{eq:shampoo_factor_updates}
\end{align}
Ignoring momentum, the parameter update is $\Theta_{t+1}=\Theta_t-\eta_tL_t^{-p}G_tR_t^{-p}$.
This applies the inverse square root of $S_t=L_t^{2p}\otimes R_t^{2p}$: $p=\nicefrac{1}{4}$ for original Shampoo, while $p=\nicefrac{1}{2}$ applies the inverse square root of $L_t\otimes R_t$.

\subsection{PSGD-Kron, SQ-Shampoo, and KL-Shampoo}
\label{app:psgd_updates}
\label{app:objective_geometry_updates}

Under $P=S^{-\nicefrac{1}{2}}$, SQ-Shampoo shares PSGD's population gradient-whitening objective (\cref{eq:psgd_sqrt_divergence}), but uses different factor updates.

\paragraph{Factor updates.}
PSGD parameterizes the preconditioner as $P=Q^\top Q$ and optimizes $Q$ on a matrix Lie group~\citep{li2019preconditioner}.
In the Kronecker form,
\begin{align}
Q=Q_L\otimes Q_R,\quad P=P_L\otimes P_R,\quad P_L:=Q_L^\top Q_L,\quad P_R:=Q_R^\top Q_R.
\label{eq:psgd_kron_parameterization}
\end{align}
With row-wise vectorization, the parameter update is $\Theta_{t+1}=\Theta_t-\eta_tP_{L,t}G_tP_{R,t}$.
For the original relative-gradient coordinates $dQ=\mathcal E Q$, the factors follow
\begin{align}
Q_{L,t+1}=(I-\mu_t\mathcal R_{L,t})Q_{L,t},\quad Q_{R,t+1}=(I-\mu_t\mathcal R_{R,t})Q_{R,t},
\label{eq:psgd_kron_relative_gradient}
\end{align}
where $\mathcal R_{L,t},\mathcal R_{R,t}$ are the factor-wise relative gradients of the whitening objective in \cref{eq:psgd_whitening_objective}.
The \href{https://github.com/lixilinx/psgd_torch}{official implementation} also provides other update geometries, including $dQ=Q\mathcal E Q$.

Our SQ- and KL-Shampoo implementations follow the spectral EMA scheme of \citet{lin2026understanding}.
Both average the factor statistics over time, update eigenvalue estimates in the current eigenbases, and periodically refresh bases.
The two methods differ in the statistics supplied to this scheme: SQ uses the square-root divergence statistics, while KL uses the KL statistics in \cref{tab:variant_statistics}.
\cref{alg:online_bregman_shampoo} gives the full update.

\begin{table}[H]
\centering
\caption{Objectives and update schemes discussed here.
Objectives are expressed in $S$, up to positive scale and additive constants; the PSGD row refers to gradient whitening.}
\label{tab:objective_geometry_updates}
\small
\begin{tabularx}{\linewidth}{@{}l l X@{}}
\toprule
\textbf{Method} & \textbf{Population objective} & \textbf{Factor updates discussed here} \\
\midrule
PSGD-Kron & SQ divergence & Original relative-gradient steps; other geometries are available \\
SQ-Shampoo & SQ divergence & Spectral EMA of SQ statistics \\
KL-Shampoo & KL divergence & Spectral EMA of KL statistics \\
\bottomrule
\end{tabularx}
\end{table}

\subsection{Divergence-Based Factor Statistics}
\cref{tab:variant_statistics} gives the factor statistics for F-, VN-, SQ-, and KL-Shampoo, obtained by setting $q=q_{\mathrm D}$ in \cref{eq:trace_power_bilateral_solutions}.
These statistics define both the offline and online updates that follow.
For SQ-Shampoo, $q=\nicefrac{1}{2}$ gives the bilateral fixed-point iteration
\begin{align}
L^{(k+1)}=\frac{\mathbb E[G(R^{(k)})^{-\nicefrac{1}{2}}G^\top]}{\Tr((R^{(k)})^{\nicefrac{1}{2}})},\quad R^{(k+1)}=\frac{\mathbb E[G^\top(L^{(k)})^{-\nicefrac{1}{2}}G]}{\Tr((L^{(k)})^{\nicefrac{1}{2}})}.
\label{eq:sq_shampoo_bilateral_updates}
\end{align}

\begin{table}[H]
\centering
\caption{Instantaneous factor statistics for the Shampoo variants.}
\label{tab:variant_statistics}
\small
\setlength{\tabcolsep}{5pt}
\begin{tabular}{lll}
\toprule
\textbf{Method} & \textbf{\boldmath $\Psi_L(G;L,R)$} & \textbf{\boldmath $\Psi_R(G;L,R)$} \\
\midrule
F-Shampoo & $GRG^\top/\Tr(R^2)$ & $G^\top LG/\Tr(L^2)$ \\
VN-Shampoo & $GG^\top/\Tr(R)$ & $G^\top G/\Tr(L)$ \\
SQ-Shampoo & $GR^{-\nicefrac{1}{2}}G^\top/\Tr(R^{\nicefrac{1}{2}})$ & $G^\top L^{-\nicefrac{1}{2}}G/\Tr(L^{\nicefrac{1}{2}})$ \\
KL-Shampoo & $GR^{-1}G^\top/n$ & $G^\top L^{-1}G/m$ \\
\bottomrule
\end{tabular}
\end{table}

\subsection{Offline Fixed-Point Fits}

Given $N$ gradients, the offline bilateral rule for F-, VN-, SQ-, and KL-Shampoo is
\begin{align}
\widetilde L^{(k+1)}_{\mathrm D}:=\frac{1}{N}\sum_{b=1}^{N}\Psi_{\mathrm D,L}(G_b;L_{\mathrm D}^{(k)},R_{\mathrm D}^{(k)}),\quad \widetilde R^{(k+1)}_{\mathrm D}:=\frac{1}{N}\sum_{b=1}^{N}\Psi_{\mathrm D,R}(G_b;L_{\mathrm D}^{(k)},R_{\mathrm D}^{(k)}),
\label{eq:offline_bilateral_updates}
\end{align}
followed by under-relaxation,
\begin{align}
L_{\mathrm D}^{(k+1)}=(1-\alpha)L_{\mathrm D}^{(k)}+\alpha\widetilde L_{\mathrm D}^{(k+1)},\quad R_{\mathrm D}^{(k+1)}=(1-\alpha)R_{\mathrm D}^{(k)}+\alpha\widetilde R_{\mathrm D}^{(k+1)}.
\label{eq:offline_bilateral_relaxation}
\end{align}
We initialize $L^{(0)}=I_m$ and $R^{(0)}=I_n$, use $\alpha=0.3$, and iterate to numerical convergence.

\paragraph{Remark on few-sample fitting.}
We later consider fits with $N<mn$, where the empirical second moment $\widehat C_N:=N^{-1}\sum_{b=1}^Ng_bg_b^\top$ is necessarily singular.
This does not affect the iterations above: they require only positive-definite factors.
For KL, the updates correspond to the criterion $\Tr(\widehat C_NS^{-1})+\log\det S$, with $S=L\otimes R$, which remains finite for singular $\widehat C_N$.

\subsection{Online EMA Updates}
\label{app:online_bregman_shampoo}
Following the KL-Shampoo implementation of \citep{lin2026understanding}, the optimizer maintains factor accumulators $(A_L,A_R)$ for the eigenbases and separate eigenvalue estimates $(\lambda_L,\lambda_R)$.
Writing $\bar L=Q_L\operatorname{Diag}(\lambda_L)Q_L^\top$ and $\bar R=Q_R\operatorname{Diag}(\lambda_R)Q_R^\top$, \cref{alg:online_bregman_shampoo} applies to $\mathrm{D}\in\{\mathrm{F},\mathrm{VN},\mathrm{SQ},\mathrm{KL}\}$.

\begin{algorithm}[H]
\caption{Online EMA update for a Shampoo variant}
\label{alg:online_bregman_shampoo}
\begin{algorithmic}[1]
\REQUIRE Matrix gradients $G_t$, current momentum $M_t$, method $\mathrm{D}$, second-moment EMA decay $\beta_2$, basis period $T$, and initialized state $(A_L,A_R,Q_L,Q_R,\lambda_L,\lambda_R)$
\FOR{$t=1,2,\ldots$}
\STATE Form $\bar L=Q_L\operatorname{Diag}(\lambda_L)Q_L^\top$ and $\bar R=Q_R\operatorname{Diag}(\lambda_R)Q_R^\top$
\STATE Precondition $M_t$ as $\widetilde M_t\leftarrow Q_L\operatorname{Diag}(\lambda_L^{-\nicefrac{1}{2}})Q_L^\top M_tQ_R\operatorname{Diag}(\lambda_R^{-\nicefrac{1}{2}})Q_R^\top$
\STATE $\Delta_L\leftarrow\Psi_{\mathrm{D},L}(G_t;\bar L,\bar R)$, $\Delta_R\leftarrow\Psi_{\mathrm{D},R}(G_t;\bar L,\bar R)$
\STATE $A_L\leftarrow\beta_2A_L+(1-\beta_2)\Delta_L$, $A_R\leftarrow\beta_2A_R+(1-\beta_2)\Delta_R$
\STATE $\lambda_L\leftarrow\beta_2\lambda_L+(1-\beta_2)\operatorname{diag}(Q_L^\top\Delta_LQ_L)$
\STATE $\lambda_R\leftarrow\beta_2\lambda_R+(1-\beta_2)\operatorname{diag}(Q_R^\top\Delta_RQ_R)$
\IF{$t\bmod T=0$}
\STATE Refresh $Q_L,Q_R$ from $A_L,A_R$ by eigendecomposition/QR, while preserving $\lambda_L,\lambda_R$
\ENDIF
\STATE Return the preconditioned statistic $\widetilde M_t$
\ENDFOR
\end{algorithmic}
\end{algorithm}

The input $M_t=\beta_1M_{t-1}+(1-\beta_1)G_t$ is the current momentum, with $M_t=G_t$ when momentum is disabled.
It is preconditioned before $G_t$ is added into the preconditioner statistics.
The coefficient $\beta_2$ controls the EMA of both the factor accumulators and eigenvalue estimates.
In the EMA diagnostics, $T=5$ for all four methods.

\section{Synthetic Validation}
\label{app:synthetic_details}

We examine spectral error allocation and finite-sample stability using a smooth, long-tailed second moment that is not exactly Kronecker.

\subsection{Experimental Setup}
\label{app:synthetic_common}

For matrix gradients $G_b\in\mathbb{R}^{m\times n}$, the empirical gradient second moment is
\begin{align}
\widehat C:=\frac{1}{N}\sum_{b=1}^{N}g_bg_b^\top,\quad g_b:=\operatorname{vec}_{\mathrm{row}}(G_b).
\label{eq:appendix_empirical_second_moment}
\end{align}
F-, VN-, SQ-, and KL-Shampoo share samples and use the offline updates in \cref{app:shampoo_variant_updates}.
\cref{tab:synthetic_settings} shows the detailed settings.

\begin{table}[htbp]
\centering
\caption{Settings for the synthetic experiments.}
\label{tab:synthetic_settings}
\small
\setlength{\tabcolsep}{6pt}
\begin{tabular}{lccc}
\toprule
\textbf{Setting} & \textbf{\boldmath $m\times n$} & \textbf{\boldmath $N$} & \textbf{Additional parameters} \\
\midrule
Smooth long-tailed & $16\times24$ & $4096$ & scale $=10^{-1}$, $\sigma=0.02$ \\
\bottomrule
\end{tabular}
\end{table}

\paragraph{Smooth long-tailed spectrum.}
\label{app:smooth_synthetic}

We combine three matrix-normal components with different eigenvalue decay rates and orientations:
\begin{align}
G_b:=10^{-1}\left(\sum_{k=1}^{3}w_kL_kZ_{bk}R_k^\top+\sigma E_b\right),
\label{eq:appendix_smooth_generator}
\end{align}
where $Z_{bk}$ and $E_b$ have i.i.d.\ standard Gaussian terms.
The weights $(w_1,w_2,w_3)=(1,0.3,0.15)$ give the components different strengths.
We control their eigenvalue decay through
\begin{align}
L_k:=U_k\operatorname{Diag}\bigl(i^{-\nicefrac{p_k}{2}}\bigr)_{i=1}^{m},\quad R_k:=V_k\operatorname{Diag}\bigl(j^{-\nicefrac{p_k}{2}}\bigr)_{j=1}^{n},
\end{align}
where $U_k,V_k$ are independently sampled orthogonal matrices and $(p_1,p_2,p_3)=(1.8,1.2,0.8)$.
The second-moment factors $L_kL_k^\top$ and $R_kR_k^\top$ have eigenvalues $i^{-p_k}$ and $j^{-p_k}$, respectively, so larger $p_k$ gives faster decay.
The resulting second moment is generally not a single Kronecker product.
We also add isotropic noise with $\sigma=0.02$ and scale the gradients by $10^{-1}$.

\paragraph{Fitting.}
We generate a fixed pool of $4096$ gradients.
All methods start from $L=I_{16}$ and $R=I_{24}$ and use a relaxation coefficient of $0.3$.
We stop when the sum of the two factors' relative Frobenius changes falls below $10^{-10}$, or after $10{,}000$ iterations.

\subsection{Metrics and Results}

For $S:=L\otimes R$, let $\lambda_i(M)$ denote descending eigenvalues and $U_{M,k}$ the leading $k$ eigenvectors.
We measure cumulative eigenvalue error and leading-subspace overlap by
\begin{align}
E(k):=\frac{1}{k}\sum_{i=1}^{k}\left|\log_{10}\frac{\lambda_i(S)}{\lambda_i(\widehat C)}\right|,\quad O(k):=\frac{1}{k}\left\|U_{\widehat C,k}^{\top}U_{S,k}\right\|_F^2.
\label{eq:synthetic_metrics}
\end{align}

\paragraph{Finite-sample stability.}
\label{app:synthetic_finite_sample}
For each divergence $\mathrm D$, let $S_{\mathrm D}^{\mathrm{ref}}$ be its fit to the full $4096$-sample pool and $S_{N,r}^{\mathrm D}$ its fit to $N$ samples in repeat $r$.
We measure the squared relative Frobenius error
\begin{align}
E_{\mathrm{rel},r}^{\mathrm D}(N):=\frac{1}{mn}\left\|(S_{\mathrm D}^{\mathrm{ref}})^{-\nicefrac{1}{2}}S_{N,r}^{\mathrm D}(S_{\mathrm D}^{\mathrm{ref}})^{-\nicefrac{1}{2}}-I_{mn}\right\|_F^2.
\label{eq:synthetic_stability_metric}
\end{align}
This error is invariant to reciprocal factor rescaling.
If $S_{\mathrm D}^{\mathrm{ref}}=L_\star\otimes R_\star$ and $S_{N,r}^{\mathrm D}=\widehat L_{\mathrm D}\otimes R_\star$, it reduces to $\nicefrac{1}{m}\|\widetilde L_{\mathrm D}-I_m\|_F^2$, whose expectation is $\operatorname{Err}_{\mathrm D}^{L}$ in \cref{prop:finite_sample_factor_stability}.
Here both factors are estimated jointly, and the reference is each method's large-sample Kronecker fit.

We use $20$ random permutations of the same pool and fit the first $N\in\{10,20,40,80,160,320,640\}$ gradients.
\cref{fig:effective_rank} (right) plots $\nicefrac{1}{20}\sum_{r=1}^{20}E_{\mathrm{rel},r}^{\mathrm D}(N)$.

\cref{fig:smooth_validation} shows that F prioritizes large eigenvalues but severely underestimates small ones; VN follows more mildly.
SQ and KL have larger fitted tail eigenvalues.
The mean stability errors follow KL $<$ SQ $<$ VN $<$ F at every tested $N$, consistent with the theoretical ordering despite jointly fitting the non-Kronecker second moment.

\begin{figure*}[htbp]
\centering
\includegraphics[width=0.32\textwidth]{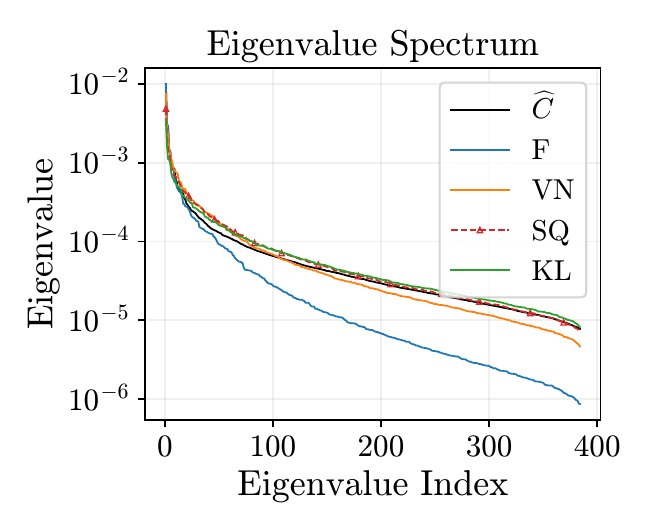}
\hfill
\includegraphics[width=0.32\textwidth]{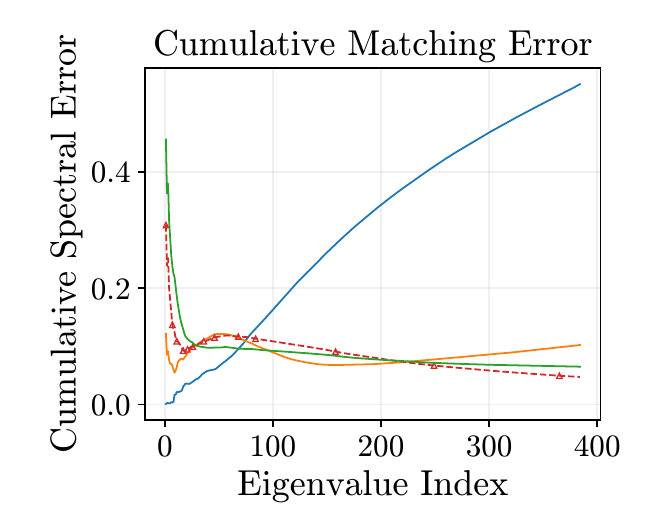}
\hfill
\includegraphics[width=0.32\textwidth]{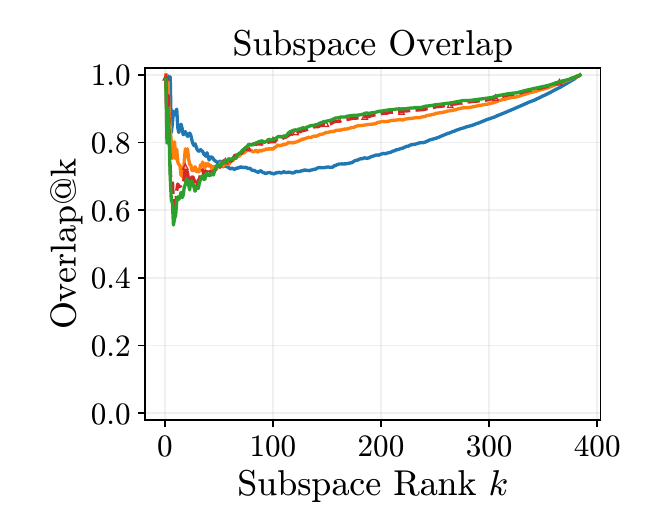}
\caption{Smooth synthetic setting. From left to right: empirical and fitted spectra, cumulative spectral error, and leading-subspace overlap. Finite-sample stability is shown in \cref{fig:effective_rank} (right).}
\label{fig:smooth_validation}
\end{figure*}

\FloatBarrier
\subsection{Directional Gain Interventions on a Quadratic}
\label{app:quadratic_gain_interventions}
\label{sec:quadratic_gain_interventions}
We test whether gain differences explain the methods' training performance by replacing KL's leading or tail gains with those of F or SQ in a stochastic quadratic.

\paragraph{Setup.}
To model stochastic gradient noise, we use linear regression with squared loss $\ell(x;z,y)=\tfrac12(z^\top x-y)^2$, where $z\sim\mathcal N(0,H)$ and $y\sim\mathcal N(0,\nicefrac{1}{6})$ are independent.
Its population loss is $\mathbb E[\ell(x;z,y)]=\tfrac12x^\top Hx+1/12$, so it is equivalent to the quadratic $f(x)=\tfrac12x^\top Hx$. 

We set $d=384$ and $H=dC_\mathrm{true}/\Tr(C_\mathrm{true})$, where $C_\mathrm{true}$ is the population second moment of the generator in \cref{eq:appendix_smooth_generator}.
Each minibatch gradient averages $B=64$ sample gradients.
Its population second moment is $C(x)=\mathbb E[g_B(x)g_B(x)^\top]$, with expectation over minibatch draws:
\begin{align}
g_B(x)=\frac1B\sum_{b=1}^B z_b(z_b^\top x-y_b),\qquad C(x)=\frac{(x^\top Hx+1/6)H}{B}+\left(1+\frac1B\right)(Hx)(Hx)^\top.
\end{align}
At each step, we fit the Kronecker factors to $40$ independent mini-batch gradients reshaped into $16\times24$, then update with a separate mini-batch.
We fit the factors using the offline updates in \cref{app:shampoo_variant_updates}, with no EMA.

We run F, VN, SQ, and KL for $1{,}000$ steps at a shared learning rate $\eta=2\times10^{-3}$ (\cref{fig:quadratic_training}).
Within each seed, all methods share the same initialization and mini-batches.
We initialize $x_0=H^{-\nicefrac{1}{2}}v$, with a Gaussian vector $v$ normalized to norm $2$, giving $f(x_0)=2$.

\textbf{Gain replacement.}
At the current $x$, fit KL and the comparison method $\mathrm D\in\{\mathrm F,\mathrm{SQ}\}$ to the same gradients.
Write $C(x)=U\Lambda U^\top$ in decreasing eigenvalue order and $P_{\mathrm D}=(L_{\mathrm D}\otimes R_{\mathrm D})^{-\nicefrac{1}{2}}$.
For selected directions $\mathcal I$, set
\begin{align}
r_i=\begin{cases}\|P_{\mathrm D}u_i\|/\|P_{\mathrm{KL}}u_i\|,&i\in\mathcal I,\\1,&i\notin\mathcal I,\end{cases}\qquad \widetilde P=P_{\mathrm{KL}}U\operatorname{Diag}(r_i)U^\top.
\end{align}
We replace KL's gains in the first $39$ (top 10\%) or last $96$ (bottom 25\%) directions by rescaling each $P_{\mathrm{KL}}u_i$ to match $\|P_{\mathrm D}u_i\|$.
Starting from the same KL iterate at step $900$, we continue to step $1{,}000$ using the original training minibatches and learning rate, recomputing the factors and population eigenbasis at each step.

\noindent\textbf{Results.}
F reaches the highest loss (\cref{fig:quadratic_gain_interventions}, left).
Averaged over the last $20$ updates and $3$ seeds, KL achieves a loss of $0.00844$.
Replacing KL's leading or tail gains with F's raises the loss to $1.441$ and $0.03333$, respectively (middle and right).
The corresponding SQ replacements yield $0.00857$ and $0.00855$, close to KL.

\begin{figure}[htbp]
\centering
\includegraphics[width=\linewidth]{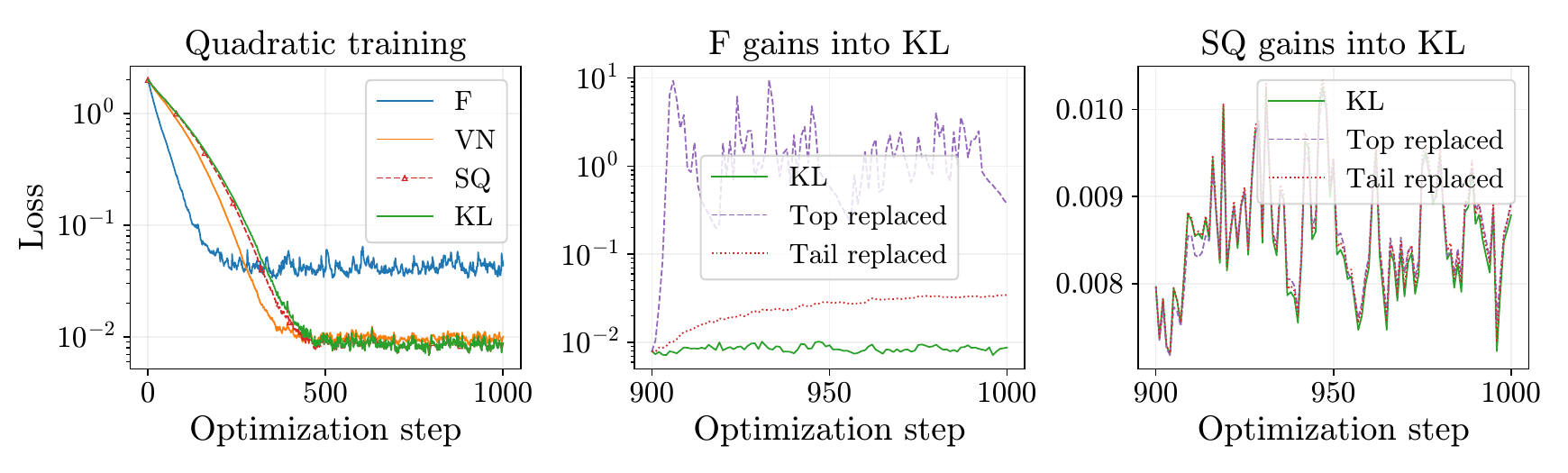}\par
\caption{Quadratic training loss averaged over $3$ seeds. Left: four methods. Middle and right: replacing KL gains from step $900$. F gains raise loss in either region; SQ gains leave it close to KL.}
\label{fig:quadratic_gain_interventions}
\label{fig:quadratic_training}
\end{figure}

\FloatBarrier
\newenvironment{appendixfigure}[3]{%
  \par\begingroup
  \renewcommand{\LTcaptype}{figure}%
  \setlength{\LTcapwidth}{\linewidth}%
  \setlength{\LTpre}{6pt}%
  \setlength{\LTpost}{8pt}%
  \begin{longtable}{@{}p{\linewidth}@{}}
  \caption{#2}\label{#1}\\
  \endfirsthead
  \caption[]{#3 (continued).}\\
  \endhead
}{%
  \end{longtable}
  \endgroup
}

\section{Real-Gradient Experimental Details}
\label{app:real_model_protocol}
We analyze $128\times128$ query, value, and attention-output gradients from block 5 of a 12-layer, four-head GPT-2 model trained with AdamW.
We freeze the model at steps 100, 2500, and 4500, giving nine checkpoint-projection pairs.
\cref{tab:real_gradient_settings} summarizes the shared settings; \cref{app:distributed_shampoo_checkpoints} extends the diagnostics to Distributed Shampoo checkpoints.

\begingroup
\setlength{\intextsep}{7pt}
\begin{table}[H]
\centering
\caption{Shared settings for real-gradient diagnostics.}
\label{tab:real_gradient_settings}
\small
\renewcommand{\arraystretch}{1.12}
\begin{tabularx}{\linewidth}{@{}l>{\raggedright\arraybackslash}X@{}}
\toprule
\textbf{Setting} & \textbf{Specification} \\
\midrule
Gradient pools & Four splits per experiment; $20{,}480$ FP32 mini-batch gradients per split, assigned by round-robin interleaving \\
Gradient collection & Batch size $64$; sequence length $512$  \\
Large-sample fit & All $20{,}480$ gradients in split $0$ \\
Few-sample fit & $40$ gradients from split $0$; four repeats \\
Fixed-point solver & Relaxation $\alpha=0.3$; tolerance $10^{-8}$; at most $1{,}000$ steps \\
Online EMA & $16$ disjoint chronological repeats from split $0$, each with $512$ gradients; $\beta_2=0.95$; eigendecomposition every five steps \\
\bottomrule
\end{tabularx}
\end{table}
\endgroup

All four methods retain their eigenvalue EMAs during each eigenbasis refresh.
Reliability uses all splits. For gain calibration, split $0$ supplies the fits and source eigenvectors; splits $1$, $2$, and $3$ supply the held-out directional scales in \cref{sec:real_inverse_root_action}.

\paragraph{Numerical details.}
Second moments and eigendecompositions use FP64.
We sort eigenvalues in decreasing order and retain those exceeding $d\times2^{-52}$ times the largest eigenvalue, where $d=16{,}384$ and $2^{-52}$ is FP64 machine precision.
For reliability summaries, retained directions are grouped by rank into top $0$--$10\%$, middle $10$--$40\%$, lower $40$--$75\%$, and tail $75$--$100\%$.

\subsection{Empirical Second Moments and Finite-Sample Effects}
\label{app:real_selection}
\addtocontents{toc}{\protect\setcounter{tocdepth}{3}}


\subsubsection{Finite-Sample Tail Selection}
\label{app:finite_sample_tail_selection}
\paragraph{Cross-fit control.}
We take eigenvectors $u_{b,i}$ of $C_b$ from a third split $b$ and compare their scales on splits $s$ and $t$:
\begin{align}
\xi^{\mathrm{cf}}_{b;s,t,i}:=\log_{10}\frac{u_{b,i}^\top C_tu_{b,i}}{u_{b,i}^\top C_su_{b,i}},\quad b\notin\{s,t\}.
\label{eq:crossfit_signed_ratio}
\end{align}
The direction $u_{b,i}$ is therefore selected independently of both second-moment estimates being compared.
We compare the six pairs with $s<t$, using each remaining split to select directions.
This gives $12$ comparisons without reverse pairs that would cancel the signed log-ratios.
Cross-fit scale errors have medians and IQRs near zero across the spectrum (\cref{fig:real_reliability_all}), supporting sampling-induced selection bias rather than distribution shift.

\paragraph{Selection bias.}
Another independent split estimates the population second moment along each source eigenvector.
The identity below shows that underestimating the source eigenvalue raises the expected held-out ratio, with larger relative errors giving larger ratios.
Let $\rho_{s\to t,i}:=10^{\xi_{s\to t,i}}$ denote the ratio before taking the logarithm in \cref{eq:heldout_second_moment_ratio}.

\begin{proposition}[Finite-Sample Tail Selection]
\label[proposition]{prop:heldout_reliability}
Let $C_s=C+E_s$ and $C_t=C+E_t$ be independent empirical gradient second moments drawn from the same distribution, where $E_s$ and $E_t$ denote their zero-mean finite-sample estimation errors.
For a source eigenpair $C_su_{s,i}=\lambda_{s,i}u_{s,i}$ with $\lambda_{s,i}>0$,
\begin{align}
\mathbb{E}\!\left[\rho_{s\rightarrow t,i}\mid C_s\right]=\frac{u_{s,i}^\top C u_{s,i}}{\lambda_{s,i}}=1-\frac{u_{s,i}^\top E_su_{s,i}}{\lambda_{s,i}}.
\label{eq:heldout_selection_bias}
\end{align}
Thus, if the source error along this direction is negative, $u_{s,i}^\top E_su_{s,i}<0$, its conditional expected held-out ratio is larger than one.
\end{proposition}

\begin{proof}
Conditional on $C_s$, the vector $u_{s,i}$ and the denominator $\lambda_{s,i}$ are fixed.
Independence and $\mathbb E[C_t]=C$ give $\mathbb E[\rho_{s\to t,i}\mid C_s]=u_{s,i}^\top Cu_{s,i}/\lambda_{s,i}$.
Substituting $C=C_s-E_s$ yields \cref{eq:heldout_selection_bias}.
\end{proof}

\paragraph{Why ratios rise toward the tail.}
Selecting directions with small sample eigenvalues can favor downward sampling fluctuations, so their scales may increase when evaluated on an independent split.
In \cref{eq:heldout_selection_bias}, the downward estimation error is divided by the source eigenvalue $\lambda_{s,i}$.
Small tail eigenvalues can therefore turn downward errors into large held-out ratios.

\subsubsection{High-Dimensional Sampling and Tail Signals}
\label{app:high_dimensional_sampling}
Finite-sample spectral distortion is well documented in high-dimensional covariance estimation~\citep{ledoit2012nonlinear,bartz2016cross}, and similar selection bias has also been observed in mini-batch Hessian and GGN estimates~\citep{tatzel2025debiasing}.
With $d/N=0.8$, the spectral distortion observed here may partly be explained by Mar\v{c}enko--Pastur-type spreading~\citep{bai2010spectral}. 
However, this does not explain why the leading scales are much more consistent across splits than the tail. 
Prior work has reported alignment between the leading subspaces of the empirical second moment and the Hessian~\citep{zhu2019anisotropic,wu2022alignment,liu2026rethinking}. 
This may help explain the stronger reproducibility of the leading scales, although the connection remains to be established.

\paragraph{Tail directions and learning signals.}
\label{app:weak_signals}
If class $k$ occurs with probability $p_k$ and produces gradient $g=a_ku_k$, with orthonormal $u_k$, then $C=\sum_kp_ka_k^2u_ku_k^\top$.
For comparable $|a_k|$, infrequent classes correspond to directions with small second-moment scales.
This illustrates how tail directions can carry useful signals~\citep{wang2026muon,kim2026sharp}.
However, finite-sample noise can still severely underestimate their scales.

\addtocontents{toc}{\protect\setcounter{tocdepth}{2}}
\subsection{Reliability Across Checkpoints}
\label{app:real_reliability_all}
\cref{fig:real_reliability_all} shows held-out scale error for all $9$ experiments.
\cref{tab:real_reliability_regions_full} reports all four regions, including absolute held-out scale error.
We take the median over all directions in each region and all $12$ comparisons within an experiment, then report the mean and sample standard deviation across the nine checkpoint-projection pairs per training optimizer.

\begin{table}[!htbp]
\centering
\caption{Held-out scale error across all four eigenvalue regions. Zero indicates agreement; positive $\xi$ indicates underestimation, and larger $|\xi|$ indicates greater mismatch. Entries show mean $\pm$ sample SD across nine checkpoint-projection pairs per training optimizer, using each's regional median.}
\label{tab:real_reliability_regions_full}
\small
\setlength{\tabcolsep}{5pt}
\begin{tabular}{lcccc}
\toprule
\multicolumn{5}{c}{\textbf{AdamW checkpoints}} \\
\midrule
\textbf{Metric} & \textbf{Top (0--10\%)} & \textbf{Middle (10--40\%)} & \textbf{Lower (40--75\%)} & \textbf{Tail (75--100\%)} \\
\midrule
$|\xi|$ ($\downarrow$) & $0.027\,\pm\,0.016$ & $0.124\,\pm\,0.037$ & $0.431\,\pm\,0.030$ & $1.009\,\pm\,0.120$ \\
Signed $\xi$ & $0.010\,\pm\,0.030$ & $0.124\,\pm\,0.037$ & $0.431\,\pm\,0.030$ & $1.009\,\pm\,0.120$ \\
Cross-fit $\xi^{\mathrm{cf}}$ & $-0.008\,\pm\,0.004$ & $-0.006\,\pm\,0.004$ & $-0.006\,\pm\,0.005$ & $-0.007\,\pm\,0.006$ \\
\midrule
\multicolumn{5}{c}{\textbf{Distributed Shampoo checkpoints}} \\
\midrule
\textbf{Metric} & \textbf{Top (0--10\%)} & \textbf{Middle (10--40\%)} & \textbf{Lower (40--75\%)} & \textbf{Tail (75--100\%)} \\
\midrule
$|\xi|$ ($\downarrow$) & $0.020\,\pm\,0.009$ & $0.104\,\pm\,0.030$ & $0.429\,\pm\,0.025$ & $1.048\,\pm\,0.017$ \\
Signed $\xi$ & $0.002\,\pm\,0.020$ & $0.104\,\pm\,0.030$ & $0.429\,\pm\,0.025$ & $1.048\,\pm\,0.017$ \\
Cross-fit $\xi^{\mathrm{cf}}$ & $-0.009\,\pm\,0.005$ & $-0.008\,\pm\,0.004$ & $-0.007\,\pm\,0.004$ & $-0.007\,\pm\,0.004$ \\

\bottomrule
\end{tabular}
\end{table}

\begin{appendixfigure}{fig:real_reliability_all}{Held-out scale error.
Rows: steps 100, 2500, 4500.
Bands show split-comparison IQRs.}{Held-out scale error}
\begin{minipage}{\linewidth}
\centering
\textbf{Step 100}\par\vspace{2pt}
\begin{minipage}[t]{.3\linewidth}
\centering
\textbf{Query}\\[2pt]
\includegraphics[width=\linewidth]{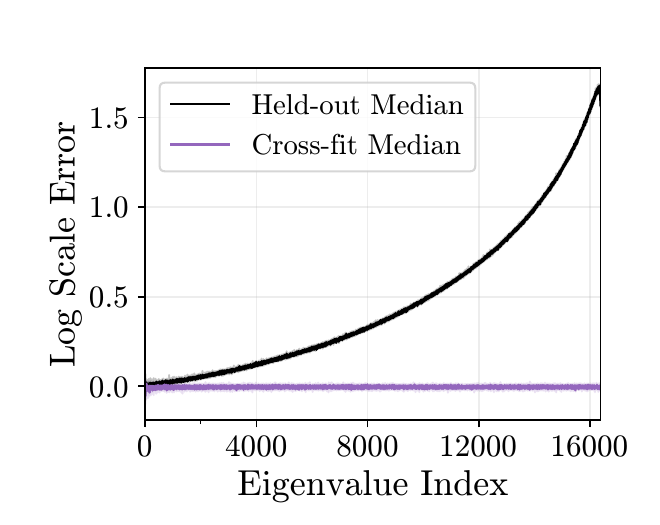}
\end{minipage}\hfill
\begin{minipage}[t]{.3\linewidth}
\centering
\textbf{Value}\\[2pt]
\includegraphics[width=\linewidth]{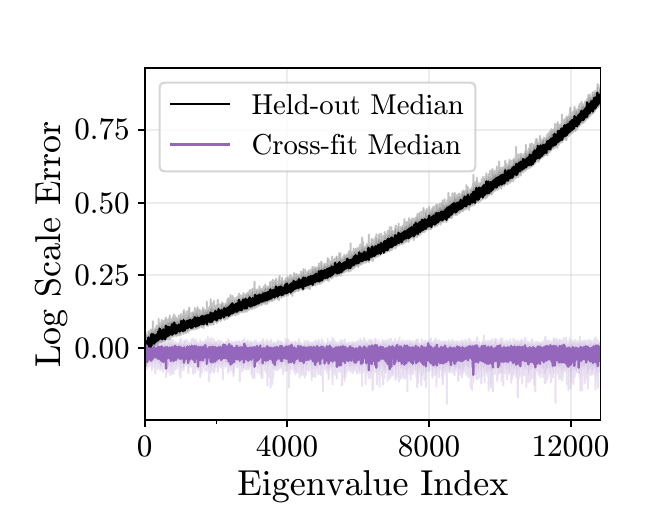}
\end{minipage}\hfill
\begin{minipage}[t]{.3\linewidth}
\centering
\textbf{Output}\\[2pt]
\includegraphics[width=\linewidth]{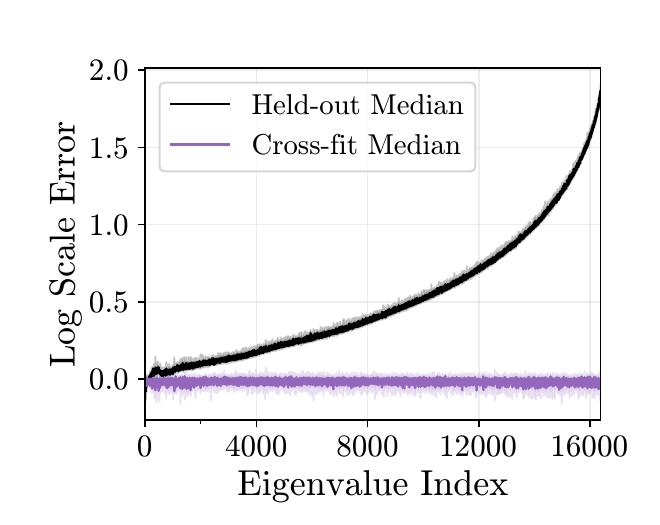}
\end{minipage}
\end{minipage}\\[7pt]
\begin{minipage}{\linewidth}
\centering
\textbf{Step 2500}\par\vspace{2pt}
\begin{minipage}[t]{.3\linewidth}
\centering
\includegraphics[width=\linewidth]{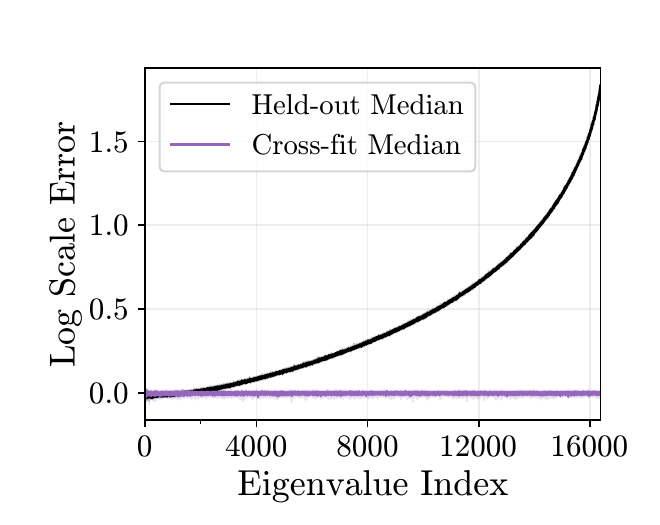}
\end{minipage}\hfill
\begin{minipage}[t]{.3\linewidth}
\centering
\includegraphics[width=\linewidth]{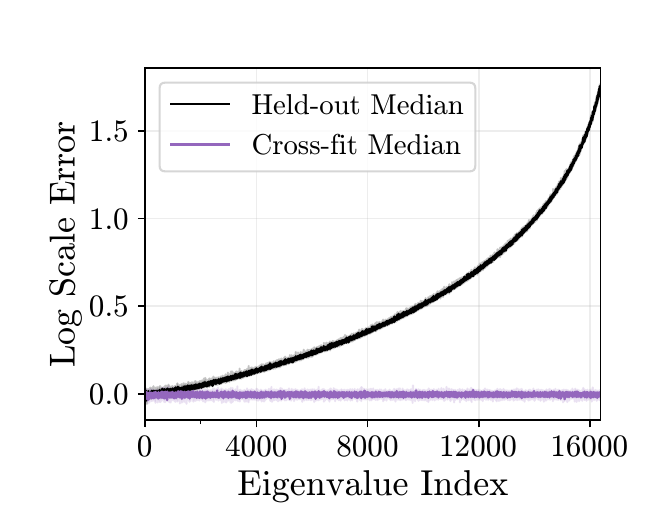}
\end{minipage}\hfill
\begin{minipage}[t]{.3\linewidth}
\centering
\includegraphics[width=\linewidth]{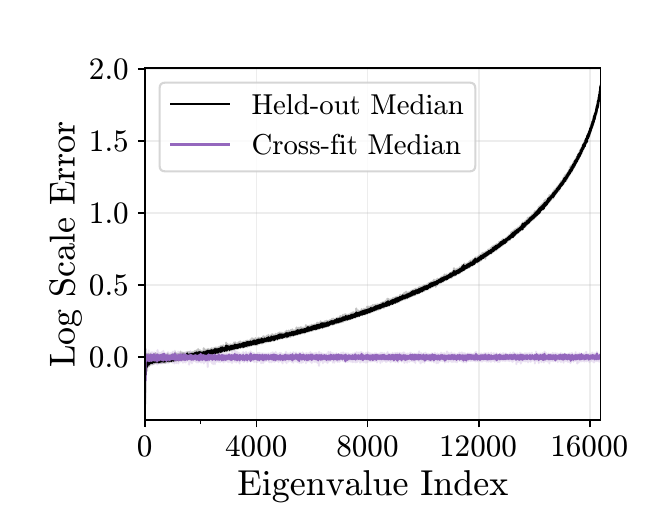}
\end{minipage}
\end{minipage}\\[7pt]
\begin{minipage}{\linewidth}
\centering
\textbf{Step 4500}\par\vspace{2pt}
\begin{minipage}[t]{.3\linewidth}
\centering
\includegraphics[width=\linewidth]{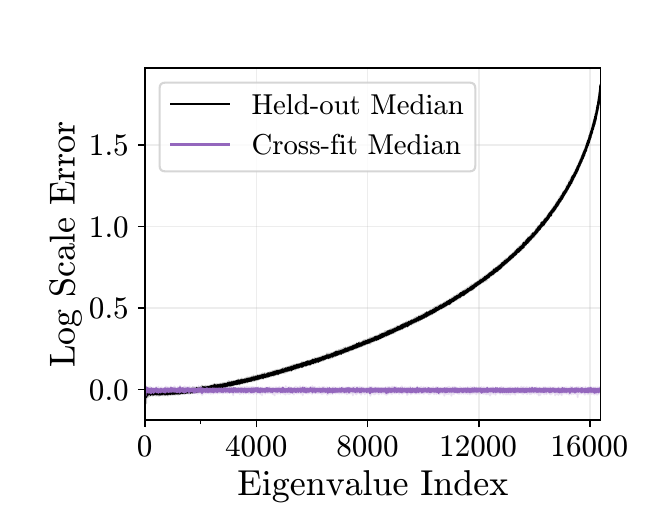}
\end{minipage}\hfill
\begin{minipage}[t]{.3\linewidth}
\centering
\includegraphics[width=\linewidth]{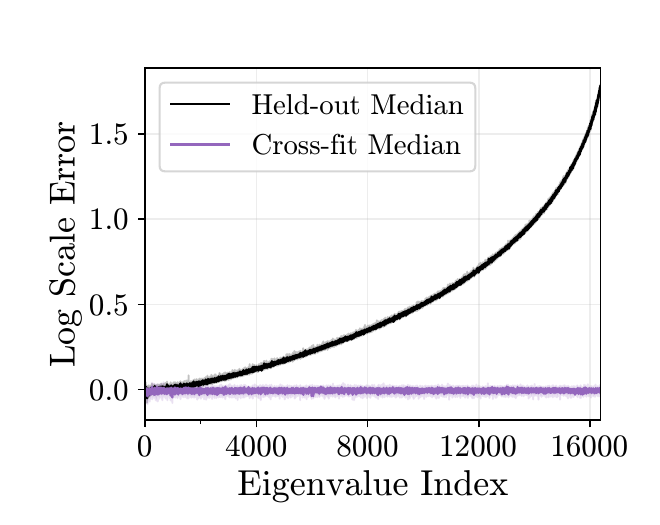}
\end{minipage}\hfill
\begin{minipage}[t]{.3\linewidth}
\centering
\includegraphics[width=\linewidth]{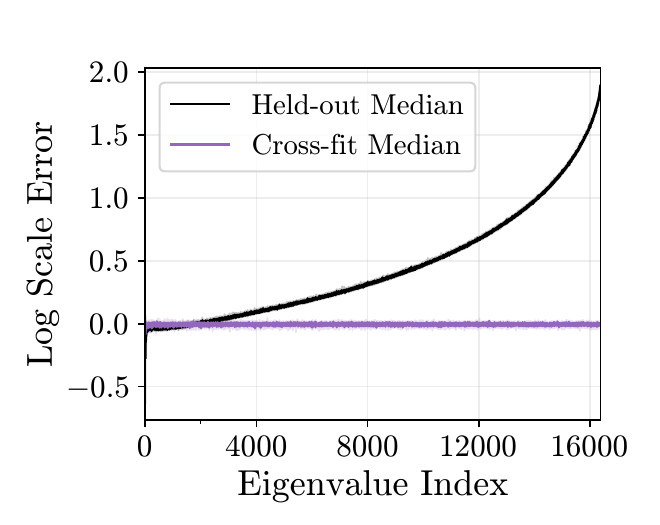}
\end{minipage}
\end{minipage}\\[7pt]
\end{appendixfigure}

\subsection{Inverse-Root Gain Across Checkpoints}
\label{app:real_gain_all}
\crefrange{fig:real_gain_100}{fig:real_gain_4500} show the gains of large-sample fits, few-sample fits, and EMA states at all three checkpoints.
Within each checkpoint-projection pair, all fits use the same split-$0$ directions and the same scales averaged over splits $1$, $2$, and $3$ (\cref{eq:heldout_log_gain}).
Evaluations use the numerically retained directions from split $0$.
The black solid curve shows the source preconditioner $C_0^{-\nicefrac{1}{2}}$, with error $\frac{1}{2}\log_{10}(h_i/\lambda_{0,i})$; the dashed line marks zero.

Large-sample curves use one fit per method.
Few-sample and EMA curves average held-out gain error over $4$ repeats and $16$ repeats, respectively.
For \cref{tab:real_heldout_gain}, we take the median absolute held-out gain error over all retained directions and repeats within each experiment, then report the mean and sample standard deviation across nine experiments.
This is $\operatorname{median}(|e_{D,i}|)$, not $|\operatorname{median}(e_{D,i})|$.

\begin{appendixfigure}{fig:real_gain_100}{Held-out gain error on AdamW checkpoints at step 100.
Rows: large-sample fit, 40-sample fit, EMA with eigendecomposition every five steps.}{Held-out gain error on AdamW checkpoints at step 100}
\begin{minipage}{\linewidth}
\centering
\begin{minipage}[t]{.3\linewidth}
\centering
\textbf{Query}\\[2pt]
\includegraphics[width=\linewidth]{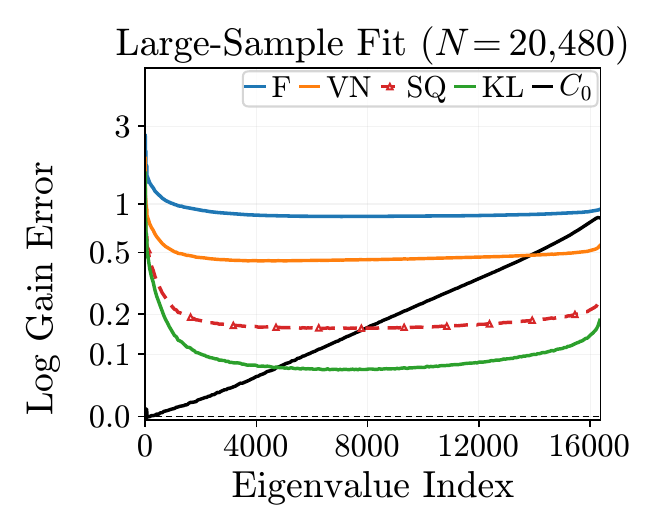}
\end{minipage}\hfill
\begin{minipage}[t]{.3\linewidth}
\centering
\textbf{Value}\\[2pt]
\includegraphics[width=\linewidth]{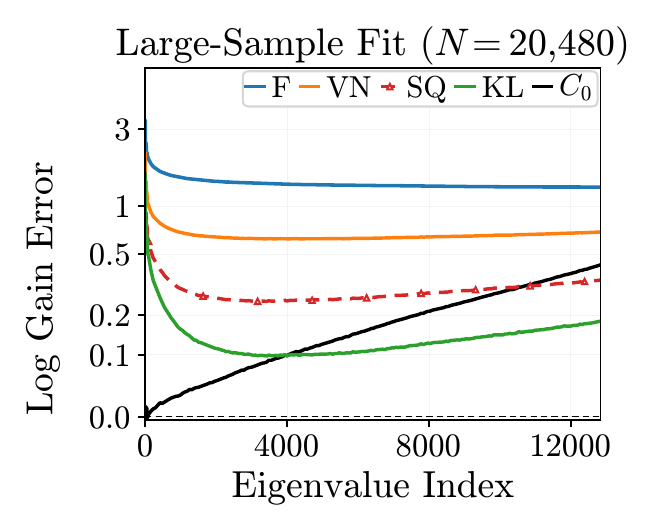}
\end{minipage}\hfill
\begin{minipage}[t]{.3\linewidth}
\centering
\textbf{Output}\\[2pt]
\includegraphics[width=\linewidth]{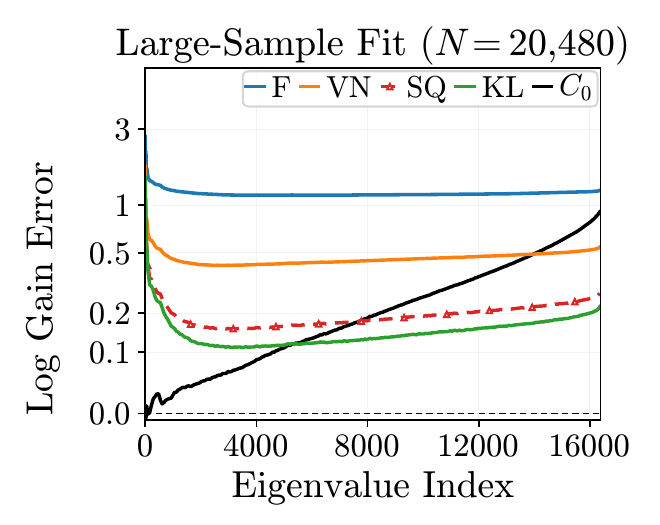}
\end{minipage}
\end{minipage}\\[7pt]
\begin{minipage}{\linewidth}
\centering
\begin{minipage}[t]{.3\linewidth}
\centering
\includegraphics[width=\linewidth]{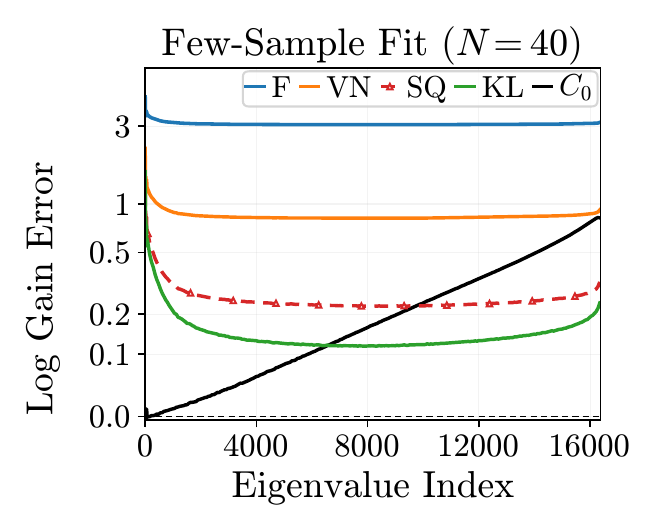}
\end{minipage}\hfill
\begin{minipage}[t]{.3\linewidth}
\centering
\includegraphics[width=\linewidth]{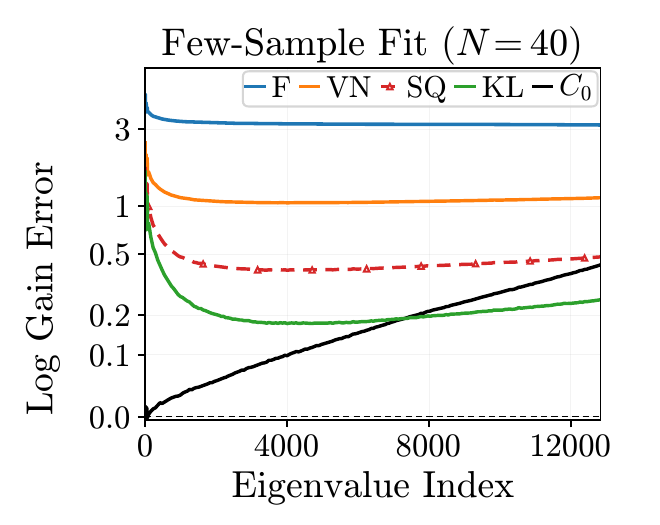}
\end{minipage}\hfill
\begin{minipage}[t]{.3\linewidth}
\centering
\includegraphics[width=\linewidth]{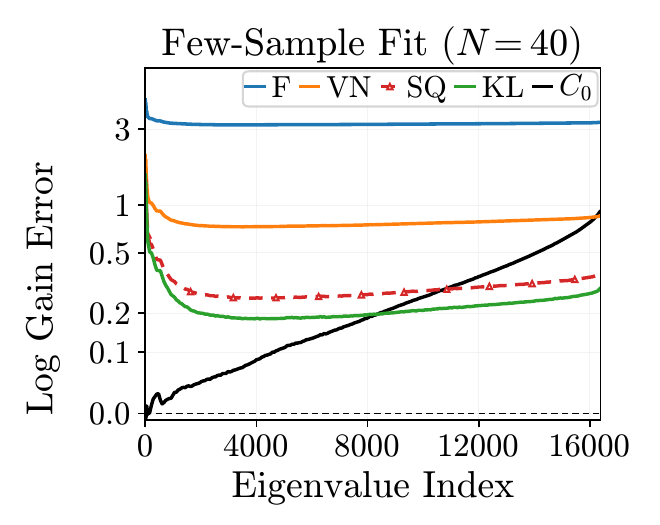}
\end{minipage}
\end{minipage}\\[7pt]
\begin{minipage}{\linewidth}
\centering
\begin{minipage}[t]{.3\linewidth}
\centering
\includegraphics[width=\linewidth]{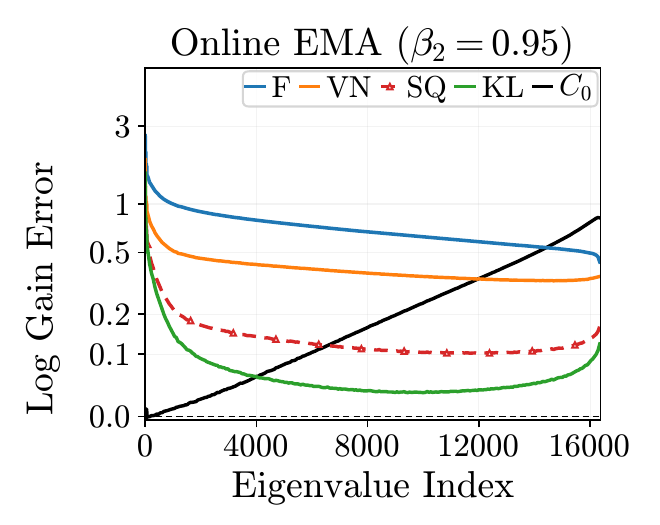}
\end{minipage}\hfill
\begin{minipage}[t]{.3\linewidth}
\centering
\includegraphics[width=\linewidth]{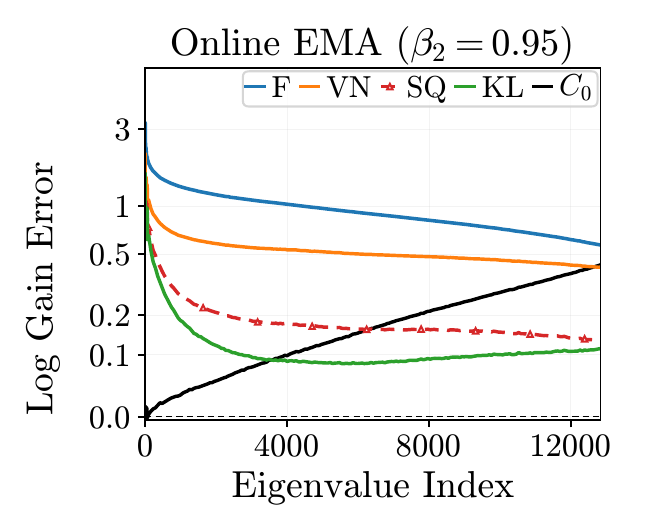}
\end{minipage}\hfill
\begin{minipage}[t]{.3\linewidth}
\centering
\includegraphics[width=\linewidth]{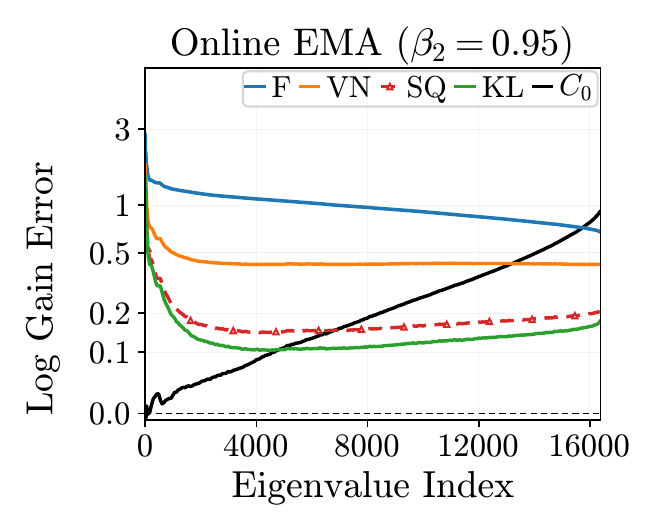}
\end{minipage}
\end{minipage}\\[7pt]
\end{appendixfigure}

\begin{appendixfigure}{fig:real_gain_2500}{Held-out gain error on AdamW checkpoints at step $2500$.
Rows: large-sample fit, 40-sample fit, EMA with eigendecomposition every five steps.}{Held-out gain error on AdamW checkpoints at step $2500$}
\begin{minipage}{\linewidth}
\centering
\begin{minipage}[t]{.3\linewidth}
\centering
\textbf{Query}\\[2pt]
\includegraphics[width=\linewidth]{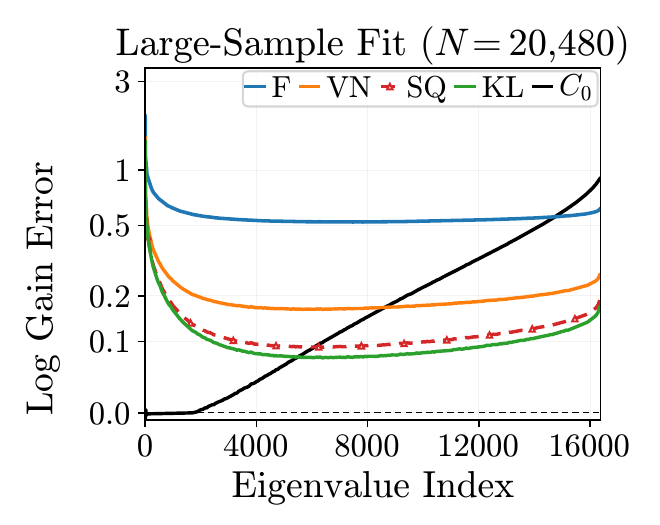}
\end{minipage}\hfill
\begin{minipage}[t]{.3\linewidth}
\centering
\textbf{Value}\\[2pt]
\includegraphics[width=\linewidth]{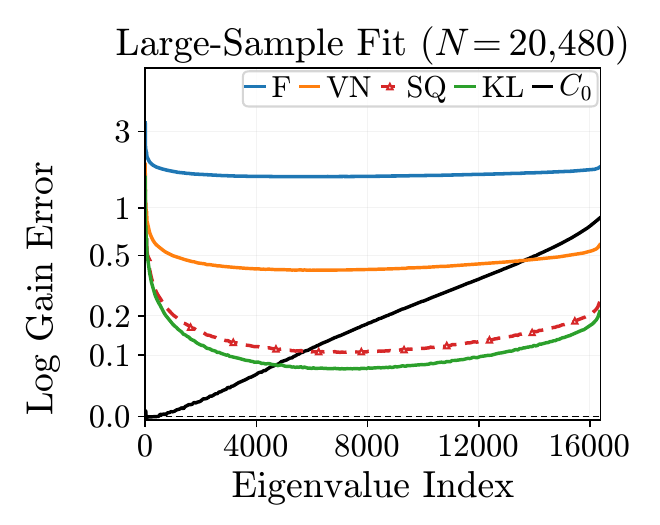}
\end{minipage}\hfill
\begin{minipage}[t]{.3\linewidth}
\centering
\textbf{Output}\\[2pt]
\includegraphics[width=\linewidth]{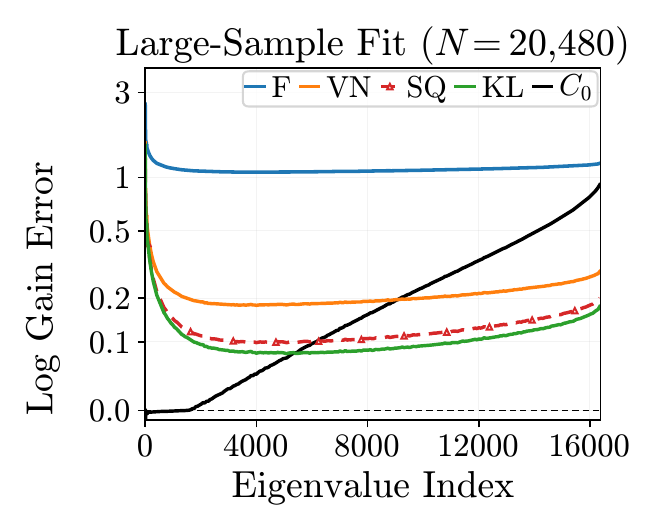}
\end{minipage}
\end{minipage}\\[7pt]
\begin{minipage}{\linewidth}
\centering
\begin{minipage}[t]{.3\linewidth}
\centering
\includegraphics[width=\linewidth]{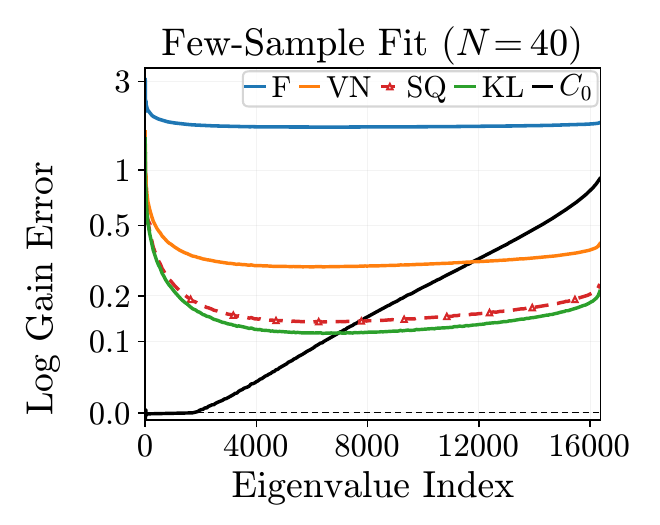}
\end{minipage}\hfill
\begin{minipage}[t]{.3\linewidth}
\centering
\includegraphics[width=\linewidth]{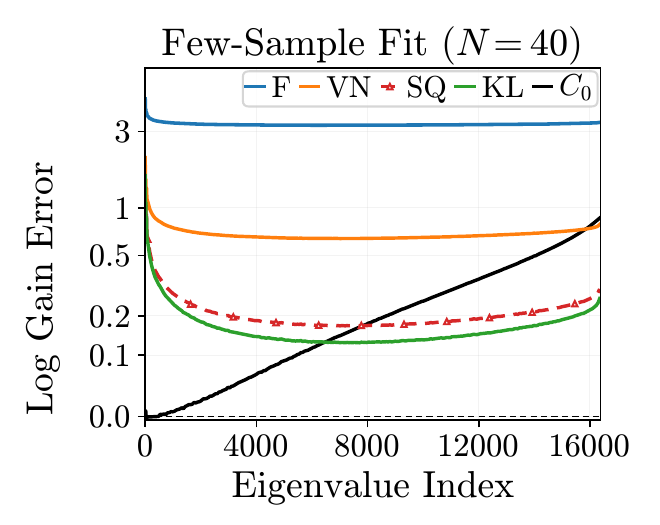}
\end{minipage}\hfill
\begin{minipage}[t]{.3\linewidth}
\centering
\includegraphics[width=\linewidth]{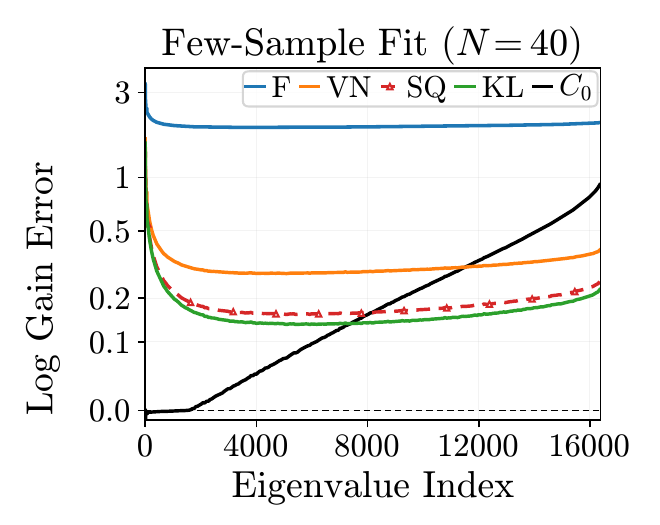}
\end{minipage}
\end{minipage}\\[7pt]
\begin{minipage}{\linewidth}
\centering
\begin{minipage}[t]{.3\linewidth}
\centering
\includegraphics[width=\linewidth]{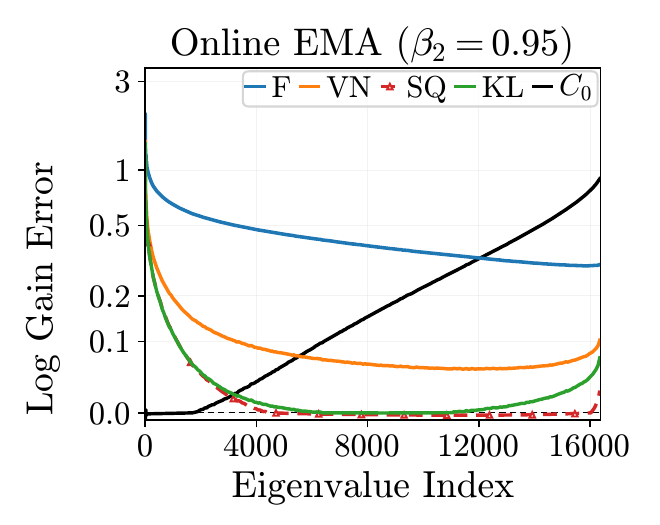}
\end{minipage}\hfill
\begin{minipage}[t]{.3\linewidth}
\centering
\includegraphics[width=\linewidth]{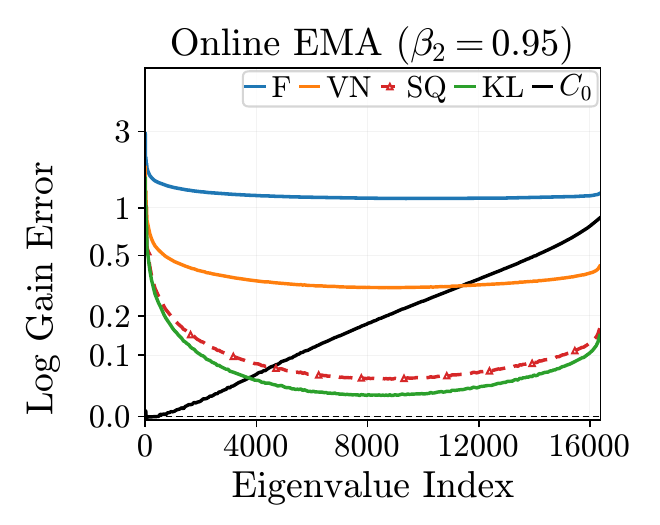}
\end{minipage}\hfill
\begin{minipage}[t]{.3\linewidth}
\centering
\includegraphics[width=\linewidth]{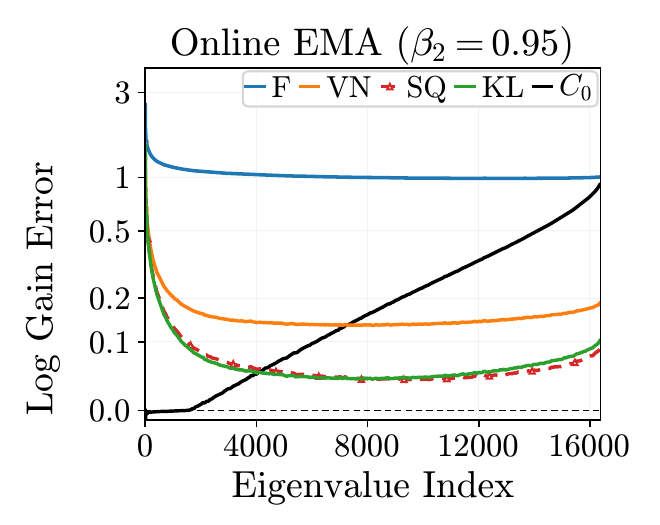}
\end{minipage}
\end{minipage}\\[7pt]
\end{appendixfigure}

\begin{appendixfigure}{fig:real_gain_4500}{Held-out gain error on AdamW checkpoints at step 4500.
Rows: large-sample fit, 40-sample fit, EMA with eigendecomposition every five steps.}{Held-out gain error on AdamW checkpoints at step 4500}
\begin{minipage}{\linewidth}
\centering
\begin{minipage}[t]{.3\linewidth}
\centering
\textbf{Query}\\[2pt]
\includegraphics[width=\linewidth]{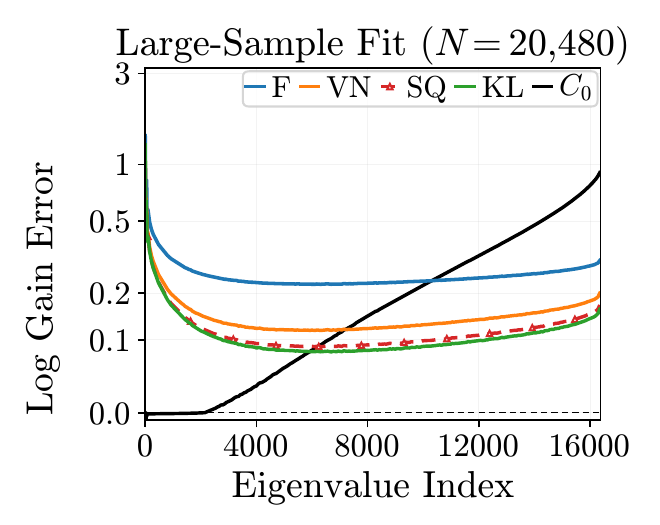}
\end{minipage}\hfill
\begin{minipage}[t]{.3\linewidth}
\centering
\textbf{Value}\\[2pt]
\includegraphics[width=\linewidth]{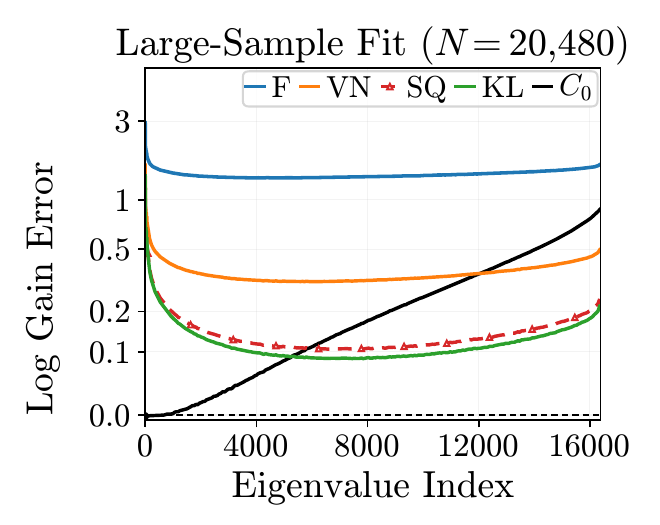}
\end{minipage}\hfill
\begin{minipage}[t]{.3\linewidth}
\centering
\textbf{Output}\\[2pt]
\includegraphics[width=\linewidth]{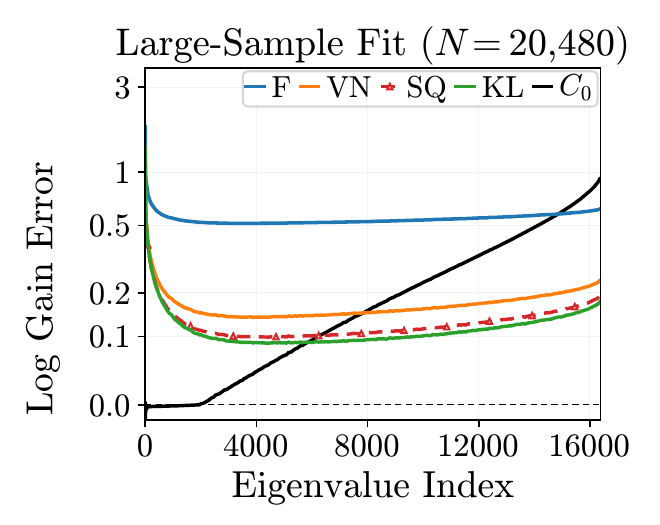}
\end{minipage}
\end{minipage}\\[7pt]
\begin{minipage}{\linewidth}
\centering
\begin{minipage}[t]{.3\linewidth}
\centering
\includegraphics[width=\linewidth]{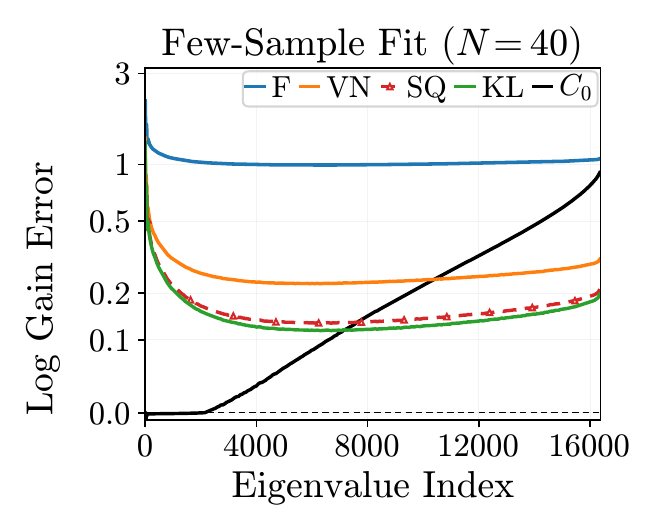}
\end{minipage}\hfill
\begin{minipage}[t]{.3\linewidth}
\centering
\includegraphics[width=\linewidth]{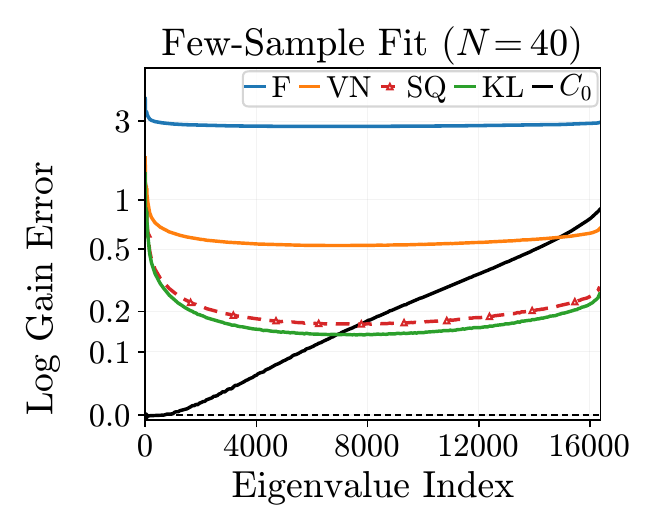}
\end{minipage}\hfill
\begin{minipage}[t]{.3\linewidth}
\centering
\includegraphics[width=\linewidth]{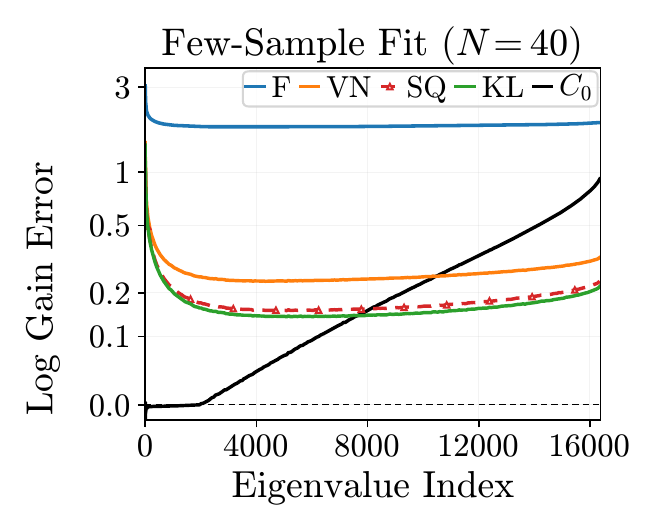}
\end{minipage}
\end{minipage}\\[7pt]
\begin{minipage}{\linewidth}
\centering
\begin{minipage}[t]{.3\linewidth}
\centering
\includegraphics[width=\linewidth]{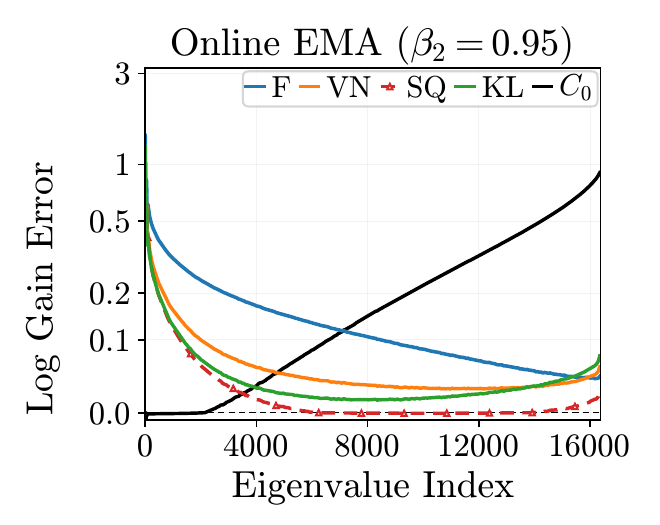}
\end{minipage}\hfill
\begin{minipage}[t]{.3\linewidth}
\centering
\includegraphics[width=\linewidth]{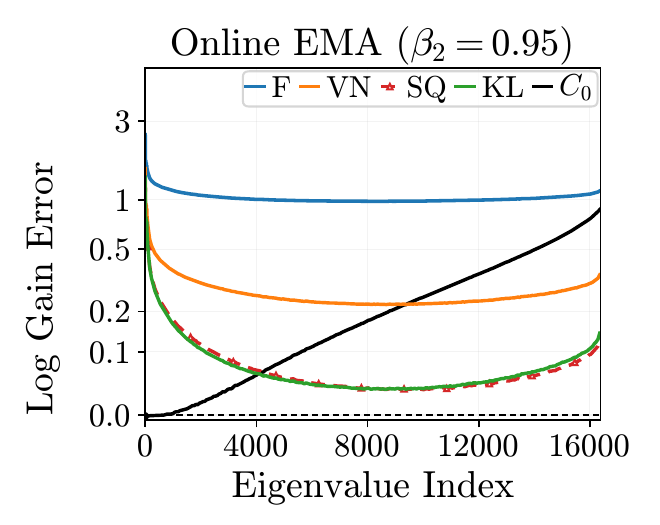}
\end{minipage}\hfill
\begin{minipage}[t]{.3\linewidth}
\centering
\includegraphics[width=\linewidth]{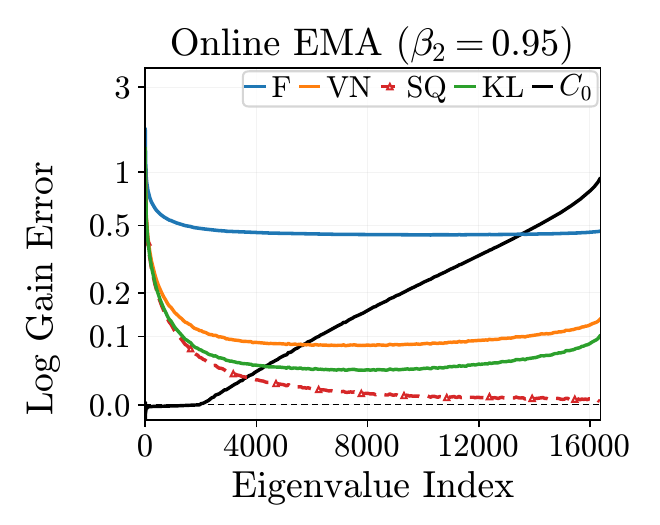}
\end{minipage}
\end{minipage}\\[7pt]
\end{appendixfigure}

\subsection{Kronecker Fits of Different Divergences}
\label{sec:real_divergence_calibration}
We compare fitted spectra across all nine checkpoint-projection pairs and show how fitted eigenvalues and eigenvectors determine gains.
Each fit uses the same split-$0$ gradients as the empirical second moment (\cref{fig:real_divergence_spectrum}).

\paragraph{Spectral preferences.}
F and VN put more weight on large eigenvalues and penalize overestimation more than reciprocal underestimation (\cref{tab:spectral_preferences}).
SQ gives large eigenvalues less weight and penalizes reciprocal errors equally.
KL weights relative errors equally across scales and penalizes underestimation more strongly.
These preferences help explain the larger fitted tail eigenvalues of SQ and KL in \cref{sec:real_kronecker_approximation}.

\paragraph{How small eigenvalues increase gains.}
Let $(\omega_{D,j},v_{D,j})$ be the eigenpairs of $S_D^{\mathrm{ref}}$ in decreasing order, and write $P_{D,ij}:=\langle u_{0,i},v_{D,j}\rangle^2$.
Then
\begin{align}
10^{2e_{D,i}}=h_i\sum_j\frac{P_{D,ij}}{\omega_{D,j}}.
\label{eq:inverse_gain_decomposition}
\end{align}
Each overlap is divided by a fitted eigenvalue, so even a small overlap can matter when that eigenvalue is tiny.
Since $\sum_jP_{D,ij}=1$, we can also write
\begin{align}
10^{2e_{D,i}}=\frac{h_i}{\omega_{D,i}}+h_i\sum_{j\ne i}P_{D,ij}\left(\frac1{\omega_{D,j}}-\frac1{\omega_{D,i}}\right).
\label{eq:inverse_gain_aligned_decomposition}
\end{align}
With aligned eigenvectors, only the first term remains.
The second term shows how gains change when the fitted eigenvectors differ from the source eigenvectors.

Let $\varepsilon_{D,i}$ be the total squared overlap with fitted eigenvectors whose eigenvalues satisfy $\omega_{D,j}\leq\tau$.
Then
\begin{align}
10^{2e_{D,i}}\geq h_i\varepsilon_{D,i}/\tau.
\label{eq:inverse_gain_leakage}
\end{align}
If $h_i\varepsilon_{D,i}>\tau$, the gain exceeds the held-out target, even if the leading eigenvalues are fitted well.

\begin{appendixfigure}{fig:real_divergence_spectrum}{Large-sample spectra.
Rows: steps 100, 2500, 4500.}{Large-sample spectra}
\begin{minipage}{\linewidth}
\centering
\textbf{Step 100}\par\vspace{2pt}
\begin{minipage}[t]{.3\linewidth}
\centering
\textbf{Query}\\[2pt]
\includegraphics[width=\linewidth]{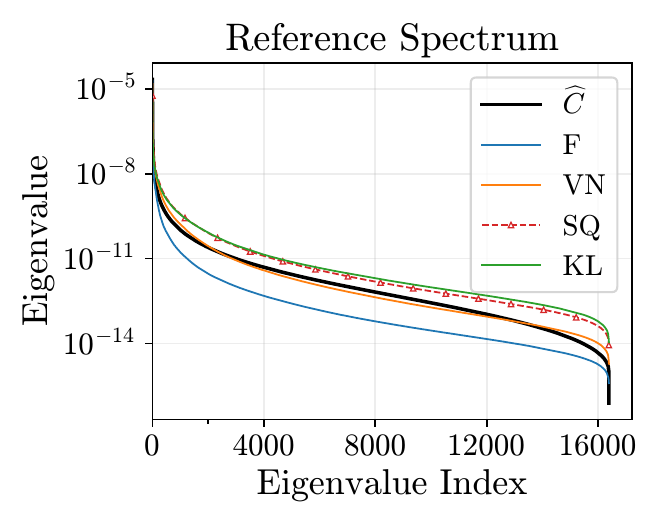}
\end{minipage}\hfill
\begin{minipage}[t]{.3\linewidth}
\centering
\textbf{Value}\\[2pt]
\includegraphics[width=\linewidth]{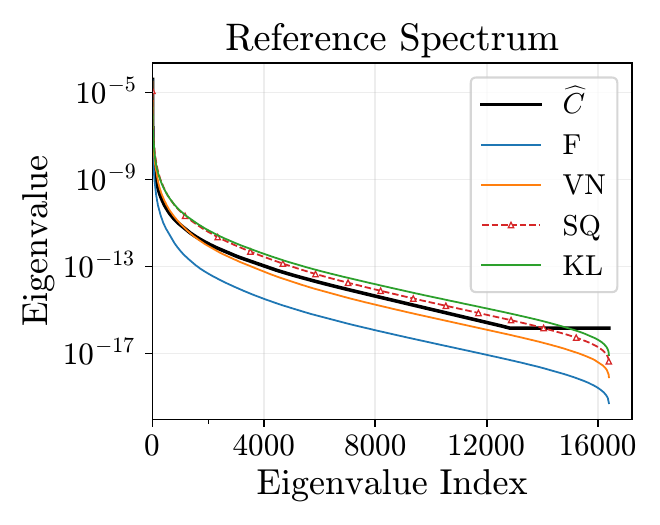}
\end{minipage}\hfill
\begin{minipage}[t]{.3\linewidth}
\centering
\textbf{Output}\\[2pt]
\includegraphics[width=\linewidth]{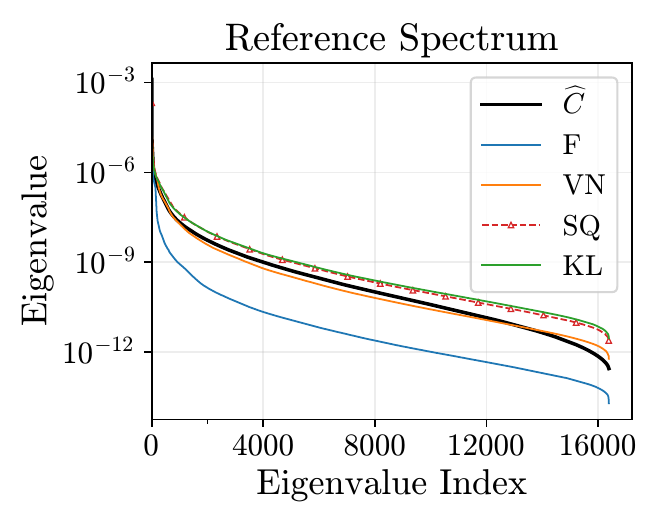}
\end{minipage}
\end{minipage}\\[7pt]
\begin{minipage}{\linewidth}
\centering
\textbf{Step 2500}\par\vspace{2pt}
\begin{minipage}[t]{.3\linewidth}
\centering
\includegraphics[width=\linewidth]{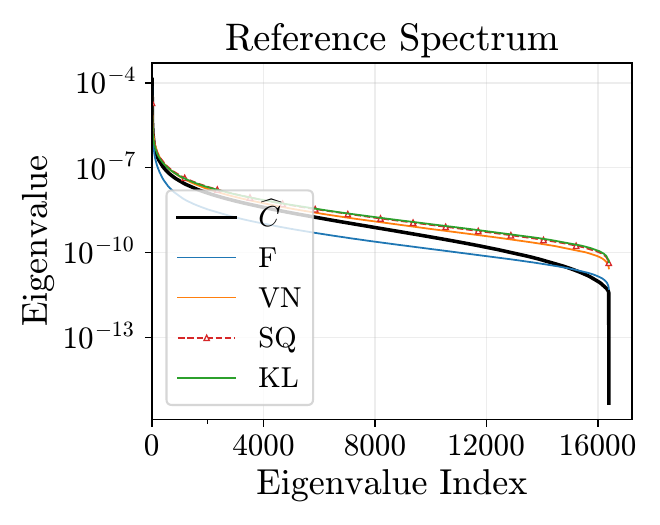}
\end{minipage}\hfill
\begin{minipage}[t]{.3\linewidth}
\centering
\includegraphics[width=\linewidth]{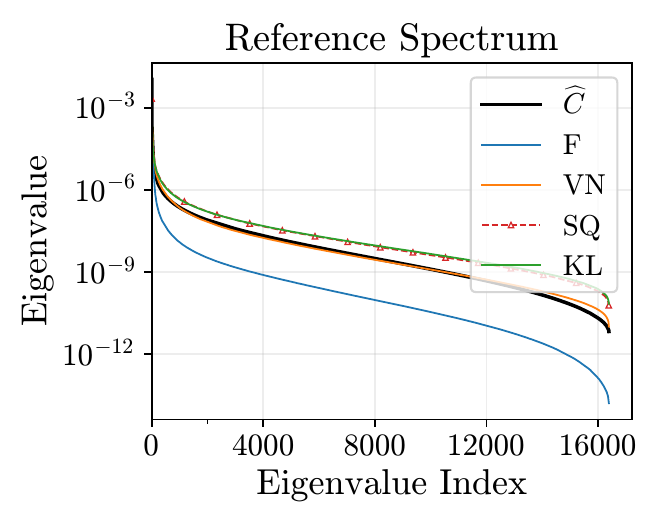}
\end{minipage}\hfill
\begin{minipage}[t]{.3\linewidth}
\centering
\includegraphics[width=\linewidth]{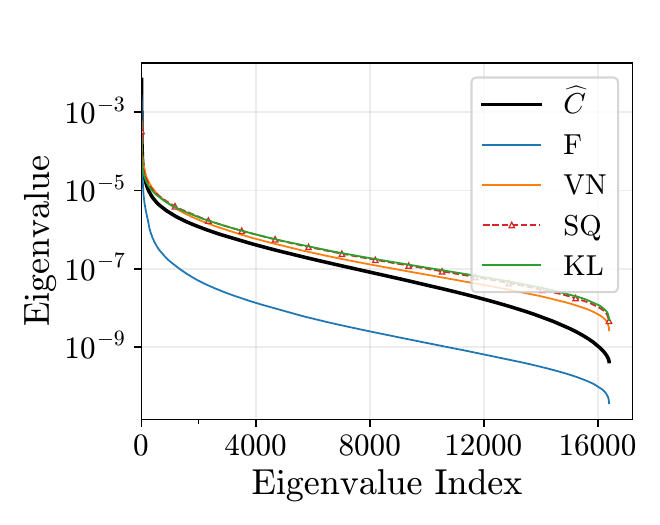}
\end{minipage}
\end{minipage}\\[7pt]
\begin{minipage}{\linewidth}
\centering
\textbf{Step 4500}\par\vspace{2pt}
\begin{minipage}[t]{.3\linewidth}
\centering
\includegraphics[width=\linewidth]{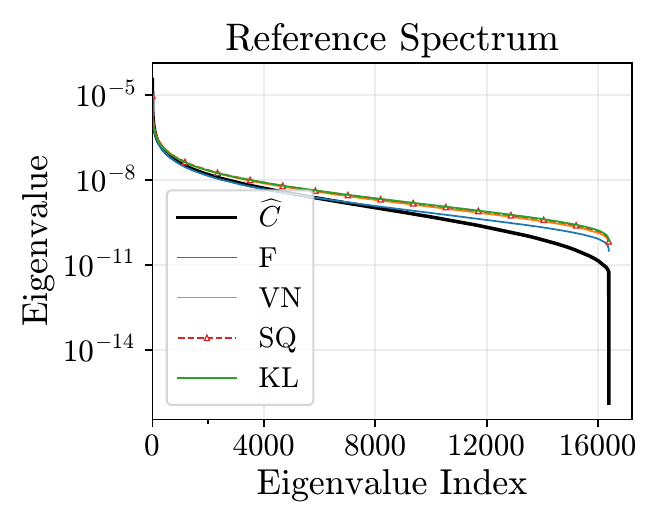}
\end{minipage}\hfill
\begin{minipage}[t]{.3\linewidth}
\centering
\includegraphics[width=\linewidth]{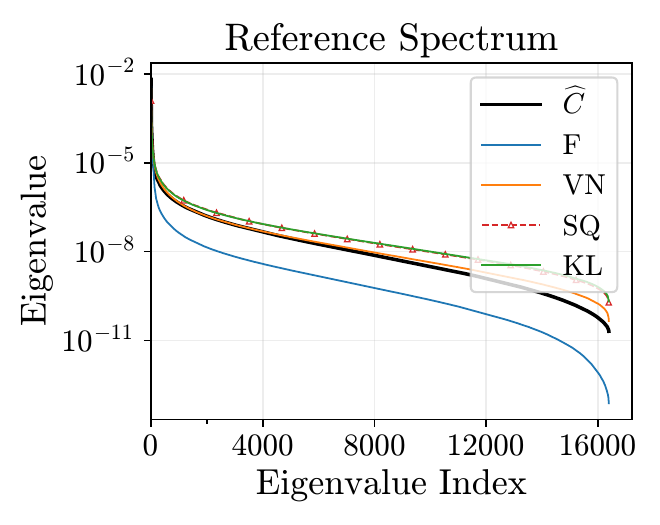}
\end{minipage}\hfill
\begin{minipage}[t]{.3\linewidth}
\centering
\includegraphics[width=\linewidth]{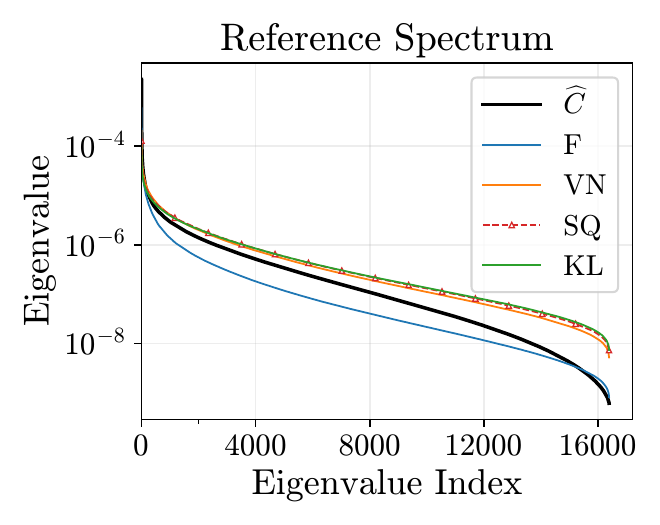}
\end{minipage}
\end{minipage}\\[7pt]
\end{appendixfigure}

\subsection{Gain Recovery with Short Gradient Histories}
\label{app:real_state_details}
Recovery is measured against each method's own large-sample fit over all $16{,}384$ reference directions (\cref{eq:same_divergence_stability_gain}).
For \cref{fig:real_state_2500}, we first average the log gain ratios across $4$ offline or $16$ EMA repeats, and then average every $64$ consecutive directions.
For \cref{tab:real_state_stability}, we compute the median absolute log gain ratio over all directions and repeats without binning, and report the mean and sample SD across nine checkpoint-projection pairs.

\begin{appendixfigure}{fig:real_state_2500}{Large-sample fit recovery on AdamW checkpoints at step $2500$.
Rows: F, VN, SQ, KL.
Solid: 40-sample fits; dashed: EMA with eigendecomposition every five steps.}{Large-sample fit recovery on AdamW checkpoints at step $2500$}
\begin{minipage}{\linewidth}
\centering
\textbf{F}\par
\begin{minipage}[t]{.3\linewidth}
\centering
\textbf{Query}\\[2pt]
\includegraphics[width=\linewidth]{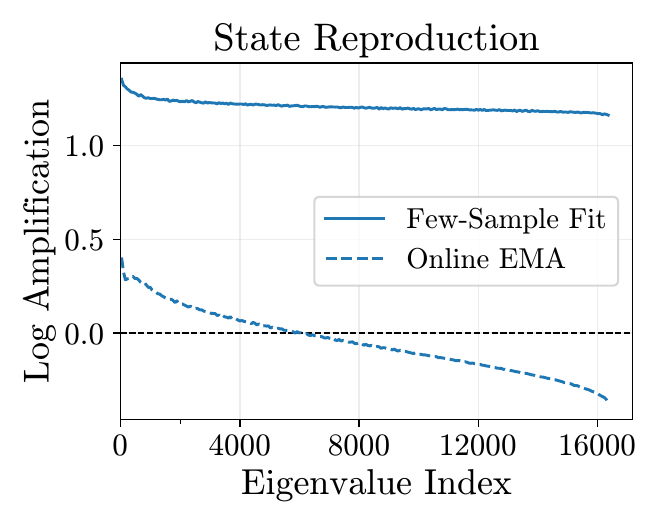}
\end{minipage}\hfill
\begin{minipage}[t]{.3\linewidth}
\centering
\textbf{Value}\\[2pt]
\includegraphics[width=\linewidth]{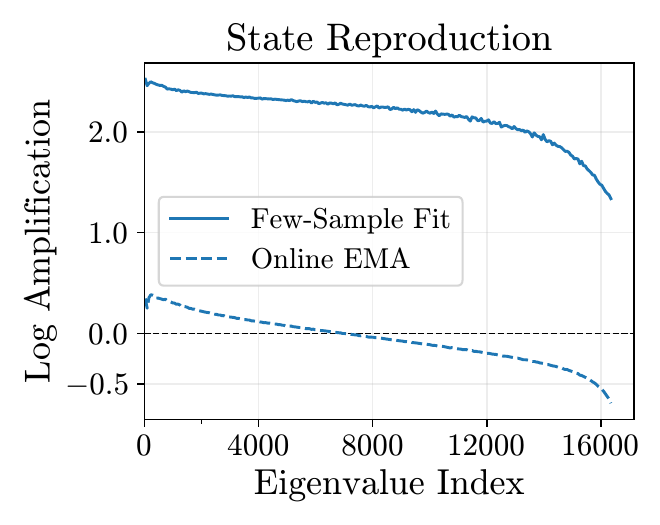}
\end{minipage}\hfill
\begin{minipage}[t]{.3\linewidth}
\centering
\textbf{Output}\\[2pt]
\includegraphics[width=\linewidth]{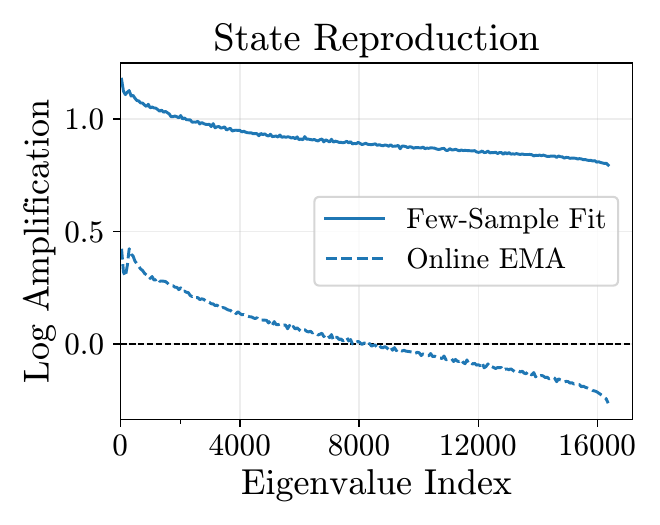}
\end{minipage}
\end{minipage}\\[7pt]
\begin{minipage}{\linewidth}
\centering
\textbf{VN}\par
\begin{minipage}[t]{.3\linewidth}
\centering
\textbf{Query}\\[2pt]
\includegraphics[width=\linewidth]{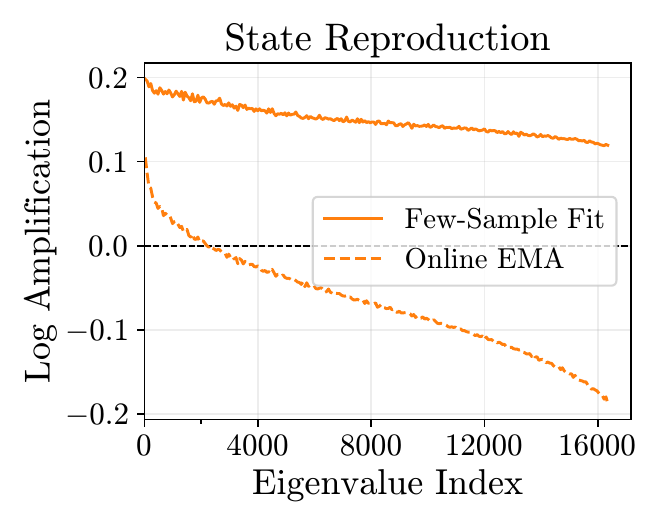}
\end{minipage}\hfill
\begin{minipage}[t]{.3\linewidth}
\centering
\textbf{Value}\\[2pt]
\includegraphics[width=\linewidth]{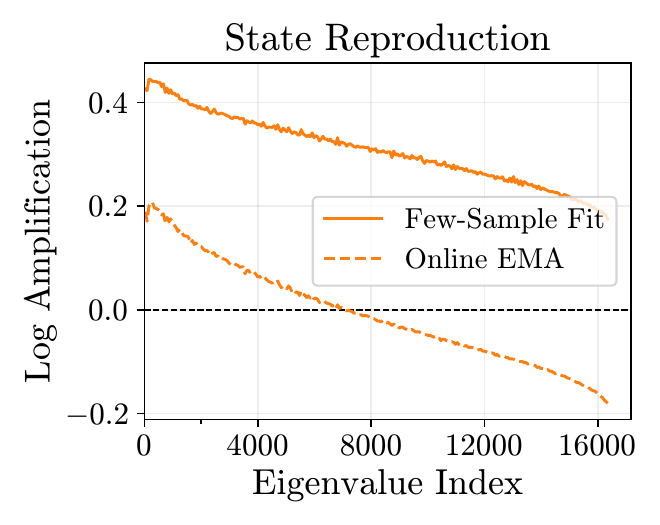}
\end{minipage}\hfill
\begin{minipage}[t]{.3\linewidth}
\centering
\textbf{Output}\\[2pt]
\includegraphics[width=\linewidth]{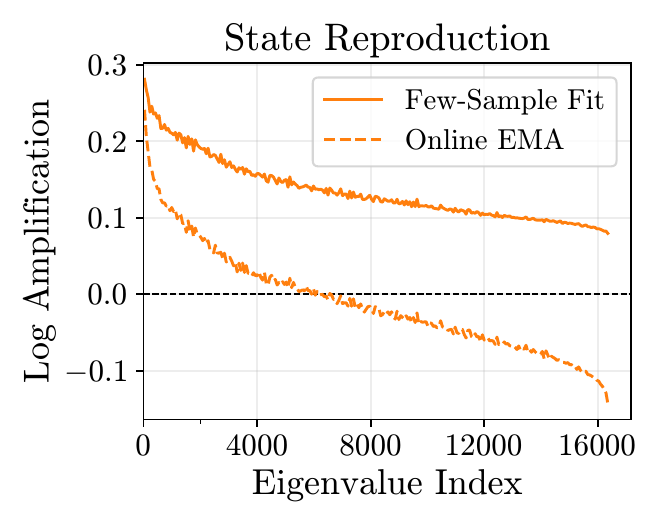}
\end{minipage}
\end{minipage}\\[7pt]
\begin{minipage}{\linewidth}
\centering
\textbf{SQ}\par
\begin{minipage}[t]{.3\linewidth}
\centering
\textbf{Query}\\[2pt]
\includegraphics[width=\linewidth]{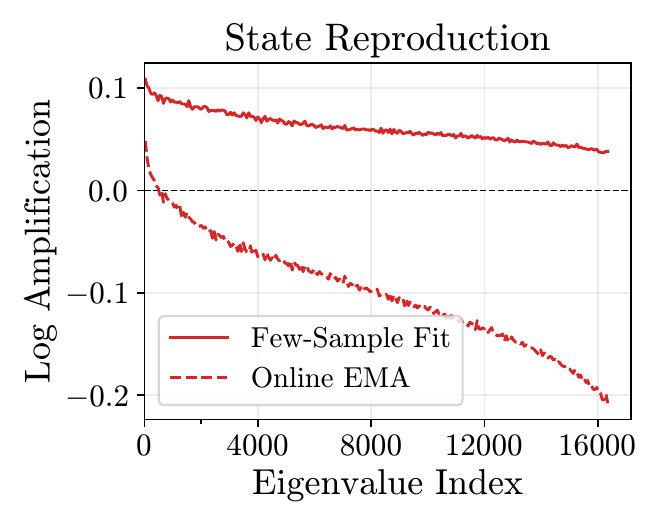}
\end{minipage}\hfill
\begin{minipage}[t]{.3\linewidth}
\centering
\textbf{Value}\\[2pt]
\includegraphics[width=\linewidth]{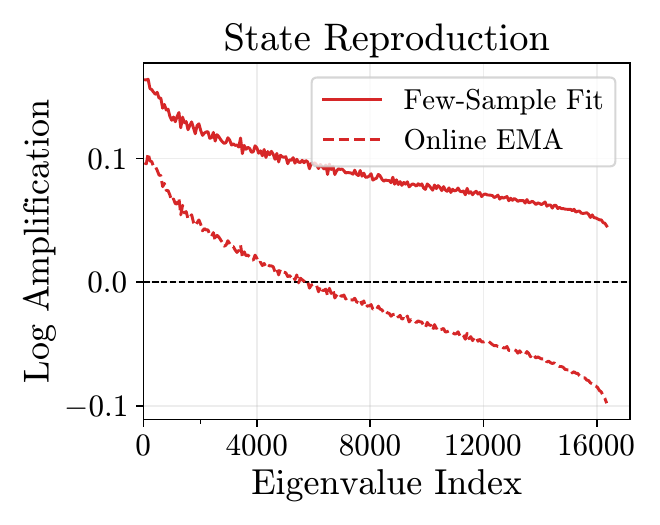}
\end{minipage}\hfill
\begin{minipage}[t]{.3\linewidth}
\centering
\textbf{Output}\\[2pt]
\includegraphics[width=\linewidth]{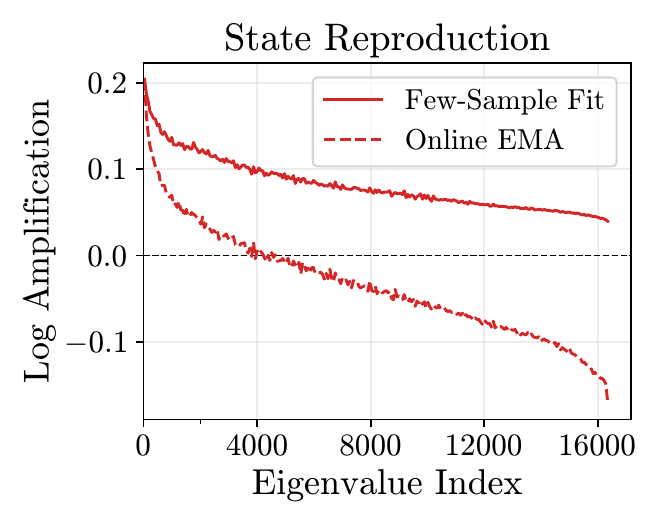}
\end{minipage}
\end{minipage}\\[7pt]
\begin{minipage}{\linewidth}
\centering
\textbf{KL}\par
\begin{minipage}[t]{.3\linewidth}
\centering
\textbf{Query}\\[2pt]
\includegraphics[width=\linewidth]{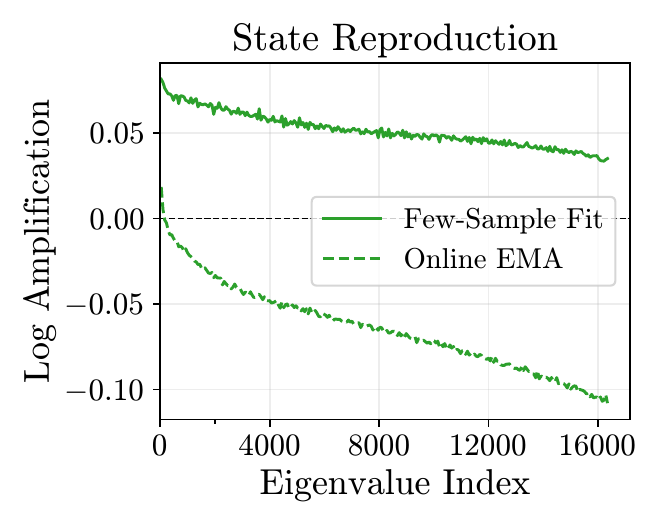}
\end{minipage}\hfill
\begin{minipage}[t]{.3\linewidth}
\centering
\textbf{Value}\\[2pt]
\includegraphics[width=\linewidth]{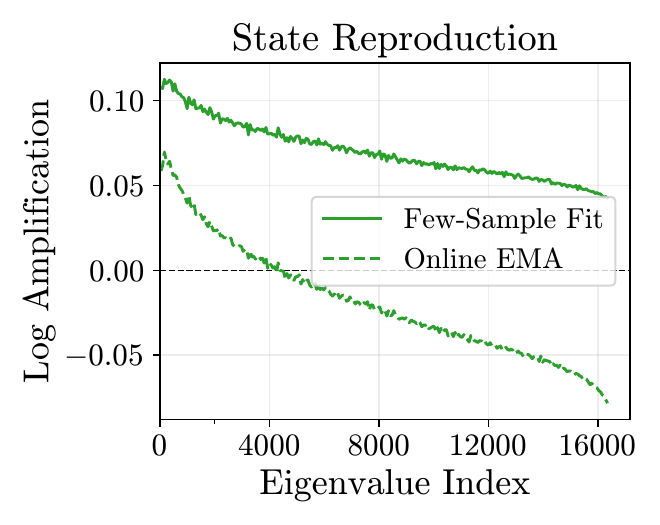}
\end{minipage}\hfill
\begin{minipage}[t]{.3\linewidth}
\centering
\textbf{Output}\\[2pt]
\includegraphics[width=\linewidth]{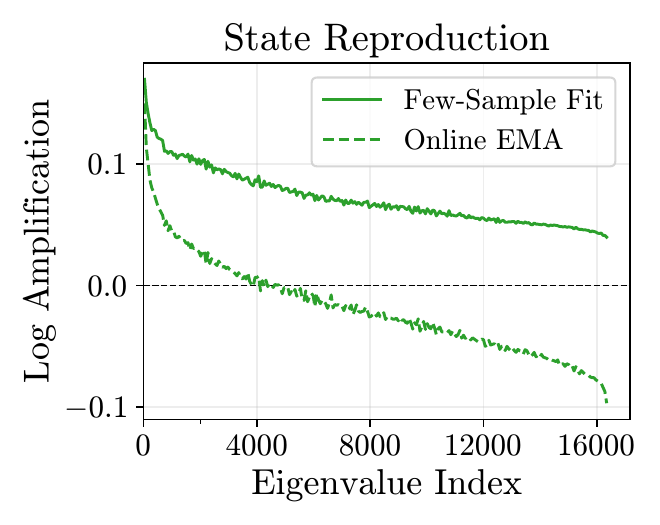}
\end{minipage}
\end{minipage}\\[7pt]
\end{appendixfigure}

\FloatBarrier

\FloatBarrier
\clearpage
\subsection{Diagnostics on Distributed Shampoo Checkpoints}
\label{app:distributed_shampoo_checkpoints}
We train the same GPT-2 model with Distributed Shampoo~\citep{shi2023distributed}, using inverse-square-root factors and Adam grafting.

We use the same checkpoint steps, projections, gradient collection, fitting settings, and aggregation as for AdamW (\cref{app:real_model_protocol}).
All four divergences are fitted to the new gradients.
\cref{fig:distributed_shampoo_diagnostics_2500} shows step $2500$; the main-text tables summarize all nine checkpoint-projection pairs.

\begin{appendixfigure}{fig:distributed_shampoo_diagnostics_2500}{Distributed Shampoo diagnostics at step $2500$.
Top: held-out scale error and cross-fit controls.
Remaining rows: held-out gain error for large-sample fits, $40$-sample fits (four repeats), and EMA ($512$ gradients, $16$ repeats, eigendecomposition every five steps).}{Distributed Shampoo diagnostics at step $2500$}
\begin{minipage}{\linewidth}
\centering
\begin{minipage}[t]{.3\linewidth}
\centering
{\small\textbf{Query}}\\[2pt]
\includegraphics[width=\linewidth]{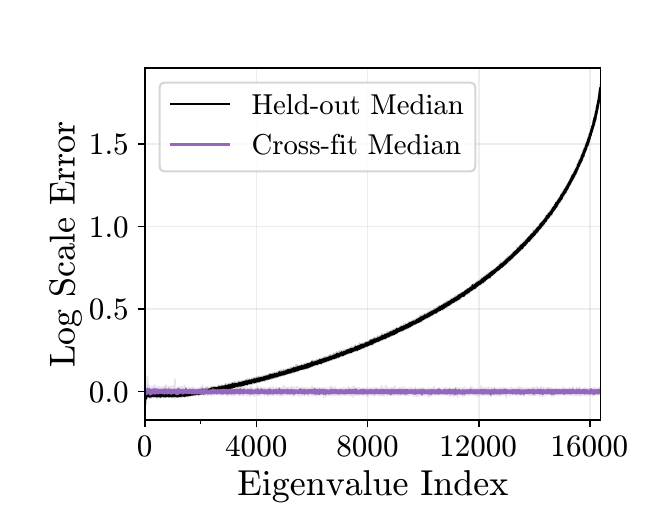}
\end{minipage}\hfill
\begin{minipage}[t]{.3\linewidth}
\centering
{\small\textbf{Value}}\\[2pt]
\includegraphics[width=\linewidth]{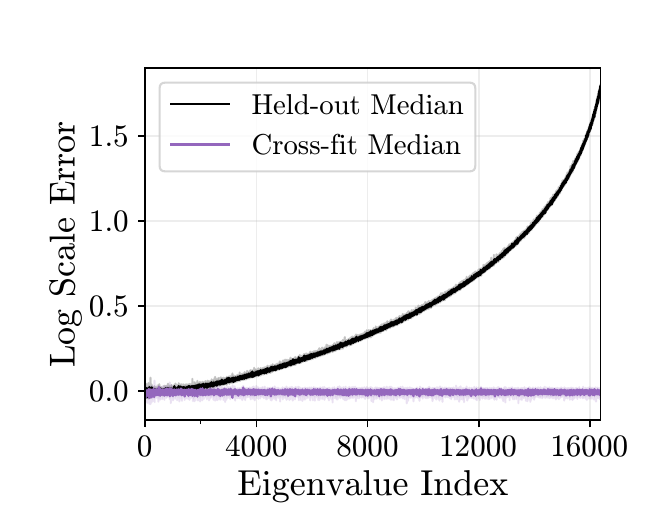}
\end{minipage}\hfill
\begin{minipage}[t]{.3\linewidth}
\centering
{\small\textbf{Output}}\\[2pt]
\includegraphics[width=\linewidth]{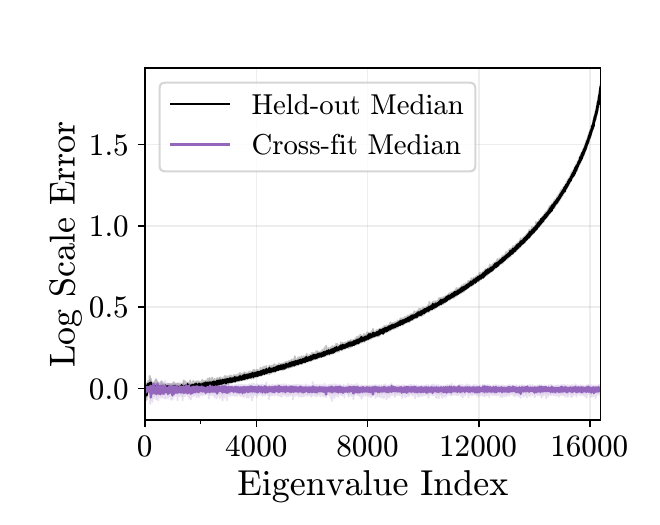}
\end{minipage}
\end{minipage}\\[7pt]
\begin{minipage}{\linewidth}
\centering
\begin{minipage}[t]{.3\linewidth}
\centering
\textbf{Query}\\[2pt]
\includegraphics[width=\linewidth]{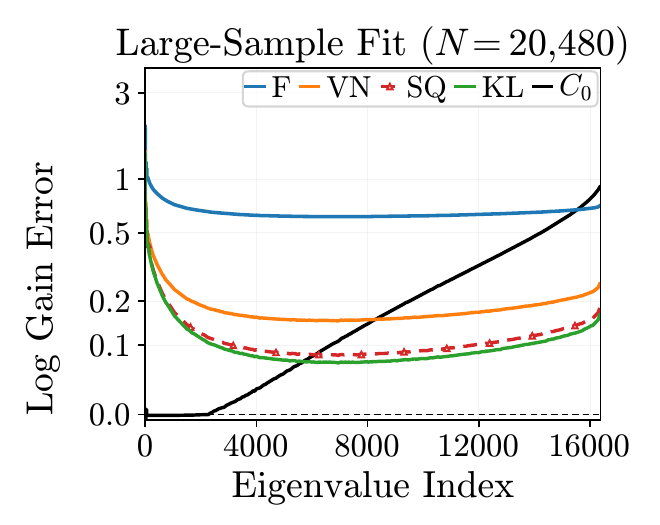}
\end{minipage}\hfill
\begin{minipage}[t]{.3\linewidth}
\centering
\textbf{Value}\\[2pt]
\includegraphics[width=\linewidth]{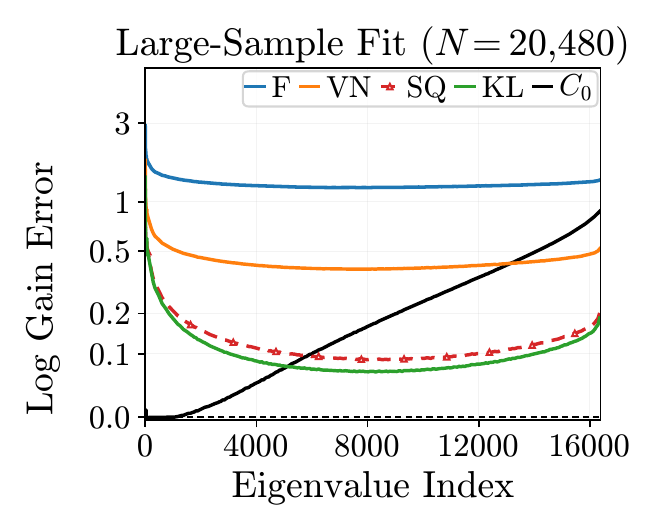}
\end{minipage}\hfill
\begin{minipage}[t]{.3\linewidth}
\centering
\textbf{Output}\\[2pt]
\includegraphics[width=\linewidth]{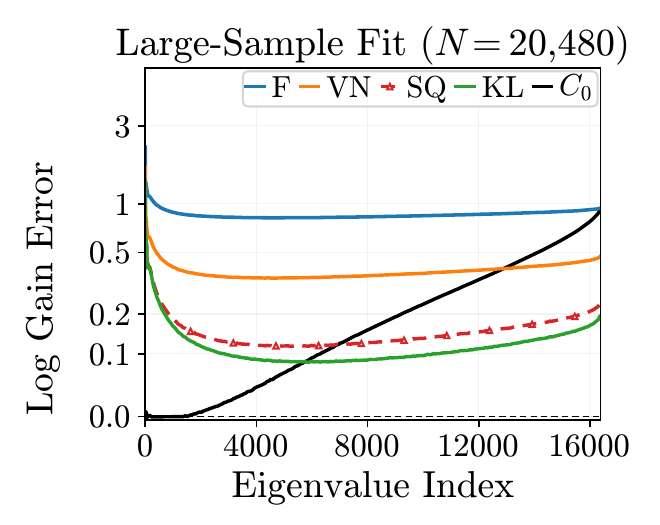}
\end{minipage}
\end{minipage}\\[7pt]
\begin{minipage}{\linewidth}
\centering
\begin{minipage}[t]{.3\linewidth}
\centering
\includegraphics[width=\linewidth]{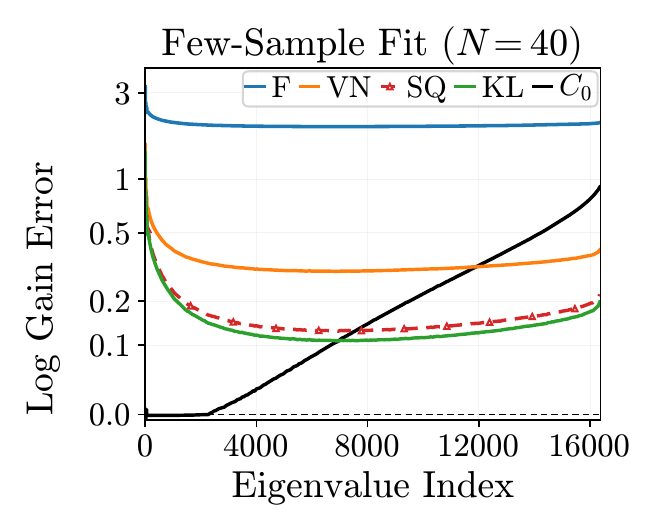}
\end{minipage}\hfill
\begin{minipage}[t]{.3\linewidth}
\centering
\includegraphics[width=\linewidth]{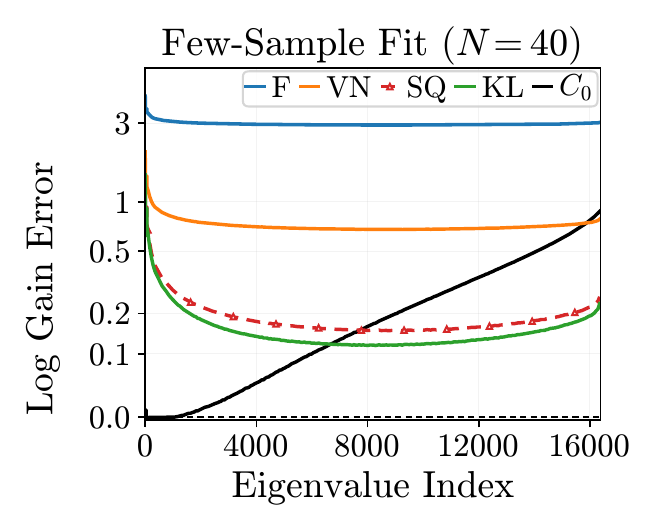}
\end{minipage}\hfill
\begin{minipage}[t]{.3\linewidth}
\centering
\includegraphics[width=\linewidth]{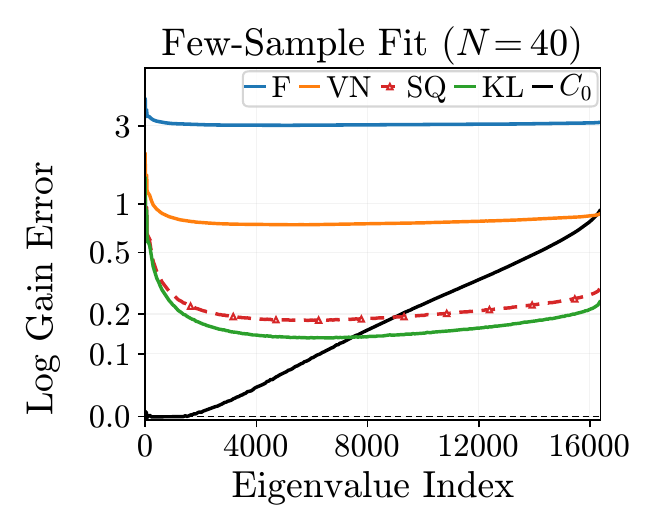}
\end{minipage}
\end{minipage}\\[7pt]
\begin{minipage}{\linewidth}
\centering
\begin{minipage}[t]{.3\linewidth}
\centering
\includegraphics[width=\linewidth]{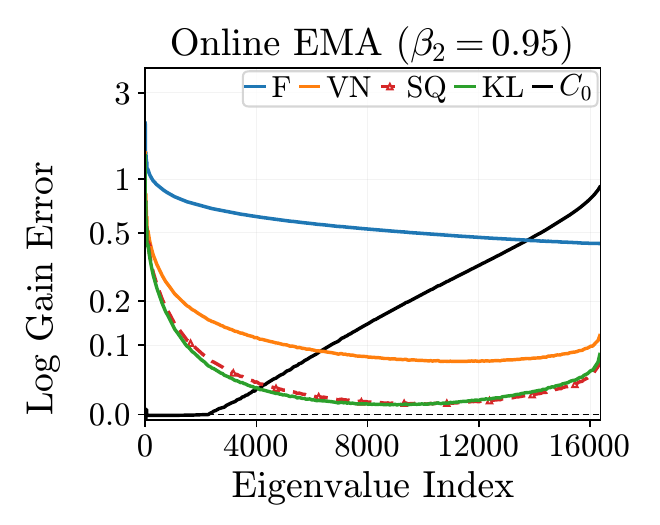}
\end{minipage}\hfill
\begin{minipage}[t]{.3\linewidth}
\centering
\includegraphics[width=\linewidth]{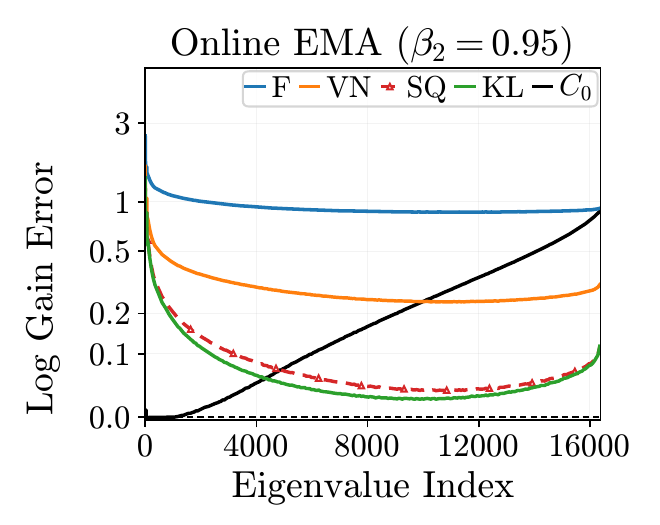}
\end{minipage}\hfill
\begin{minipage}[t]{.3\linewidth}
\centering
\includegraphics[width=\linewidth]{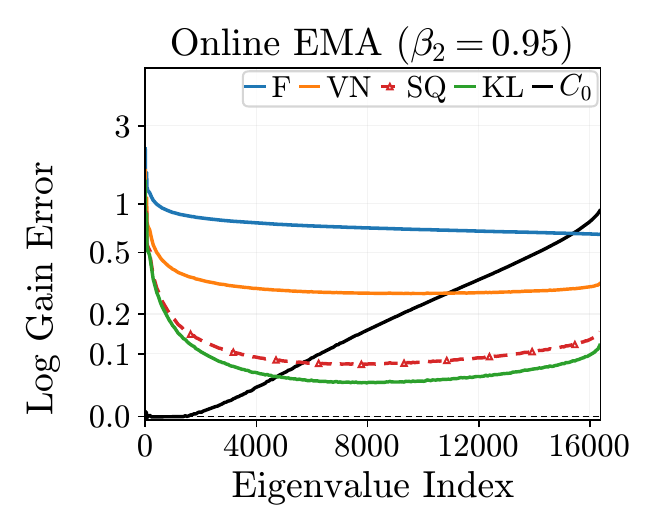}
\end{minipage}
\end{minipage}\\[7pt]
\end{appendixfigure}
\FloatBarrier

\clearpage

\section{GPT-2 Training Details and Full Results}
\label{app:training_protocol}
\begingroup
\setlength{\intextsep}{6pt plus 1pt minus 1pt}
\setlength{\textfloatsep}{8pt plus 2pt minus 2pt}
\setlength{\floatsep}{6pt plus 1pt minus 1pt}
\setlength{\abovecaptionskip}{4pt}
\setlength{\belowcaptionskip}{3pt}

\paragraph{Models and data.}
The $12$-layer models use the architecture from~\citep{modded_nanogpt_2024}.
FineWeb provides 3B training tokens.
Each run processes $2.67$B tokens with sequence length $1024$ and a global batch of $512$ sequences.

\paragraph{Training settings.}
Grid searches at model widths up to $768$ found nearly identical optimal learning rates across the tested optimizers.
We therefore use a shared learning rate across optimizers at each larger width without further searches.
Learning-rate tuning can absorb an overall gain offset, but cannot correct gain errors that vary across directions, layers, or checkpoints.
Within each width, methods share initialization, data order, and learning rate (\cref{tab:training_final_losses}).
We follow \citep{lin2026understanding} for the other hyperparameters.
The schedule uses $250$ warmup steps, a constant-rate phase, and $1450$ steps of linear decay.
The embedding and other non-matrix valued parameters use AdamW with learning rate $0.0036$, $\beta=(0.9,0.95)$, and no weight decay.

\begin{table}[!htbp]
\centering
\caption{Final validation loss after $5,100$ steps; one run per configuration.}
\label{tab:training_final_losses}
\small
\setlength{\tabcolsep}{4pt}
\begin{tabular}{rrrrrrr}
\toprule
\textbf{Width} & \textbf{Params (M)} & \textbf{LR} & \textbf{F} & \textbf{VN} & \textbf{SQ} & \textbf{KL}  \\
\midrule
256 & 22.3 & 0.0024 & 3.8281 & 3.7573 & \textbf{3.7450} & 3.7475 \\
512 & 63.5 & 0.0024 & 3.5082 & 3.4377 & \textbf{3.4173} & 3.4180  \\
768 & 123.6 & 0.0018 & 3.3993 & 3.2876 & 3.2661 & \textbf{3.2660} \\
1024 & 202.5 & 0.0012 & 3.3528 & 3.1979 & 3.1824 & \textbf{3.1769} \\
1536 & 417.0 & 0.0008 & 3.3111 & 3.0972 & 3.0817 & \textbf{3.0758} \\
\bottomrule
\end{tabular}
\end{table}
\FloatBarrier

\paragraph{Training results.}
\cref{tab:training_final_losses} supplements \cref{tab:gpt2_width_sweep} with parameter counts and learning rates.
\cref{fig:training_full_curves} shows the validation-loss trajectories from step 2,000 to the end of training.

\begin{figure}[H]
\centering
\includegraphics[width=\linewidth]{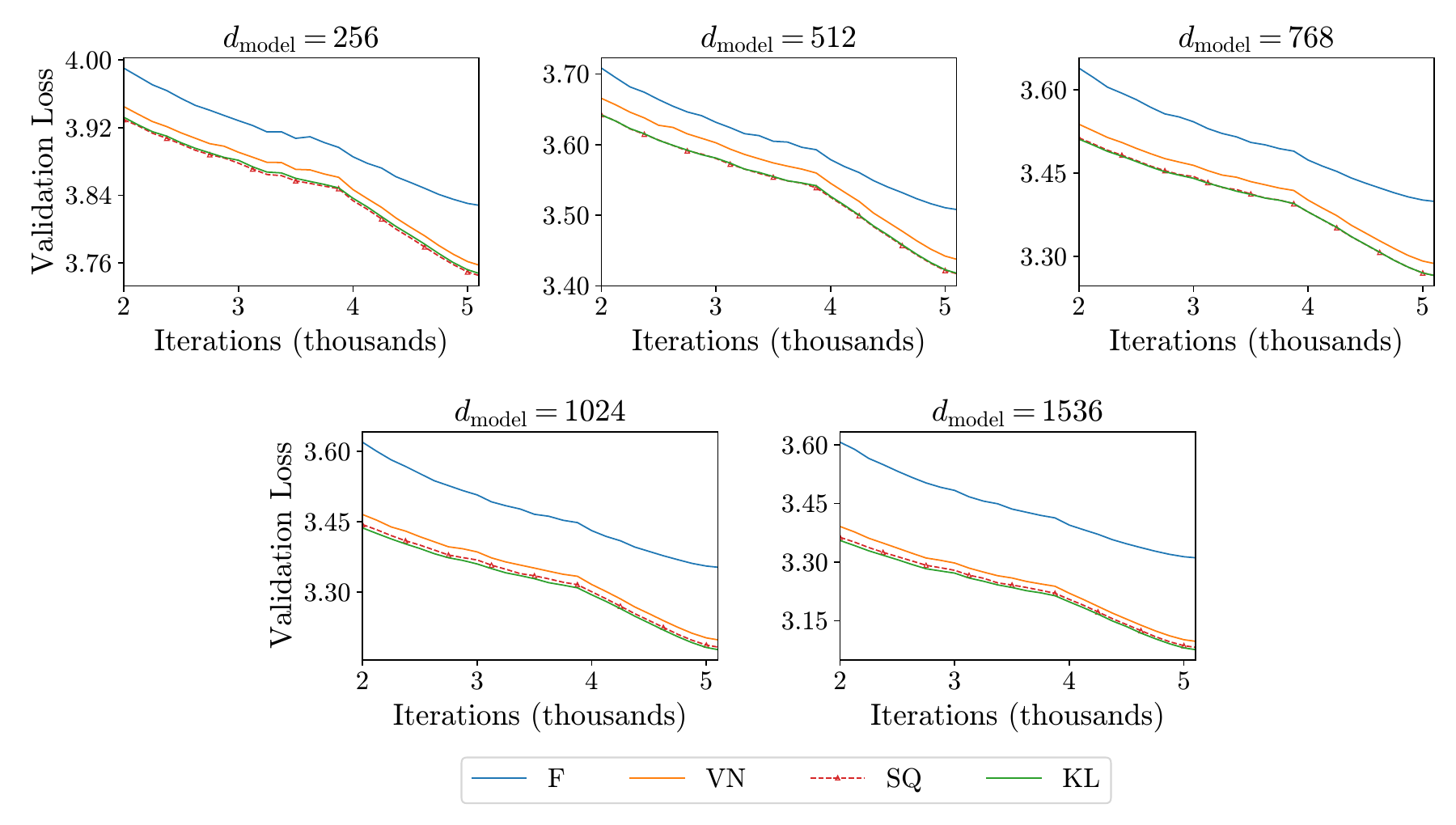}
\caption{Validation loss from step $2,000$ at all five widths.}
\label{fig:training_full_curves}
\end{figure}
\FloatBarrier

\endgroup

\end{document}